\documentclass[10pt,journal,compsoc]{IEEEtran}
\ifCLASSOPTIONcompsoc
  \usepackage[nocompress]{cite}
\else
  \usepackage{cite}
\fi

\usepackage{epsfig}
\usepackage{graphicx,subfigure}
\usepackage{amsmath}
\usepackage{amssymb}
\usepackage{color}
\usepackage{booktabs}
\usepackage{multirow}
\usepackage{url}
\usepackage[ruled, linesnumbered]{algorithm2e}
\usepackage{xcolor}
\usepackage{enumitem}
\usepackage{amsthm}
\usepackage{empheq}
\usepackage{latexsym}
\newtheorem{theorem}{Theorem}
\newtheorem{lemma}{Lemma}
\newtheorem{definition}{Definition}
\newtheorem{proposition}{Proposition}

\def\A{\mathbf{A}}

\def\C{\mathbf{C}}

\def\e{\mathbf{e}}
\def\f{\mathbf{f}}
\def\E{\mathbf{E}}

\def\R{\mathbf{R}}

\def\p{\mathbf{p}}

\def\X{\mathbf{X}}
\def\x{\mathbf{x}}

\def\t{\mathbf{t}}
\def\G{\mathbf{G}}
\def\Q{\mathbf{Q}}
\def\h{\mathbf{h}}

\DeclareMathOperator*{\vect}{vec}
\DeclareMathOperator*{\mat}{mat}
\DeclareMathOperator*{\trace}{trace}
\DeclareMathOperator*{\rank}{rank}
\DeclareMathOperator*{\acq}{ACQ}

\ifCLASSINFOpdf
\else
\fi
\begin{document}
%
\title{An Efficient Solution to Non-Minimal Case Essential Matrix Estimation}
%
%
%
%

\author{Ji Zhao
\IEEEcompsocitemizethanks{\IEEEcompsocthanksitem J. Zhao is with TuSimple, Beijing, China.\protect\\
E-mail: zhaoji84@gmail.com
}
}

%
%

\markboth{}
{Shell \MakeLowercase{\textit{et al.}}: Bare Demo of IEEEtran.cls for Computer Society Journals}
%



\IEEEtitleabstractindextext{%
\begin{abstract}
Finding relative pose between two calibrated images is a fundamental task in computer vision. Given five point correspondences, the classical five-point methods can be used to calculate the essential matrix efficiently. For the case of $N$ ($N > 5$) inlier point correspondences, which is called $N$-point problem, existing methods are either inefficient or prone to local minima. In this paper, we propose a certifiably globally optimal and efficient solver for the $N$-point problem. First we formulate the problem as a quadratically constrained quadratic program (QCQP). Then a certifiably globally optimal solution to this problem is obtained by semidefinite relaxation. This allows us to obtain certifiably globally optimal solutions to the original non-convex QCQPs in polynomial time. The theoretical guarantees of the semidefinite relaxation are also provided, including tightness and local stability. To deal with outliers, we propose a robust $N$-point method using M-estimators. Though global optimality cannot be guaranteed for the overall robust framework, the proposed robust $N$-point method can achieve good performance when the outlier ratio is not high. Extensive experiments on synthetic and real-world datasets demonstrated that our $N$-point method is $2\sim3$ orders of magnitude faster than state-of-the-art methods. Moreover, our robust $N$-point method outperforms state-of-the-art methods in terms of robustness and accuracy. 
\end{abstract}

}

\maketitle

\IEEEdisplaynontitleabstractindextext

%
\IEEEpeerreviewmaketitle

\IEEEraisesectionheading{\section{Introduction}}

\IEEEPARstart{F}{INDNG} relative pose between two images using 2D-2D point correspondences is a cornerstone in geometric vision. It makes a basic building block in many structure-from-motion (SfM), visual odometry, and simultaneous localization and mapping (SLAM) systems~\cite{hartley2003multiple}. 
Relative pose estimation is a difficult problem since it is by nature non-convex and known to be plagued by local minima and ambiguous solutions. Relative pose for uncalibrated and calibrated cameras are usually characterized by fundamental matrix and essential matrix, respectively~\cite{hartley2003multiple}. A matrix is a fundamental matrix if and only if it has two non-zero singular values. An essential matrix has an additional property that the two nonzero singular values are equal. Due to these strict constraints, essential matrix estimation is arguably thought to be more challenging than fundamental matrix estimation. 
This paper focuses on optimal essential matrix estimation. 

Due to the scale ambiguity of translation, relative pose for calibrated cameras has $5$ degrees-of-freedom (DoFs), including $3$ for rotation and $2$ for translation. Except for degenerate configurations, $5$ point correspondences are hence enough to determine the relative pose. 
Given five point correspondences, the five-point methods using essential matrix~\cite{nister2004efficient,stewenius2006recent} or rotation matrix parametrization~\cite{kneip2012finding} can be used to calculate the relative pose efficiently.
The aforementioned solvers are the so-called \emph{minimal solvers}. When point correspondences contain outliers, minimal solvers are usually integrated into a hypothesize-and-test framework, such as RANSAC~\cite{Fischler81}, to find the solution corresponding to the maximal consensus set. Hence this framework can provide high robustness.

Once the maximal consensus set has been found, the standard RANSAC optionally re-estimates a model by using all inliers to reduce the influence of noise~\cite{hartley2003multiple,raguram2013usac}.  
Thus a \emph{non-minimal solver} is needed in this procedure. This RANSAC framework is called the gold standard algorithm~\cite{hartley2003multiple}.
Figure~\ref{fig:teaser} illustrates this framework by taking line fitting as an example. 
In addition to the usage as post-processing, the non-minimal solvers can also be integrated more tightly into RANSAC variants. For example, LO-RANSAC~\cite{chum2003locally} attempts to enlarge the consensus set of an initial RANSAC estimate by generating hypotheses from \emph{larger-than-minimal subsets} of the consensus set. The rationale is that hypotheses fitted on a larger number of inliers typically lead to better estimates with higher support.

In this paper, the non-minimal case relative pose estimation is called \emph{$N$-point problem}, and its solver is called \emph{$N$-point method}.  
As it has been investigated in~\cite{chesi2009camera,kneip2013direct,briales2018certifiably}, $N$-point methods usually lead to more accurate results than five-point methods. 
Thus the $N$-point method is useful for scenarios that require accurate pose estimation, such as visual odometry and image-based visual servoing. 
The well-known direct linear transformation (DLT) technique~\cite{hartley2003multiple} with proper normalization~\cite{hartley1995defence} can be used to estimate the essential matrix using $8$ or more point correspondences. However, DLT ignores the inherent nonlinear constraints of the essential matrix. To deal with this problem, an essential matrix is recovered after an approximated essential matrix is obtained from the DLT solution~\cite{hartley2003multiple}. 
An eigenvalue-based formulation and its variant were proposed to solve $N$-point problem~\cite{kneip2013direct,briales2018certifiably}. 
However, all the aforementioned methods fail to guarantee global optimality and efficiency simultaneously.

\begin{figure}[tbp]
	\begin{center}
	{
		\includegraphics[width=0.99\linewidth]{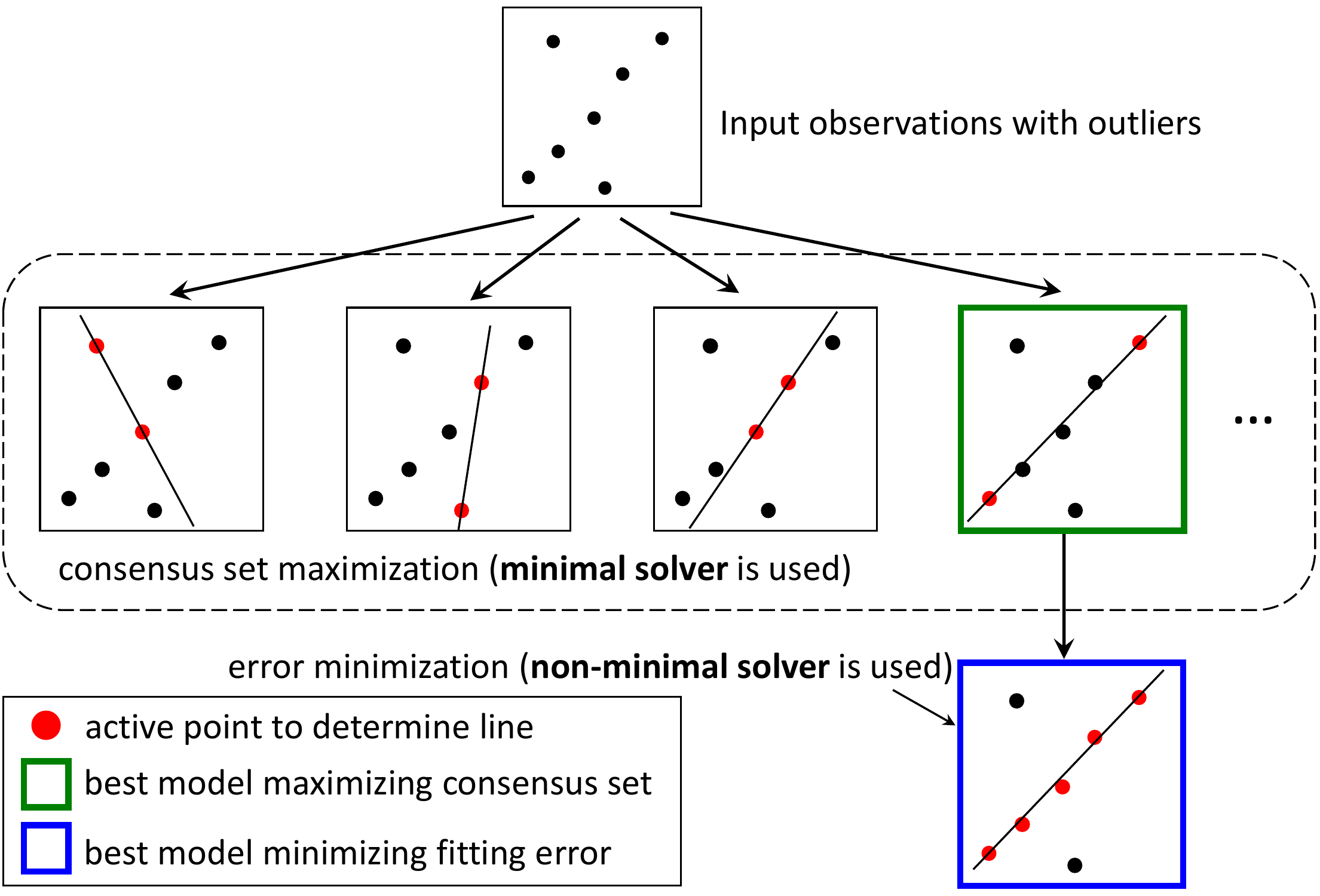}
	}
	\end{center}
	\vspace{-0.15in}
	\caption{The RANSAC framework contains the collaboration of a minimal solver and a non-minimal solver. This framework is known as the gold standard algorithm~\cite{hartley2003multiple}. This paper focuses on relative pose estimation. Here we take line fitting for an example due to its convenient visualization.}
	\label{fig:teaser}
\end{figure}

Since the $N$-point problem is challenging, the progress of its solvers is far behind the absolute pose estimation (perspective-$n$-point (P$n$P)) and point cloud registration. 
For example, the EP$n$P algorithm ~\cite{lepetit2009epnp} in P$n$P area is very efficient and has linear complexity with the number of observations. It has been successfully used in RANSAC framework given an arbitrary number of 2D-3D correspondences. 
Though relative pose estimation is more difficult than absolute pose estimation, it is desirable to find a practical solver to $N$-point problem, whose efficiency and global optimality are both satisfactory. This is one motivation of this paper.

Another motivation of this paper is developing an \emph{M-estimator} based method for relative pose estimation. 
In practice, it arises frequently that the data have been contaminated by large noise and outliers.
Existing robust estimation methods are mainly classified into two main categories, i.e., inlier set maximization~\cite{chin2017maximum} and M-estimator based method~\cite{huber1981robust}. 
Inlier set maximization can be achieved by randomized sampling methods~\cite{Fischler81,chum2003locally} or deterministic optimization methods~\cite{cai2018deterministic,le2019deterministic}. 
For M-estimator based methods, the associated optimization problems are non-convex and difficult to solve. In this paper, the proposed robust $N$-point method uses M-estimators. To solve the associated optimization problem effectively, the line process is adopted~\cite{black1996unification}. The line process uses a continuous iterative optimization strategy which takes a weighted version of non-minimal solver as a vital requirement. Due to the lack of an efficient and globally optimal $N$-point method, there did not exist a practical M-estimator based relative estimation method before.

Based on the aforementioned motivations, in this paper we propose a novel $N$-point method and integrate this method into M-estimators. The contributions of this paper are three-fold.
\begin{itemize}
	\item {\bf Efficient and globally optimal $N$-point method}. A simple parameterization is proposed to characterize the essential manifold. Based on this parameterization, a certifiably globally optimal $N$-point method is proposed, which is $2\sim 3$ orders of magnitude faster than state-of-the-art methods.
	 
	\item {\bf Robust $N$-point method}. We propose a robust essential matrix estimation method by integrating $N$-point method into M-estimators. 
	Considering that the robust components of the overall framework are not certifiably and provably optimal, we can only demonstrated empirical performance assurances. 
	
	\item {\bf Theoretical aspects}. We provide theoretical proofs of the semidefinite relaxation (SDR) in the proposed $N$-point method, including SDR tightness and local stability with small observation noise. 
\end{itemize}

The paper is organized as follows. Section~\ref{sec:related} introduces the related work. In Section~\ref{sec:formulation}, we propose a simple parameterization of essential manifold and provide novel formulations of $N$-point methods. Based on these formulations, Section~\ref{sec:sdp} derives a convex optimization approach by SDR. Section~\ref{sec:theory} proves tightness and local stability of SDR. 
A robust $N$-point method based on robust loss function is proposed in Section~\ref{sec:robust}. 
Section~\ref{sec:exp} presents the performance
of our method in comparison to other approaches, followed by a concluding discussion in Section~\ref{sec:conclusion}.

\section{Related Work}
\label{sec:related}

Estimating an essential matrix for a calibrated camera from point correspondences is an active research area in computer vision. 
Finding the optimal essential matrix by $L_\infty$ norm cost and branch-and-bound (BnB) was proposed in~\cite{hartley2009global}. It achieves global optima but is inefficient. 
There are several works for $N$-point fundamental/essential matrix estimation using local optimization~\cite{zhang1998determining,kanatani2010unified} or manifold optimization~\cite{ma2001optimization,helmke2007essential,tron2017space}. 
A method for minimizing an algebraic error was investigated in~\cite{chesi2009camera}. Its global optimality was obtained by a square matrix representation of homogeneous forms and relaxation. 
An eigenvalue-based formulation was proposed to estimate rotation matrix~\cite{kneip2013direct}, in which the problem is optimized by local gradient descent or BnB search. It was later improved by a certifiably globally optimal solution by relaxation and semidefinite programming (SDP)~\cite{briales2018certifiably}. 
However, none of the aforementioned methods can find the globally optimal solution efficiently. 
Though BnB search methods~\cite{hartley2009global,kneip2013direct} can obtain global optima in theory, they have the exponential time complexity in the worst case. In~\cite{chesi2009camera,briales2018certifiably}, a theoretical guarantee of the convexification procedures is not provided.
The most related paper to this work is~\cite{briales2018certifiably}, which converts the $N$-point problem of an eigenvalue-based formulation to a QCQP. However, its efficiency is not satisfactory and tightness of its SDR has not been proved.

There are also several works on fundamental matrix estimation for uncalibrated cameras. 
The eight-point method~\cite{hartley1995defence} uses a linear solution, then recovers a valid fundamental matrix by SVD decomposition. This method ignores the rank constraint in the fundamental matrix, thus the solution is not optimal.
In~\cite{hartley1998minimizing}, a method for minimizing an algebraic error was proposed which ensures the rank constraint. However, it does not guarantee global minima. 
In~\cite{zheng2011branch,bugarin2015rank}, the constraint for a fundamental matrix is imposed by setting its determinant as $0$, leading to a cubic polynomial constraint.
In~\cite{chesi2002estimating}, the fundamental matrix estimation problem is reduced to one or several constrained polynomial optimization problems. 
Unfortunately, the aforementioned methods deal with uncalibrated cameras only, where the underlying Euclidean constraints of an essential matrix are not exploited. Thus they cannot be applied to essential matrix estimation.

For both the essential matrix and fundamental matrix, optimal pose estimation can be formulated as a polynomial optimization problem~\cite{mevissen2010sdp}. 
A polynomial optimization problem can be converted to a QCQP. 
In multiple view geometry, SDR for polynomial optimization problems was first studied in~\cite{kahl2007global}. Later, a large number of methods using QCQP formulations were developed in computer vision and robotics. For example, SDR or Lagrangian duality of QCQPs has been used in point set registration~\cite{olsson2008solving}, triangulation~\cite{aholt2012qcqp}, relative pose estimation~\cite{briales2018certifiably}, rotation averaging~\cite{eriksson2020rotation}, pose synchronization~\cite{rosen2019se}, and the Wahba Problem~\cite{yang2019quaternion}. 
However, SDR does not guarantee {\it a priori} that it generates an optimal solution. 
If a QCQP satisfies certain conditions and data noise lies within a critical threshold, a recent study proved that the solution to the SDP optimization algorithm is guaranteed to be globally optimal~\cite{cifuentes2018thesis,cifuentes2017local}. 
A noise threshold that guarantees tightness of SDR is given for the rotation averaging problem~\cite{eriksson2020rotation}. 
For general QCQPs, the global optimality still remains an open problem.

Existing robust estimation methods in geometric vision are mainly based on inlier set maximization~\cite{chin2017maximum} or M-estimators~\cite{huber1981robust}. 
Inlier set maximization was proven to be NP-hard~\cite{chin2020robust}. 
BnB search can be used to find the globally optimal solution~\cite{enqvist2008robust,enqvist2009two,li2009consensus,yang2014optimal,fredriksson2016optimal}, but its efficiency is not satisfactory.
There are a variety of methods to approximately and efficiently solve the inlier set maximization problem. The most popular algorithms belong to a class of randomized sampling techniques, i.e., RANSAC~\cite{Fischler81} and its variants~\cite{chum2003locally,raguram2013usac}\cite{cai2018deterministic}\cite{le2019deterministic}. 
A hybrid method of BnB and mixed integer programming (MIP)~\cite{speciale2017consensus} was proposed to solve the inlier set maximization in relative pose estimation. 
Another alternative robust framework, which is based on M-estimators, has been successfully applied to many fields such as bundle adjustment~\cite{zach2014robust}, registration~\cite{zhou2016fast,yang2020graduated}, and data clustering~\cite{shah2017robust}. 
An important technique to optimize M-estimators is the line process~\cite{black1996unification}, which is also a building block of the robust version of the proposed method.
In the line process, progress is hindered by a lack of an efficient and globally optimal non-minimal solver.
The proposed $N$-point method in this paper can be integrated into the line process to make a robust $N$-point method. 

\section{Formulations of $N$-Point Method}
\label{sec:formulation}

Denote $(\p_i, \p'_i)$ as the $i$-th point correspondence of the same 3D world point from two distinct viewpoints. Point observations $\p_i$ and $\p'_i$ are represented as homogeneous coordinates in normalized image plane\footnote{Bold capital letters denote matrices (e.g., $\E$ and $\R$); bold lower-case letters denote column vectors (e.g., $\mathbf{e}, \mathbf{t}$); non-bold lower-case letters represent scalars (e.g., $\lambda$). 
By Matlab syntax, we use semicolon/comma in matrix concatenation to arrange entries vertically/horizontally. For example, $[[a], [b]] = [a, b]$ and $[[a]; [b]] = \begin{bmatrix} a \\ b \end{bmatrix}$.}. 
Each point in the normalized image plane can be translated into a unique unit bearing vector originating from the camera center. Let ($\f_i$, $\f'_i$) denote a correspondence of bearing vectors pointing at the same 3D world point from two distinct viewpoints, where $\f_i$ represents the observation from the first viewpoint, and $\f'_i$ the observation from the second viewpoint. 
The bearing vectors are determined by $\f_i = \frac{\p_i}{\|\p_i\|}$ and $\f'_i = \frac{\p'_i}{\| \p'_i \|}$. 

The relative pose is composed of rotation $\R$ and translation $\mathbf{t}$. 
Rotation $\R$ transforms vectors from the second into the first frame. 
Translation $\mathbf{t}$ is expressed in the first frame and denotes the position of the second frame with respect to the first one. 
The normalized translation $\mathbf{t} = [t_1, t_2, t_3]^\top$ will be identified with points in the $2$-sphere $\mathcal{S}^2$, i.e., 
\begin{align}
\mathcal{S}^2 \triangleq \{ \mathbf{t} \in \mathbb{R}^3| \mathbf{t}^\top \mathbf{t} = 1 \}. \nonumber
\end{align}
The 3D rotation will be featured as $3\times 3$ orthogonal matrix with positive determinant belonging to the special orthogonal group $\text{SO}(3)$, i.e.,
\begin{align}
\text{SO}(3) \triangleq \{\R \in \mathbb{R}^{3\times 3} | \R^\top \R = \mathbf{I}, \text{det}(\R) = 1\}, \nonumber
\end{align}
where $\mathbf{I}$ is a $3\times 3$ identity matrix.

\subsection{Parametrization for Essential Manifold}
The essential matrix $\E$ is defined as~\cite{hartley2003multiple}
\begin{align}
\E = [\mathbf{t}]_{\times} \R, \label{equ:essential}
\end{align}
where $[\cdot]_\times$ defines the corresponding skew-symmetric matrix for a $3$-dimensional vector, i.e.,
\begin{align}
[\mathbf{t}]_{\times} = 
\begin{bmatrix}
t_1 \\
t_2 \\
t_3
\end{bmatrix}_{\times} = 
\begin{bmatrix}
0 & -t_3 & t_2 \\
t_3 & 0 & -t_1 \\
-t_2 & t_1 & 0
\end{bmatrix}.
\end{align}
Denote the essential matrix $\E$ as
\begin{align}
\E =
\begin{bmatrix}
\e_1^\top \\
\e_2^\top \\
\e_3^\top \\
\end{bmatrix}
=
\begin{bmatrix}
e_{11} & e_{12} & e_{13} \\
e_{21} & e_{22} & e_{23} \\
e_{31} & e_{32} & e_{33} \\
\end{bmatrix}. \nonumber
\end{align}
where $\e_i^\top$ is the $i$-th row of $\E$. Denote its corresponding vector as
\begin{align}
\e &\triangleq \vect(\E) = [\e_1; \e_2; \e_3] \nonumber \\
&= [e_{11}; e_{12}; e_{13}; e_{21}; e_{22}; e_{23}; e_{31}; e_{32}; e_{33}].
\end{align}
where $\vect(\cdot)$ means stacking all the entries of a matrix by row-first order. There is not any essential difference between different orders of entry stacking in this paper. We adopt row-first order to make the notations in Section~\ref{sec:theory} more convenient.

In this paper, an essential matrix set is defined as
\begin{align}
\mathcal{M}_\E \triangleq \{\E \ | \ \E = [\mathbf{t}]_{\times} \R, \exists \ \R \in \text{SO}(3), \mathbf{t} \in \mathcal{S}^2 \}.
\end{align}
This essential matrix set is called \emph{normalized essential manifold}~\cite{ma2001optimization,helmke2007essential,tron2017space}. 
Theorem~\ref{theorem:equivalent1} provides equivalent conditions to define $\mathcal{M}_\E$, which will greatly simplify the optimization in the proposed methods.

\begin{theorem}
	\label{theorem:equivalent1}
	A real $3\times 3$ matrix, $\E$, is an element in $\mathcal{M}_\E$ if and only if there exists a vector $\mathbf{t} \in \mathbb{R}^3$ satisfying the following two conditions:
	\begin{align}
	\emph{\text{(i)}} \ \ \E \E^\top = [\mathbf{t}]_{\times} [\mathbf{t}]_{\times}^\top \quad \emph{\text{and}} \quad \emph{\text{(ii)}} \ \ \mathbf{t}^\top \mathbf{t} = 1.
	\end{align}
\end{theorem}
\begin{proof}
	For \emph{if} direction, first it can be verified that $\text{det}\left( [\mathbf{t}]_{\times} [\mathbf{t}]_{\times}^\top - \sigma \mathbf{I} \right) = -\sigma[\sigma-(t_1^2+t_2^2+t_3^2)]^2 = -\sigma(\sigma-1)^2$. By combining this result with condition (i), we can see that $\E\E^\top$ has an eigenvalue $1$ with multiplicity $2$ and an eigenvalue $0$. According to the definition of singular value, the nonzero singular values of $\E$ are the square roots of the nonzero eigenvalues of $\E\E^\top$. Thus the two nonzero singular values of $\E$ are equal to $1$. 
	According to Theorem~1 in~\cite{faugeras1990motion}, $\E$ is an essential matrix. By combining condition (ii), $\E$ is an element in $\mathcal{M}_\E$.
	
	For \emph{only if} direction, $\E$ is supposed to be an essential matrix $\mathcal{M}_\E$. According to the definition of $\mathcal{M}_\E$, there exists a vector $\mathbf{t}$ satisfying condition (ii). In addition, there exists a rotation matrix $\R$ such that $\E = [\mathbf{t}]_{\times} \R$. It can be verified that $\E \E^\top =([\mathbf{t}]_{\times} \R) ([\mathbf{t}]_{\times} \R)^\top = [\mathbf{t}]_{\times} [\mathbf{t}]_{\times}^\top$, thus condition (i) is also satisfied. 
\end{proof}

It is worth mentioning that a necessary condition for general essential matrix, which is similar to the \emph{only if} direction in Theorem~\ref{theorem:equivalent1}, was presented in~\cite[Proposition~2]{faugeras1990motion} and~\cite[Lemma~7.2]{faugeras1993three}. 
In Theorem~\ref{theorem:equivalent1}, we further prove that this condition is also sufficient and propose a novel parameterization for the normalized essential manifold. 

\subsection{Optimizing Essential Matrix by Minimizing an Algebraic Error}
For noise-free cases, the epipolar constraint~\cite{hartley2003multiple} implies that
\begin{align}
\f_i^\top \E \f'_i = 0.
\label{equ:epipolar}
\end{align}
Under the presence of noise, this constraint does not strictly hold. We pursue the optimal pose by minimizing an algebraic error
\begin{align}
\label{equ:op_new_obj2}
\min_{\E \in \mathcal{M}_\E} & \sum_{i=1}^N (\f_i^\top \E \f'_i)^2.
\end{align}
The algebraic error in objective has been widely used in previous literature~\cite{migita2007evaluation,hartley1998minimizing,chesi2002estimating,chesi2009camera}.

The objective in problem~\eqref{equ:op_new_obj2} can be reformulated as a standard quadratic form
\begin{align}
\label{equ:reform}
\sum_{i=1}^N (\f_i^\top \E \f'_i)^2 = \e^\top \C \e,
\end{align}
where
\begin{align}
\C = \sum_{i=1}^N \left( \f_i \otimes \f'_i \right) 
\left( \f_i \otimes \f'_i 
\right)^\top,
\label{equ:kronecker}
\end{align}
and ``$\otimes$'' means Kronecker product. Note that $\C$ is a Gram matrix, so it is positive semidefinite and symmetric.

\subsection{QCQP Formulations}
By explicitly writing the constraints for the essential manifold $\mathcal{M}_\E$, we reformulate problem~\eqref{equ:op_new_obj2} as
\begin{align}
\label{equ:op_new_obj1}
\min_{\E, \R, \mathbf{t}} & \ \ \e^\top \C \e \\
\text{s.t.}  & \ \ \E = [\mathbf{t}]_{\times} \R, \ \ \R \in \text{SO}(3), \ \ \mathbf{t} \in \mathcal{S}^2, \nonumber
\end{align}
This problem is a QCQP: The objective is positive semidefinite quadratic polynomials; the constraint on the translation vector, $\mathbf{t}^\top \mathbf{t} = 1$, is also quadratic; a rotation matrix $\R$ can be fully defined by $20$ quadratic constraints~\cite{kneip2012finding,briales2018certifiably}; and the relationship between $\E$, $\R$ and $\mathbf{t}$, $\E = [\mathbf{t}]_{\times} \R$, is also quadratic. This formulation has $21$ variables and $30$ constraints. 

According to Theorem~\ref{theorem:equivalent1}, an equivalent QCQP form of minimizing the algebraic error is
\begin{align}
\label{equ:op_new_obj3}
\min_{\E, \mathbf{t}} & \ \ \e^\top \C \e \\
\text{s.t.} & \ \ \E\E^\top = [\mathbf{t}]_{\times} [\mathbf{t}]_{\times}^\top, \quad \mathbf{t}^\top \mathbf{t} = 1. \nonumber
\end{align}
There are $12$ variables and $7$ constraints in this problem. The constraints can be written explicitly as below
\begin{subequations}
\begin{empheq}[left=\empheqlbrace]{align}
	& h_1 =\e_1^\top \e_1 - (t_2^2 + t_3^2) = 0, \label{equ:sub1} \\
	& h_2 =\e_2^\top \e_2 - (t_1^2 + t_3^2) = 0, \label{equ:sub2} \\
	& h_3 =\e_3^\top \e_3 - (t_1^2 + t_2^2) = 0, \label{equ:sub3} \\
	& h_4 =\e_1^\top \e_2 + t_1 t_2 = 0, \label{equ:sub4} \\
	& h_5 =\e_1^\top \e_3 + t_1 t_3 = 0, \label{equ:sub5} \\
	& h_6 =\e_2^\top \e_3 + t_2 t_3 = 0, \label{equ:sub6} \\
	& h_7 = \mathbf{t}^\top \mathbf{t} - 1 = 0. \label{equ:sub7}
\end{empheq}
\end{subequations}
Problems~\eqref{equ:op_new_obj1} and~\eqref{equ:op_new_obj3} are equivalent since their objectives are the same and their feasible regions are equivalent. Both of the two problems are nonconvex.
Appendix~\ref{sec:another_form} provides another equivalent optimization problem.  
In the following, we will only consider problem~\eqref{equ:op_new_obj3} due to its fewer variables and constraints. Moreover, it is homogeneous without the need of homogenization, which makes it simpler than the alternative formulations.

{\bf Remark}: 
Both problems~\eqref{equ:op_new_obj1} and~\eqref{equ:op_new_obj3} are equivalent to an eigenvalue-based formulation~\cite{kneip2013direct,briales2018certifiably}. 
A proof of the equivalence is available in~\cite{briales2018certifiably}, see its supplementary material. 
Since all the mentioned formulations essentially utilize the normalized essential manifold as feasible regions and have the equivalent objectives, all these formulations are equivalent. 
Our formulations~\eqref{equ:op_new_obj1} and~\eqref{equ:op_new_obj3} have the following two advantages: 

(1) Our formulations have fewer variables and constraints. In contrast, the eigenvalue based formulation in~\cite{briales2018certifiably} involves $40$ variables and $536$ constraints.  As shown in the following sections, the simplicity of our formulations will result in much more efficient solvers and enable the proof of tightness and local stability. 

(2) Our formulations are easy to associate priors for each point correspondence by simply introducing weights in the objective. 
For example, we may introduce a weight for each sample by slightly changing the objective tp $\sum_{i=1}^N w_i (\f_i^\top \E \f'_i)^2$,
where $w_i \ge 0$ is the weight for $i$-th observation. 
For general cases in which $w_i \ge 0$, we can keep current formulation by changing the construction of $\C$ to 
\begin{align}
\C = \sum_{i=1}^N w_i \left( \f_i \otimes \f'_i \right) 
\left( \f_i \otimes \f'_i 
\right)^\top.
\label{equ:kronecker2}
\end{align}

\section{Optimization of $N$-Point Method}
\label{sec:sdp}

QCQP is a long-standing problem in optimization literature with many applications. Solving its general case is an NP-hard problem. Global optimization methods for QCQP are typically based on convex relaxations of the problem. There are two main relaxations for QCQP: SDR and the reformulation-linearization technique.
In this paper, we use SDR since it usually has better performance~\cite{anstreicher2009semidefinite} and it is convenient for tightness analysis. 

Let us consider a QCQP in a general form as
\begin{align}
\label{equ:op_obj4}
\min_{\x \in \mathbb{R}^n} & \ \ \x^\top \C_0 \x \\
\text{s.t.} & \ \ \x^\top \A_i \x = b_i, \ \ i = 1, \cdots, m. \nonumber
\end{align}
where $\C_0$, $\A_1, \cdots, \A_m \in \mathcal{S}^n$ and $\mathcal{S}^n$ denotes the set of all real symmetric $n\times n$ matrices. 
In our problem, 
\begin{align}
\x \triangleq [\e; \mathbf{t}]
\end{align}
is a vector stacking all entries in essential matrix $\E$ and translation vector $\mathbf{t}$;  $n = 12$; $m = 7$; $\C_0 = \begin{bmatrix}
\C & \mathbf{0}_{9\times 3} \\
\mathbf{0}_{3\times 9} & \mathbf{0}_{3\times 3}
\end{bmatrix}$; $\A_1,\cdots, \A_7$ correspond to the canonical form $\x^\top \A_i \x$ of Eqs.~\eqref{equ:sub1}$\sim$\eqref{equ:sub7}, respectively. 

A crucial first step in deriving an SDR of problem~\eqref{equ:op_obj4} is to observe that
\begin{align}
\begin{cases}
\x^\top \C_0 \x = \trace(\x^\top \C_0 \x) = \trace(\C_0 \x \x^\top),  \\
\x^\top \A_i \x = \trace(\x^\top \A_i \x) = \trace(\A_i \x \x^\top).
\end{cases} \label{equ:trace2}
\end{align}
It can be seen that both the objective and constraints in problem~\eqref{equ:op_obj4} are linear in matrix $\x\x^\top$. Thus, by introducing a new variable $\X = \x\x^\top$ and noting that $\X = \x\x^\top$ is equivalent to $\X$ being a rank one symmetric positive semidefinite (PSD) matrix, we obtain the following equivalent formulation of problem~\eqref{equ:op_obj4}
\begin{align}
\label{equ:op_obj5}
\min_{\X \in \mathcal{S}^n} & \ \ \trace(\C_0 \X) \\
\text{s.t.} & \ \ \trace(\A_i \X) = b_i, \ \ i = 1, \cdots, m, \nonumber \\
& \ \ \X \succeq \mathbf{0}, \quad \rank(\X) = 1. \nonumber
\end{align}
Here, $\X \succeq \mathbf{0}$ means that $\X$ is PSD. Solving rank constrained semidefinite programs (SDPs) is NP-hard~\cite{vandenberghe1996semidefinite}. 
SDR drops the rank constraint to obtain the following relaxed version of problem~\eqref{equ:op_obj5}
\begin{align}
\label{equ:op_obj6}
\min_{\X \in \mathcal{S}^n} & \ \ \trace(\C_0 \X) \\
\text{s.t.} & \ \ \trace(\A_i \X) = b_i, \ \ i = 1, \cdots, m, \nonumber \\
& \ \ \X \succeq \mathbf{0}. \nonumber
\end{align}
Problem~\eqref{equ:op_obj6} turns out to be an instance of SDP~\cite{vandenberghe1996semidefinite}, 
which belongs to convex optimization and can be readily solved using primal-dual interior point methods~\cite{ye1997interior}. 
Its dual problem is 
\begin{align}
\label{equ:op_dual}
\max_{\boldsymbol{\lambda}} & \ \ \mathbf{b}^\top \boldsymbol{\lambda} \\
\text{s.t.} & \ \ \Q(\boldsymbol{\lambda}) = \C_0 - \sum_{i=1}^m \lambda_i \A_i \succeq 0, \nonumber
\end{align}
where $\mathbf{b} = [b_1, \cdots, b_m]^\top$, $ \boldsymbol{\lambda} = [\lambda_1, \cdots, \lambda_m]^\top \in \mathbb{R}^m$. 
Problem~\eqref{equ:op_dual} is called the \emph{Lagrangian dual problem} of problem~\eqref{equ:op_obj4}, and $\Q(\boldsymbol{\lambda})$ is the Hessian of the Lagrangian. 
In our problem, $b_i = 0$ for $i = 1, \cdots, 6$; $b_7 = 1$; and $\mathbf{b}^\top \boldsymbol{\lambda} = \lambda_7$.

In summary, the relations between different formulations are demonstrated by Fig.~\ref{fig:relation}.

\begin{figure}[htbp]
	\begin{center}
		\includegraphics[width=0.99\linewidth]{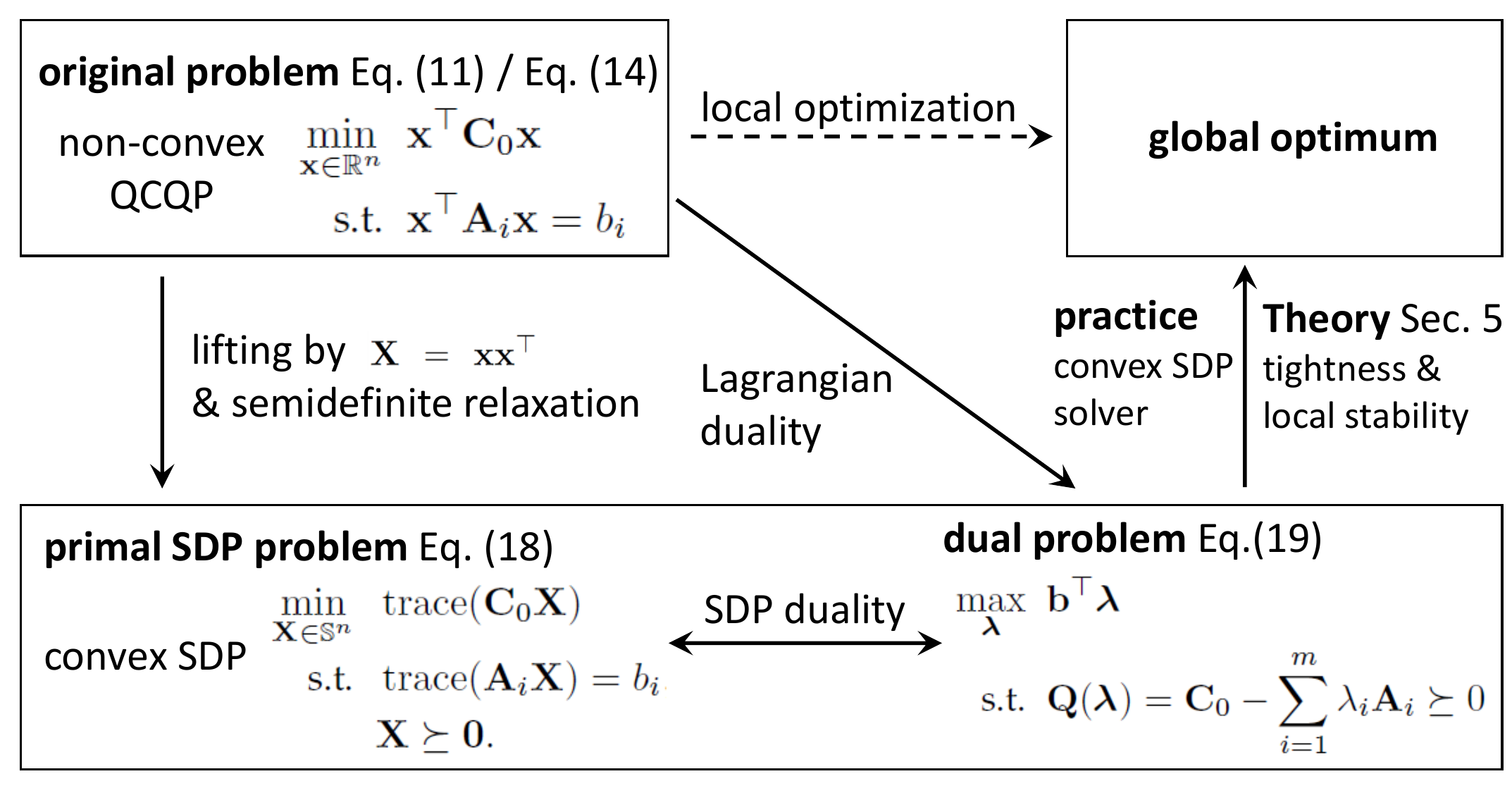}
	\end{center}
	\vspace{-0.15in}
	\caption{An overview of relations between different formulations.}
	\label{fig:relation}
\end{figure}

\subsection{Essential Matrix and Relative Pose Recovery}

Once the optimal $\X^\star$ of the SDP primal problem~\eqref{equ:op_obj6} has been obtained by an SDP solver, we need to recover the optimal essential matrix $\E^\star$. 
Denote $\X_e$ as the top-left $9\times 9$ submatrix of $\X$; and denote $\X_t$ as the bottom-right $3\times 3$ submatrix of $\X$, i.e., $\X_e \triangleq \X_{[1:9, 1:9]}$ and $\X_t \triangleq \X_{[10:12, 10:12]}$.
Empirically, we found that the largest singular value of $\X_e^{\star}$ is near $2$ and others are close to zero. 
It is common to set the rank of a matrix as the number of singular values larger than a threshold. In our method, the threshold depends on the accuracy of SDP solver (usually $10^{-7} \sim 10^{-5}$), leading to $\rank(\X_e^{\star}) = 1$. 
Denote the eigenvector that corresponding to the nonzero eigenvalue of $\X_e^{\star}$ as $\mathbf{e}^\star$, then the optimal essential matrix is recovered by
\begin{align}
\E^\star = \mat(\mathbf{e}^\star, [3, 3]),
\end{align}
where $\mat(\mathbf{e}, [r, c])$ means reshape the vector $\mathbf{e}$ to an $r\times c$ matrix by row-first order.

After the essential matrix has been obtained, we can recover the rotation and translation by the standard method in literature~\cite{hartley2003multiple}. 
A recent work proved that the rotation matrix can be accurately recovered from the essential matrix for pure rotation scenarios~\cite{cai2018equivalent}. Moreover, a statistic, the mean of $\left\{\frac{\x_i \times \R^\star \x'_i}{\|\x_i\| \|\x'_i\|}\right\}_{i=1}^N$,  was proposed to identify the pure rotation scenarios. 

In Section~\ref{sec:conditions}, the theoretical guarantee of such a pose recovery method will be provided. 
In Section~\ref{sec:theory}, the proof of tightness and local stability that guarantees the global optimality will be provided.   
The outline of $N$-point method is shown in Algorithm~\ref{alg:alg1}.

\begin{algorithm}[tbp]
	\caption{Weighted $N$-Point Method} \label{alg:alg1}
	\KwIn{observations $\{ ( \f_i, \f'_i ) \}^{N}_{i=1}$, (optional) weight $\{w_i\}_{i=1}^N$}
	\KwOut{Essential matrix $\E^\star$, rotation $\R^\star$, translation $\mathbf{t}^\star$, identification of pure rotation.}
	Construct $\C$ by Eq.~\eqref{equ:kronecker} for unweighted version or Eq.~\eqref{equ:kronecker2} for weighted version; $\C_0 = \begin{bmatrix}
	\C & \mathbf{0}_{9\times 3} \\
	\mathbf{0}_{3\times 9} & \mathbf{0}_{3\times 3}
	\end{bmatrix}$\;
	Construct $\{\A_i\}_{i=1}^7$ in problem~\eqref{equ:op_obj4} which is independent of input\;
	Obtain $\X^\star$ by solving SDP problem~\eqref{equ:op_obj6} or its dual problem~\eqref{equ:op_dual}\;
	Assert that $\rank(\X_e^\star) = \rank(\X_t^\star) = 1$\;
	$\E^\star = \mat(\mathbf{e}^\star, [3, 3])$, where $\mathbf{e}^\star$ is the eigenvector corresponding to the largest eigenvalue of $\X_e^{\star}$\;
	Decompose $\E^\star$ to obtain $\R^\star$ and $\mathbf{t}^\star$\;
	Identify whether pure rotation occurs.
\end{algorithm}

\subsection{Necessary and Sufficient Conditions for Global Optimality}
\label{sec:conditions}
The following Theorem~\ref{theorem:optimality} provides a theoretical guarantee for the proposed pose recovery method. 

\begin{theorem}
	\label{theorem:optimality}
	For QCQP~\eqref{equ:op_new_obj3},
	its SDR is tight if and only if: the optimal solution $\X^\star$ to its primal SDP problem~\eqref{equ:op_obj6} satisfies $\rank(\X_e^\star) = \rank(\X_t^\star) = 1$. 
\end{theorem}
\begin{proof}
	First, we prove the \emph{if} part. Note that $\X_e^\star$ and $\X_t^\star$ are real symmetric matrices because they are in the feasible region of the primal SDP. 
	In addition, it is given that $\rank(\X_e^\star) = \rank(\X_t^\star) = 1$, thus there exist two vectors $\e^\star$ and $\mathbf{t}^\star$ satisfying 
	$\e^\star ({\e^\star})^\top = \X_e^\star$ and $\mathbf{t}^\star (\mathbf{t}^\star)^\top = \X_t^\star$.
	Since the constraints in problem~\eqref{equ:op_new_obj3}  do not include any cross term between $\E$ and $\mathbf{t}$, the intersection part of $\E$ and $\mathbf{t}$ in matrix $\A_i$ is zero in SDP problem. 
	By substituting $\e^\star$ and $\mathbf{t}^\star$ into Eq.~\eqref{equ:trace2}, it can be verified that $\e^\star$ and $\mathbf{t}^\star$ satisfy the constraints in primal problem~\eqref{equ:op_new_obj3}. 
	Now we can see that $\X^\star$  and its uniquely determined derivatives ($\e^\star$ and $\mathbf{t}^\star$) are feasible solutions for the semidefinite relaxation problem and the primal problem, respectively. Thus the relaxation is tight.
	
	Then we prove the \emph{only if} part. Since the semidefinite relaxation is tight, we have $\rank(\X^\star) = 1$. Then $\rank(\X_e^\star) \le 1$ and $\rank(\X_t^\star) \le 1$. Since $\X_e^\star$ and $\X_t^\star$ cannot be zero matrices (otherwise $\X^\star$ is not in the feasible region), the equalities should hold, i.e., $\rank(\X_e^\star)  = \rank(\X_t^\star) = 1$.
\end{proof}

Theorem~\ref{theorem:optimality} provides a necessary and sufficient condition to recover the globally optimal solution for the primal problem. 
Empirically, the optimal $\X^\star$ by the SDP problem always satisfies this condition. Specifically, the optimal $\X^\star$ has the following structure
\begin{align}
\X^\star = 
\begin{bmatrix}
\X_e^\star & \Diamond \\
\Diamond & \X_t^\star
\end{bmatrix}
= 
\begin{bmatrix}
\mathbf{e}^\star (\mathbf{e}^\star)^\top & \Diamond \\
\Diamond & \mathbf{t}^\star (\mathbf{t}^\star)^\top
\end{bmatrix}.
\label{equ:xstar}
\end{align}
The $\Diamond$ parts could be arbitrary matrices making $\X^\star$ symmetric.

{\bf Remark}:
The block diagonal structure of $\X^\star$ in Eq.~\eqref{equ:xstar} is caused by the sparsity pattern of the problem. The \emph{aggregate sparsity pattern} in our SDP problem, which is the union of individual sparsity patterns of the data matrices, $\{\C_0, \A_1, \cdots, \A_m\}$,  includes two cliques: one includes the $1\sim 9$-th entries of $\x$, and the other includes the $10\sim 12$-th entries of $\x$, see Fig.~\ref{fig:sparsity}(a). There is no common node in these two cliques, see Fig.~\ref{fig:sparsity}(b). The \emph{chordal decomposition} theory of the sparse SDPs can explain the structure of $\X^\star$ well. The interested reader may refer to~\cite{fukuda2001exploiting,vandenberghe2015chordal} for more details.

\begin{figure}[t]
	\begin{center}
	\subfigure[]
	{
		\includegraphics[width=0.4\linewidth]{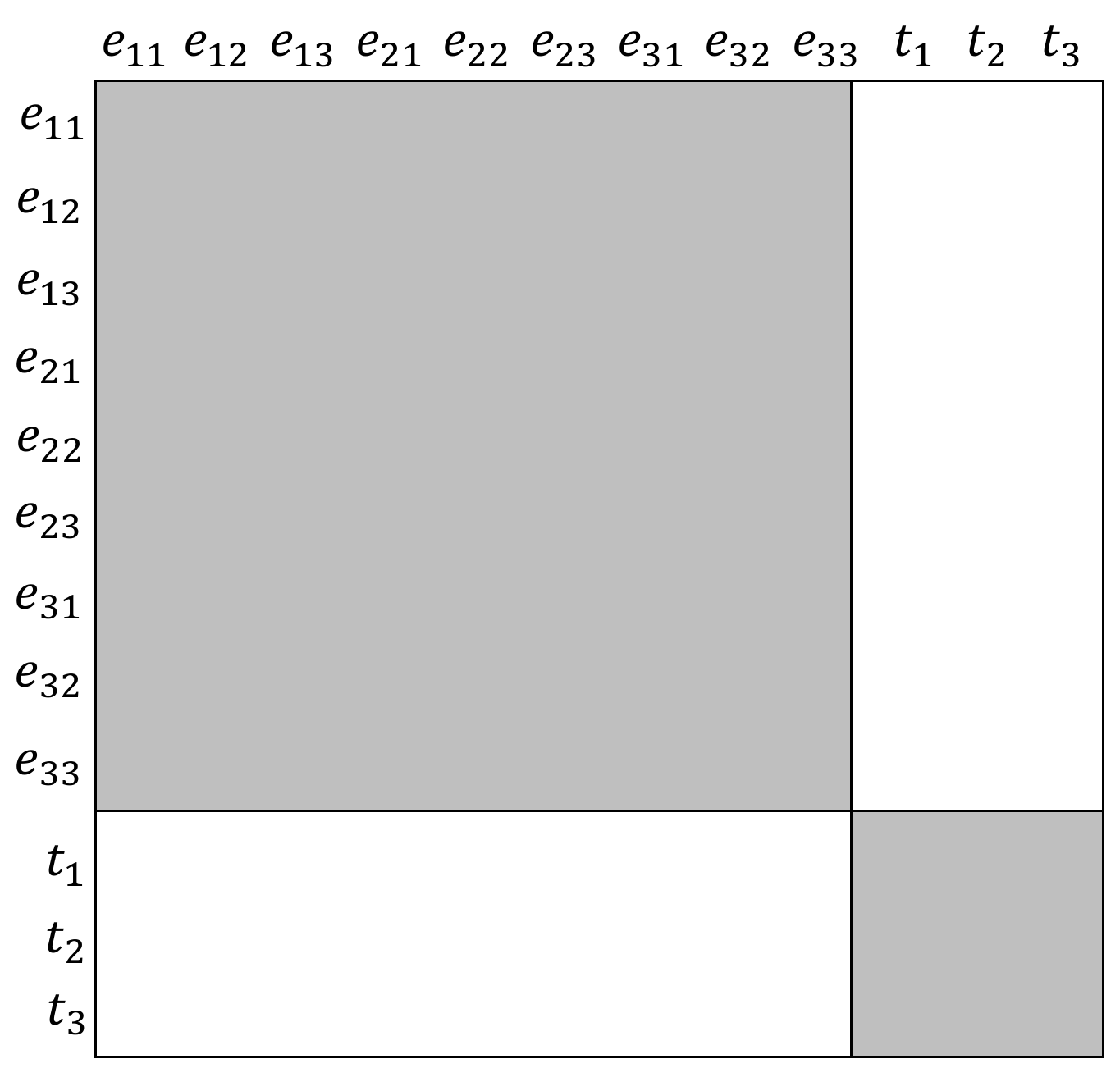}
	}
	\subfigure[]
	{
		\includegraphics[width=0.54\linewidth]{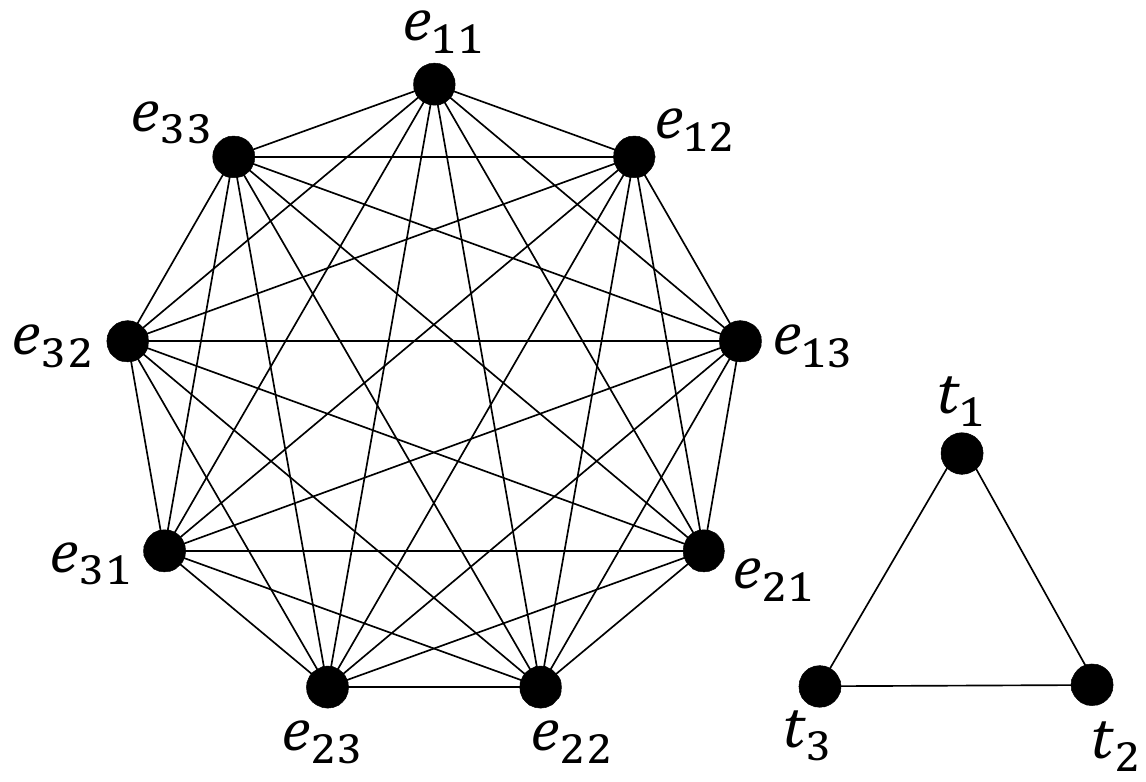}
	}
	\end{center}
	\vspace{-0.15in}
	\caption{The sparsity of our SDP problem. (a) Aggregate sparsity pattern. White parts correspond to zeros, and gray parts correspond to nonzeros. There are two diagonal blocks in this pattern. (b) Chordal decomposition of the corresponding graph. The graph contains $12$ nodes and it can be decomposed into $2$ distinct maximal cliques.}
	\label{fig:sparsity}
\end{figure}

\subsection{Time Complexity}

First, we consider the time complexity of problem construction. The construction of the optimization problem (i.e., calculating $\C$ in Eq.~\eqref{equ:kronecker} or Eq.~\eqref{equ:kronecker2}) is linear with the number of point correspondences, so its time complexity is $\mathcal{O}(N)$. 
It is worth mentioning that optimization is independent of the number of point correspondences once the problem has been constructed.

Second, we discuss the time complexity of SDP optimization. 
Most SDP solvers use an interior-point algorithm. The SDP problem~\eqref{equ:op_obj6} can be solved with a worst case complexity of 
\begin{align}
\mathcal{O}(\max(m,n)^4 n^{1/2} \log(1/\epsilon))
\end{align}
flops given a solution accuracy $\epsilon>0$~\cite{ye1997interior}. 
It can be seen that the time complexity can be largely reduced given a smaller variable number $n$ and constraint number $m$.  
Since our formulations have much fewer variables and constraints than previous work~\cite{briales2018certifiably}, they own much lower time complexity.

Finally, we discuss the time complexity of pose recovery. In our method, the essential matrix is recovered by finding the eigenvector corresponding to the largest eigenvalue of a $9\times 9$ matrix $\X_e^\star$. Thus the time complexity of pose recovery is $\mathcal{O}(1)$.

\section{Tightness and Local Stability of $N$-Point Method}
\label{sec:theory}

In this section, we prove tightness and local stability of the SDR for our problem. 
The readers who are not interested in theory can safely skip this section. 
To understand the proofs in this section, preliminary knowledge about convex optimization, manifold, and algebraic geometry is necessary. We recommend the readers to refer to~\cite[Chapter~6]{cifuentes2018thesis}\cite{cifuentes2017local} for more details.

In the following, the bar on a symbol stands for a value under the noise-free case. In other words, it represents ground truth. For example, $\bar{\mathbf{C}}$ denotes the matrix in objective constructed by noise-free observations; $\bar{\x}$ is the optimal state estimation from noise-free observations; $\bar{\mathbf{t}} = [\bar{t}_1, \bar{t}_2, \bar{t}_3]^\top$ is the optimal translation estimation from noise-free observations.

\subsection{Tightness of the Semidefinite Relaxation}
\label{sec:tightness}

In this subsection, we prove that the SDR is tight given noise-free observations. The proof is mainly based on Lemma~2.4 in~\cite{cifuentes2017local}.

\begin{lemma}
	If point correspondences $\{(\f_i, \f'_i)\}_{i=1}^N$ are noise-free, the matrix $\bar{\C}$ in Eq.~\eqref{equ:kronecker} satisfies that $\rank(\bar{\C}) \le \min(N, 8)$, where $N$ is the number of point correspondences. The equality holds except for degenerate configurations including  points on a ruled quadratic, points on a plane, and no translation (explanation of these degeneracies can be found in~\cite{maybank1990projective}).
\end{lemma}
\begin{proof}
	(i) When $N \le 8$, from the construction of $\C$ in Eq.~\eqref{equ:kronecker}, each point correspondence adds a rank-$1$ matrix to the Gram matrix $\C$. The rank-$1$ matrices are linear independent except for the degenerate cases~\cite{maybank1990projective}. Thus the rank of $\C$ is $N$ for non-degenerate cases. 
	When a degeneracy occurs, there will be linear dependence between these rank-1 matrices, and the rank will be below $N$. 
	(ii) When $N > 8$, there exists the stacked vector $\bar{\e}$ of an essential matrix satisfying that $\bar{\e}^\top \bar{\C} \bar{\e} = 0$, so the upper limit of $\rank(
	\bar{\C})$ is $8$. 
	Then we complete the proof by combining these two properties.
\end{proof}
Given a set of point correspondences, if $N \ge 8$ after excluding points that belong to planar degeneracy, the rank of $\bar{\C}$ is $8$. This principle is the basis of the eight-point method~\cite{hartley1995defence}. 
In our methods, we do not need to distinguish the points that belong to the degenerate configurations, so the rank-$8$ assumption in the following text can be easily satisfied in $N$-point problem in which $N \gg 8$.

\begin{lemma}
	\label{lemma:psd}
	Let $\C \in \mathbb{R}^{n\times n}$ be positive semidefinite.  
	If $\x^\top \C \x = 0$ for a given vector $\x$, then $\C \x = \mathbf{0}$.
\end{lemma}
\begin{proof}
	Since $\C$ is positive semidefinite, its eigenvalues are non-negative. Suppose the rank of $\C$ is $r$. The eigenvalues can be listed as $\sigma_1 \ge \sigma_2 \ge \cdots \ge \sigma_r > 0 = \sigma_{r+1} = \cdots = \sigma_n$. Denote the $i$-th eigenvector as $\x_i$. The eigenvectors are orthogonal to each other, i.e., $\x_i^\top \x_j = 0$ when $i \neq j$. Vector $\x$ can be expressed as $\x = \sum_{i=1}^n a_i \x_i$. Thus, $\C \x = \C \sum_{i=1}^n a_i \x_i = \sum_{i=1}^n a_i \sigma_i \x_i$, and $\x^\top \C \x = \sum_{i=1}^n a_i \x_i^\top \cdot \sum_{i=1}^n a_i \sigma_i \x_i = \sum_{i=1}^r \sigma_i a_i^2 \|\x_i\|^2$. Given $\x^\top \C \x = 0$, we have $a_i = 0$ for $i = 1, \cdots, r$. Thus $\C \x = \sum_{i=1}^n a_i \sigma_i \x_i = \mathbf{0}$ is obtained. 
\end{proof}

\begin{lemma}
	\label{lemma:noisefree}
	If point correspondences $\{(\f_i, \f'_i)\}_{i=1}^N$ are noise-free, there is zero-duality-gap between problem~\eqref{equ:op_new_obj3} and its Lagrangian dual problem~\eqref{equ:op_dual}.
\end{lemma}
\begin{proof}
	Our proof is an application of the Lemma~2.4 in~\cite{cifuentes2017local}. Let $\bar{\x} = [\bar{\e}; \bar{\mathbf{t}}]$ be a feasible point in the primal problem, where $\bar{\e}$ and $\bar{\mathbf{t}}$ are ground truth. And let $\boldsymbol{\lambda} = \mathbf{0}$  be a feasible point in its Lagrangian dual problem~\eqref{equ:op_dual}. The three conditions needed in Lemma~2.4 in~\cite{cifuentes2017local} are satisfied: (i) Primal feasibility. By substituting $\bar{\x}$ in the primal problem, the constraints are satisfied since $\bar{\x}$ is ground truth and the point correspondences are noise-free. 
	(ii) Dual feasibility. $\Q(\boldsymbol{\lambda}) = \C_0 - \sum_{i=1}^m \lambda_i \A_i = 
	\begin{bmatrix}
	\bar{\C} & \mathbf{0}_{9\times 3} \\
	\mathbf{0}_{3\times 9} & \mathbf{0}_{3\times 3}
	\end{bmatrix} \succeq 0$.
	(iii) Lagrangian multiplier. It satisfies that $(\C_0 - \sum_{i=1}^m \lambda_i \A_i) \bar{\x} = \begin{bmatrix} \bar{\C} \bar{\e} \\ \mathbf{0}_{3\times 1} \end{bmatrix}$. Since $\bar{\e}$ is the ground truth, $\bar{\e}^\top \C \bar{\e} = 0$ according to Eq.~\eqref{equ:reform}. Recall that $\bar{\C}$ is a Gram matrix and thus it is positive semidefinite. 
	According to Lemma~\ref{lemma:psd}, $\bar{\C} \bar{\e} = \mathbf{0}$ is obtained. 
\end{proof}
The zero-duality gap still holds for the case of noisy observations. A proof is provided in Appendix~\ref{sec:another_lemma}.

\subsection{Local Stability of the Semidefinite Relaxation}
\label{sec:stablility}
In this subsection, we prove that the SDR has local stability near noise-free observations. In other words, our QCQP formulation has a zero-duality-gap regime when its observations are perturbed (e.g., with noise in the case of sensor measurements). 
The proof is based on Theorem~5.1 in~\cite{cifuentes2017local}. 
Following~\cite{cifuentes2017local}, 
we will use the following notations in the remains of this section to make notation simplicity.
\begin{itemize}
	\item $\bar{\theta} \in \Theta$ is a zero-duality-gap parameter. In our problem, $\bar{\theta} = \{ \bar{\C} \}$.
	\item Given noise-free observations, let $\bar{\x} \in \mathbb{R}^n$ be optimal for the primal problem, and $\bar{\boldsymbol{\lambda}} \in \mathbb{R}^m$ be optimal for the dual problem. In our problem, $n = 12$ and $m = 7$. According to the proof procedure of Lemma~\ref{lemma:noisefree}, we have $\bar{\x} = [\bar{\e}; \bar{\mathbf{t}}]$ and $\bar{\boldsymbol{\lambda}} = \mathbf{0}$.
	\item Denote $\bar{\mathbf{Q}} \triangleq \mathbf{Q}_{\bar{\theta}}(\bar{\boldsymbol{\lambda}}) \in \mathcal{S}^n$ as the Hessian of the Lagrangian at $\bar{\theta}$. 
	In our problem, 
	$\bar{\mathbf{Q}} = 
	\begin{bmatrix}
	\bar{\C} & \mathbf{0}_{9\times 3} \\
	\mathbf{0}_{3\times 9} & \mathbf{0}_{3\times 3}
	\end{bmatrix}
	\succeq 0
	$.
	\item Denote $X_{\theta} \triangleq \{\x \in \mathbb{R}^n | h_{i|\theta(\x)} = 0, i = 1, \cdots, m \}$ as the primal feasible set given $\theta$, and denote $\bar{X} \triangleq X_{\bar{\theta}}$.
	\item Denote $\mathbf{h}(\x) = [h_1(\x), \cdots, h_m(\x)]$.
\end{itemize}

In QCQP~\eqref{equ:op_new_obj3}, the objective $\e^\top \C \e$ is convex with $\e$. However, the presence of the auxiliary variables $\mathbf{t}$ makes the objective is not strictly convex. 
Theorem~5.1 in~\cite{cifuentes2017local} provides a framework to prove the local stability for such kinds of problems.

\begin{theorem}[Theorem~5.1 in~\cite{cifuentes2017local}]
	\label{theorem:main}
	Assume that the following $4$ conditions are satisfied:
	
	RS (restricted Slater): There exists $\boldsymbol{\mu} \in \mathbb{R}^m$ such that $\boldsymbol{\mu}^\top \nabla \mathbf{h}_{\bar{\theta}}(\bar{\x}) = 0$ and $(\sum_{i=1}^m \mu_i \mathbf{A}_{i|\bar{\theta}}) |_V \succ 0$, where $V \triangleq \{\mathbf{v} \in \mathbb{R}^n | \bar{\mathbf{Q}} \mathbf{v} = \mathbf{0}, \bar{\x}^\top \mathbf{v} = 0 \}$.
	
	R1 (constraint qualification): Abadie constraint qualification $\acq_{\bar{\X}}(\bar{\x})$ holds.
	
	R2 (smoothness): $\mathcal{W} \triangleq \{(\theta, \x) | \mathbf{h}_\theta (\x) = \mathbf{0}\}$ is a smooth manifold nearby $\bar{w} \triangleq (\bar{\theta}, \bar{\x})$, and $\dim_{\bar{w}} \mathcal{W} = \dim \Theta + \dim_{\bar{\x}} \bar{X}$.
	
	R3 (not a branch point): $\bar{\x}$ is not a branch point of $\bar{X}$ with respect to $\mathbf{v} \mapsto \bar{\Q} \mathbf{v}$.
	
	Then the SDR relaxation is tight when $\theta$ is close enough to $\bar{\theta}$. Moreover, the QCQP has a unique optimal solution $\x_\theta$, and the SDR problem has a a unique optimal solution $\x_\theta \x_\theta^\top$.
\end{theorem}

Among the four conditions in Theorem~\ref{theorem:main}, the \emph{RS (restricted Slater)} is the main
assumption, which is related to the convexity of the Lagrangian function. It corresponds to the strict feasibility of an SDP. 
\emph{R1}$\sim$\emph{R3} are regularity assumptions which are related to the continuity of the
Lagrange multipliers. 
In the following text, we prove that the restricted Slater and \emph{R3} are satisfied in  problem~\eqref{equ:op_new_obj3}, which builds an important foundation to prove the local stability. 
In our problem, $\A_i$ and $h_i$ are independent of $\theta$, thus $\mathbf{A}_{i|\bar{\theta}} = \A_i$ and $h_{i|\bar{\theta}} = h_i$.
\begin{lemma}[Restricted Slater]
	\label{lemma:slater}
	For QCQP~\eqref{equ:op_new_obj3}, suppose $\rank(\bar{\C}) = 8$.  Then there exists $\boldsymbol{\mu} \in \mathbb{R}^m$ such that $\boldsymbol{\mu}^\top \nabla \mathbf{h}_{\bar{\theta}}(\bar{\x}) = 0$ and $(\sum_{i=1}^m \mu_i \mathbf{A}_{i|\bar{\theta}}) |_V \succ 0$, where $V \triangleq \{\mathbf{v} \in \mathbb{R}^n | \bar{\mathbf{Q}} \mathbf{v} = \mathbf{0}, \bar{\x}^\top \mathbf{v} = 0 \}$.
\end{lemma}
\begin{proof}
	For constraint of Eqs.~\eqref{equ:sub1}$\sim$\eqref{equ:sub7}, its gradient is
	\begin{align}
	\label{equ:grad_h}
	\nabla \h_{\bar{\theta}}(\bar{\x}) = &
	\begin{bmatrix}
	\nabla_{\e_1} h_1 & \nabla_{\e_2} h_1  & \nabla_{\e_3} h_1  & \nabla_{\mathbf{t}} h_1 \\
	\vdots & \vdots & \vdots & \vdots \\
	\nabla_{\e_1} h_m & \nabla_{\e_2} h_m  & \nabla_{\e_3} h_m  & \nabla_{\mathbf{t}} h_m
	\end{bmatrix}_{|\bar{\theta}}
	\nonumber \\
	 =& \begin{bmatrix}
	2\bar{\e}_1^\top & \mathbf{0} & \mathbf{0} & 0 & -2\bar{t}_2 & -2\bar{t}_3 \\
	\mathbf{0} & 2\bar{\e}_2^\top & \mathbf{0} & -2\bar{t}_1 & 0 & -2\bar{t}_3 \\
	\mathbf{0} & \mathbf{0} & 2\bar{\e}_3^\top & -2\bar{t}_1 & -2\bar{t}_2 & 0 \\
	\bar{\e}_2^\top & \bar{\e}_1^\top & \mathbf{0} & \bar{t}_2^\top & \bar{t}_1 & 0 \\
	\bar{\e}_3^\top & \mathbf{0} & \bar{\e}_1^\top & \bar{t}_3 & 0 & \bar{t}_1 \\
	\mathbf{0} & \bar{\e}_3^\top & \bar{\e}_2^\top & 0 & \bar{t}_3 & \bar{t}_2 \\
	\mathbf{0} & \mathbf{0} & \mathbf{0} & 2\bar{t}_1 & 2\bar{t}_2 & 2\bar{t}_3
	\end{bmatrix}.
	\end{align}
	Let 
	\begin{align}
	\boldsymbol{\mu} = -\begin{bmatrix}
	\frac{1}{2}\bar{t}_1^2, \ \frac{1}{2}\bar{t}_2^2, \  \frac{1}{2}\bar{t}_3^2, \ \bar{t}_1 \bar{t}_2, \ \bar{t}_1 \bar{t}_3, \ \bar{t}_2 \bar{t}_3, \ 0
	\end{bmatrix}^\top.
	\end{align}
	Note that for noise-free cases we have
	\begin{align}
	\bar{\mathbf{t}}^\top \bar{\E} = \bar{\mathbf{t}}^\top ([\bar{\mathbf{t}}]_{\times} \bar{\R}) = \mathbf{0}, \nonumber
	\end{align}
	or equivalently
	\begin{align}
	\label{equ:t_e_zero}
	\bar{t}_1 \bar{\e}_1 + \bar{t}_2 \bar{\e}_2 + \bar{t}_3 \bar{\e}_3 = \mathbf{0}.
	\end{align}
	By combining Eqs.~\eqref{equ:grad_h}$\sim$\eqref{equ:t_e_zero}, it can be verified that $\boldsymbol{\mu}^\top \nabla \mathbf{h}_{\bar{\theta}}(\bar{\x}) = 0$.
	
	It remains to check the positivity condition. According to the definition of $V$, $\begin{bmatrix}
	\bar{\Q} \\
	\bar{\x}^\top
	\end{bmatrix} \mathbf{v} = \mathbf{0}
	\Leftrightarrow
	\begin{bmatrix}
	\bar{\C} & \mathbf{0}  \\
	\bar{\e}^\top & \bar{\mathbf{t}}^\top 
	\end{bmatrix}
	\mathbf{v} = \mathbf{0}$.
	Since $\bar{\C}$ is constructed by noise-free observations, $\bar{\C} \bar{\e} = 0$. In other words, $\bar{\e}$ is orthogonal to the space spanned by $\bar{\C}$. It is given that $\rank(\bar{\C}) = 8$, thus $\rank\left(\begin{bmatrix}
	\bar{\C} \\ \bar{\e}^\top
	\end{bmatrix}\right) = 9$. Considering $\mathbf{v}$ as a non-trivial solution of a homogeneous linear equation system, $\mathbf{v}$ can be expressed by a coordinate system $\mathbf{v} = \begin{bmatrix}
	\mathbf{0}_{9\times 1} \\
	\mathbf{t}
	\end{bmatrix}$. 
	
	It can be seen that only $10\sim 12$-th entries in coordinate system $\mathbf{v}$, which correspond to $\mathbf{t}$, are nonzero. Take Hessian for variable $\mathbf{t}$ and calculate the linear combination with coefficient $\boldsymbol{\mu}$, then we have
	\begin{align}
	& \mathcal{A}(\boldsymbol{\mu})  \triangleq \sum_{i=1}^m \mu_i \nabla^2_{\mathbf{t}\mathbf{t}}h_{i|\bar{\theta}}(\bar{\x}) \nonumber \\
	= &\begin{bmatrix}
	\bar{t}_2^2+\bar{t}_3^2 & -\bar{t}_1 \bar{t}_2 & -\bar{t}_1 \bar{t}_3 \\
	-\bar{t}_1 \bar{t}_2 & \bar{t}_1^2+\bar{t}_3^2 & -\bar{t}_2 \bar{t}_3 \\
	-\bar{t}_1 \bar{t}_3 & -\bar{t}_2 \bar{t}_3 & \bar{t}_1^2+\bar{t}_2^2
	\end{bmatrix} 
	= [\bar{\mathbf{t}}]_{\times} [\bar{\mathbf{t}}]_{\times}^\top = \bar{\E} \bar{\E}^\top.
	\end{align}
	Recall that in the proof of Theorem~\ref{theorem:equivalent1} we have proved that the eigenvalues of $\bar{\E}\bar{\E}^\top$ are $1$, $1$, and $0$. 
	So the eigenvalues of $\mathcal{A}(\boldsymbol{\mu})$ are $1$, $1$, and $0$. 
	In addition, it can be verified that $\bar{\mathbf{t}} = [\bar{t}_1, \bar{t}_2, \bar{t}_3]^\top$ is the normalized eigenvector corresponding to eigenvalue $0$ of $\mathcal{A}(\boldsymbol{\mu})$. Hence $V$ is the orthogonal complement of $\bar{\mathbf{t}}$.
	
	For any vector $\mathbf{v}\in V  \backslash \{\mathbf{0}\}$, its $10\sim 12$-th entries are orthogonal to $\bar{\mathbf{t}}$, so $\mathbf{v}_{[10:12]}^\top \mathcal{A}(\boldsymbol{\mu}) \mathbf{v}_{[10:12]}$ is strictly positive. It follows that $(\sum_{i=1}^m \mu_i \mathbf{A}_{i|\bar{\theta}})|_V = \mathcal{A}(\boldsymbol{\mu})|_{(\bar{\mathbf{t}})^\perp} \succ 0$.
\end{proof}

\begin{definition}[Branch Point~\cite{cifuentes2017local}]
	Let $\pi: \mathbb{R}^n \rightarrow \mathbb{R}^k$ be a linear map: $\mathbf{v} \mapsto \bar{\Q} \mathbf{v}$. Let $\bar{X} \subseteq \mathbb{R}^n$ be the zero set of the equation system $\mathbf{h}(\x) = [h_1(\x), \cdots, h_m(\x)]$, and let $T_\x \bar{X} \triangleq \ker(\nabla \mathbf{h}(\x))$ denote the tangent space of $\bar{X}$ at $\x$. We say that $\x$ is a branch point of $X$ with respect to $\pi$ if there is a nonzero vector $\mathbf{v} \in T_\x \ \bar{X}$ with $\pi(\mathbf{v}) = \mathbf{0}$.
\end{definition}

\begin{lemma}
	\label{lemma:branch_point}
	For QCQP~\eqref{equ:op_new_obj3}, suppose $\rank(\bar{\C}) = 8$. 
	Then $\bar{\x} = [\bar{\x}; \bar{\mathbf{t}}]$ is not a branch point of $\bar{X}$ with respect to the mapping $\pi: \mathbf{v} \mapsto \bar{\Q} \mathbf{v}$.
\end{lemma}
\begin{proof}
	It can be verified that $\pi(\mathbf{v}) = \mathbf{0} \Leftrightarrow \bar{\Q}\mathbf{v} = \mathbf{0} \Leftrightarrow \begin{bmatrix}
	\bar{\C} & \mathbf{0}_{9\times 3} \\
	\mathbf{0}_{3\times 9} & \mathbf{0}_{3\times 3}
	\end{bmatrix} \mathbf{v} = \mathbf{0} \Leftrightarrow \mathbf{v} = \begin{bmatrix}
	c \bar{\e} \\
	\mathbf{t}
	\end{bmatrix}$, where $c$ and $\mathbf{t}$ are free parameters. The last equivalence takes advantage of $\rank(\bar{\C}) = 8$. 
	If $\bar{\x}$ is a branch point, there should exist a nonzero vector $\mathbf{v} \in \ker(\nabla \mathbf{h}(\bar{\x}))$, i.e., $\nabla \mathbf{h}(\bar{\x}) \mathbf{v} = \mathbf{0}$. We will prove that such nonzero $\mathbf{v}$ does not exist. By substituting $\mathbf{v} = \begin{bmatrix}
	c \bar{\e} \\
	\mathbf{t}
	\end{bmatrix}$ into equation $\nabla \mathbf{h}_{\bar{\theta}}(\bar{\x}) \mathbf{v} = \mathbf{0}$ (see Eq.~\eqref{equ:grad_h}), we obtain a homogeneous linear system with unknowns $\mathbf{t} = [t_1, t_2, t_3]^\top$ and $c$
	\begin{subequations}
		\begin{empheq}[left=\empheqlbrace]{align}
			& \bar{\e}_1^\top \bar{\e}_1 c - \bar{t}_2 t_2 - \bar{t}_3 t_3 = 0, \label{equ:b1} \\
			& \bar{\e}_2^\top \bar{\e}_2 c - \bar{t}_1 t_1 - \bar{t}_3 t_3 = 0, \label{equ:b2} \\
			& \bar{\e}_3^\top \bar{\e}_3 c - \bar{t}_1 t_1 - \bar{t}_2 t_2 = 0, \label{equ:b3} \\
			& 2 \bar{\e}_1^\top \bar{\e}_2 c + \bar{t}_2 t_1 + \bar{t}_1 t_2 = 0, \label{equ:b4} \\
			& 2 \bar{\e}_1^\top \bar{\e}_3 c + \bar{t}_3 t_1 + \bar{t}_1 t_3 = 0, \label{equ:b5} \\
			& 2 \bar{\e}_2^\top \bar{\e}_3 c + \bar{t}_3 t_2 + \bar{t}_2 t_3 = 0. \label{equ:b6} \\
			& \bar{t}_1 t_1 + \bar{t}_2 t_2 + \bar{t}_3 t_3 = 0. \label{equ:b7}
		\end{empheq}
	\end{subequations}
	By eliminating $\mathbf{t}$ from Eqs.~\eqref{equ:b1}\eqref{equ:b2}\eqref{equ:b3}\eqref{equ:b7}, we obtain that $(\bar{\e}_1^\top \bar{\e}_1 + \bar{\e}_2^\top \bar{\e}_2 + \bar{\e}_3^\top \bar{\e}_3) c = 0$ and $c = 0$. By substituting $c=0$ into this equation system, it can be further verified that this equation system only has zeros as its solution. So $\mathbf{v}$ can only be a zero vector.
\end{proof}

\begin{theorem}
	\label{theorem:stability}
	For QCQP~\eqref{equ:op_new_obj3} and its Lagrangian dual problem~\eqref{equ:op_dual}, let $\bar{\C}$ being constructed by noise-free observations and assume that $\rank(\bar{\C}) = 8$. (i) There is zero-duality-gap whenever $\C$ is close enough to $\bar{\C}$. In other words, there exists a hypersphere of nonzero radius $\epsilon$ with center $\bar{\C}$, i.e., $\mathcal{B}(\bar{\C}, \epsilon) = \|\C - \bar{\C} \| \le \epsilon$, such that for any $\C \in \mathcal{B}(\bar{\C}, \epsilon)$ there is zero-duality-gap. (ii) Moreover, the semidefinite relaxation problem~\eqref{equ:op_dual} recovers the optimum of the original QCQP problem~\eqref{equ:op_new_obj3}.
\end{theorem}
\begin{proof}
	From Lemma~\ref{lemma:noisefree}, when the point correspondences are noise-free, the relaxation is tight. From the proof procedure, the optimum is $\bar{\x} = [\bar{\e}; \bar{\mathbf{t}}]$, which uniquely determines the zero-duality-gap parameter $\bar{\C}$. 
	Our proof is an application of the Theorem~\ref{theorem:main}. 
	The four conditions needed in Theorem~\ref{theorem:main} are satisfied:
	(\emph{RS}) From Lemma~\ref{lemma:slater}, the restricted Slater is satisfied. 
	(\emph{R1}) The equality constraints in the primal problem forms a variety. Abadie constraint qualification (ACQ) holds everywhere since the variety is smooth and the ideal is radical, see~\cite[Lemma~6.1]{cifuentes2017local}. 
	(\emph{R2}) smoothness. The normalized essential manifold is smooth~\cite{tron2017space}. 
	(\emph{R3}) From Lemma~\ref{lemma:branch_point}, $\bar{\x}$ is not a branch point of $\bar{X}$.
\end{proof}
Now we complete the proof that the SDR of the proposed QCQP is tight under low noise observations. 
In other words, when the observation noise is small, Theorem~\ref{theorem:stability} guarantees that the optimum of the original QCQP can be found by optimizing its SDR. 
Finding the noise bounds that SDR can tolerate is still an open problem. In Section~\ref{sec:exp}, we will demonstrate that the SDR is tight for large noise levels which are much larger than that in actual occurrence.

\section{Robust $N$-Point Method}
\label{sec:robust}

In previous sections, we propose a  solution to non-minimal case essential matrix estimation. 
However, the presence of outliers in point correspondences is inevitable due to the ambiguities of feature points' local appearance. 
When the correspondences are contaminated by outliers, the solution to minimize an algebraic error will be biased. 
To apply our method for outlier-contamination scenarios, a robust loss instead of least-square loss can be used in the objective function. 
We take advantage of M-estimators in robust statistics~\cite{huber1981robust}, and modify problem~\eqref{equ:op_new_obj2} as
\begin{align}
\label{equ:op_new_obj4}
\min_{\E \in \mathcal{M}_\E} & \sum_{i=1}^N \rho (\f_i^\top \E \f'_i),
\end{align}
where $\rho(\cdot)$ is a robust loss function. 

The selection of an appropriate robust loss function is critical to this problem. 
A natural choice for loss function would be the $\ell_0$ norm, i.e., $\rho(x; \tau) = \begin{cases}
0, & \text{if } \ |x| \le \tau, \\
1, & \text{otherwise}.
\end{cases}$. 
It turns out to be the inlier set maximization problem, which has been proved to be a computationally expensive problem~\cite{chin2020robust}. 
The presented approach below can accommodate many M-estimators in a computationally efficient framework. First we will take the scaled Welsch (Leclerc) function~\cite{geman1987statistical,zach2017iterated} as an example
\begin{align}
\rho(x; \tau) = \frac{\tau^2}{2}\left(1 - e^{-x^2/\tau^2} \right), 
\end{align}
where $\tau \in (0, +\infty)$ is a scale parameter.

Problem~\eqref{equ:op_new_obj4} can be optimized by local optimization methods directly. However, local optimization is prone to local optima. To solve this problem, our approach is based on Black-Rangarajan duality between M-estimators and line processes~\cite{black1996unification}. The duality introduces an auxiliary variable 
$w_i$ for $i$-th point correspondence and optimizes a joint objective
over the essential matrix $\E$ and the line process variables $\mathbb{W} = \{w_i\}_{i=1}^N$
\begin{align}
\label{equ:op_new_obj5}
\min_{\E \in \mathcal{M}_\E, \mathbb{W}} \ \ \sum_{i=1}^N w_i (\f_i^\top \E \f'_i)^2 + \sum_{i=1}^N \Psi(w_i).
\end{align}
Here $\Psi(w_i)$ is a loss on ignoring $i$-th correspondence.
$\Psi(w_i)$ tends to zero when the $i$-th correspondence is active and to one
when it is inactive. A broad variety of robust estimators $\rho(\cdot)$ have corresponding loss functions $\Psi(\cdot)$ such that problems~\eqref{equ:op_new_obj4} and~\eqref{equ:op_new_obj5} are equivalent with respect to $\E$: optimizing either of the two problem yields the same essential matrix. 
The form of problem~\eqref{equ:op_new_obj5} enables efficient and scalable optimization by an iterative solution of the weighted $N$-point method. This yields a general approach that can accommodate many robust nonconvex functions $\rho(\cdot)$. 

The loss function that makes problems~\eqref{equ:op_new_obj4} and~\eqref{equ:op_new_obj5} equivalent with respect
to $\E$ is
\begin{align}
\Psi(w_i) = \frac{\tau^2}{2}\left( 1 + w_i \log(w_i) - w_i \right).
\end{align}
Objective~\eqref{equ:op_new_obj5} is biconvex on $\E, \mathbb{W}$. When $\E$ is fixed, the optimal value of each $w_i$ has a closed-form solution. When variables $\mathbb{W}$ are fixed, objective~\eqref{equ:op_new_obj5} turns into a weighted $N$-point problem, which can be efficiently solved by the proposed method. 
Specifically, we exploit this special structure and optimize the objective by alternatively updating variable sets $\E$ and $\mathbb{W}$. As a block coordinate descent algorithm, this alternating minimization
scheme provably converges.
By fixing $\E$, the optimal value of each $w_i$ is given by
\begin{align}
\label{equ:weight}
w_i = e^{-(\f_i^\top \E \f'_i)^2/\tau^2}.
\end{align}
This can be verified by substituting Eq.~\eqref{equ:weight} into objective~\eqref{equ:op_new_obj5}, which yields a constant $1$ with respect to $\E$. 
Thus optimizing objective~\eqref{equ:op_new_obj5} yields a solution $\E$ that is also optimal for the original objective~\eqref{equ:op_new_obj4}. 
The line process theory~\cite{black1996unification} provides general formulations to calculate $w_i$ for a broad variety of robust loss functions $\rho(x; \tau)$. The alternating minimization of a line process is illustrated in Fig.~\ref{fig:m-estimator}.

\begin{figure}[htbp]
	\begin{center}
	{
		\includegraphics[width=0.99\linewidth]{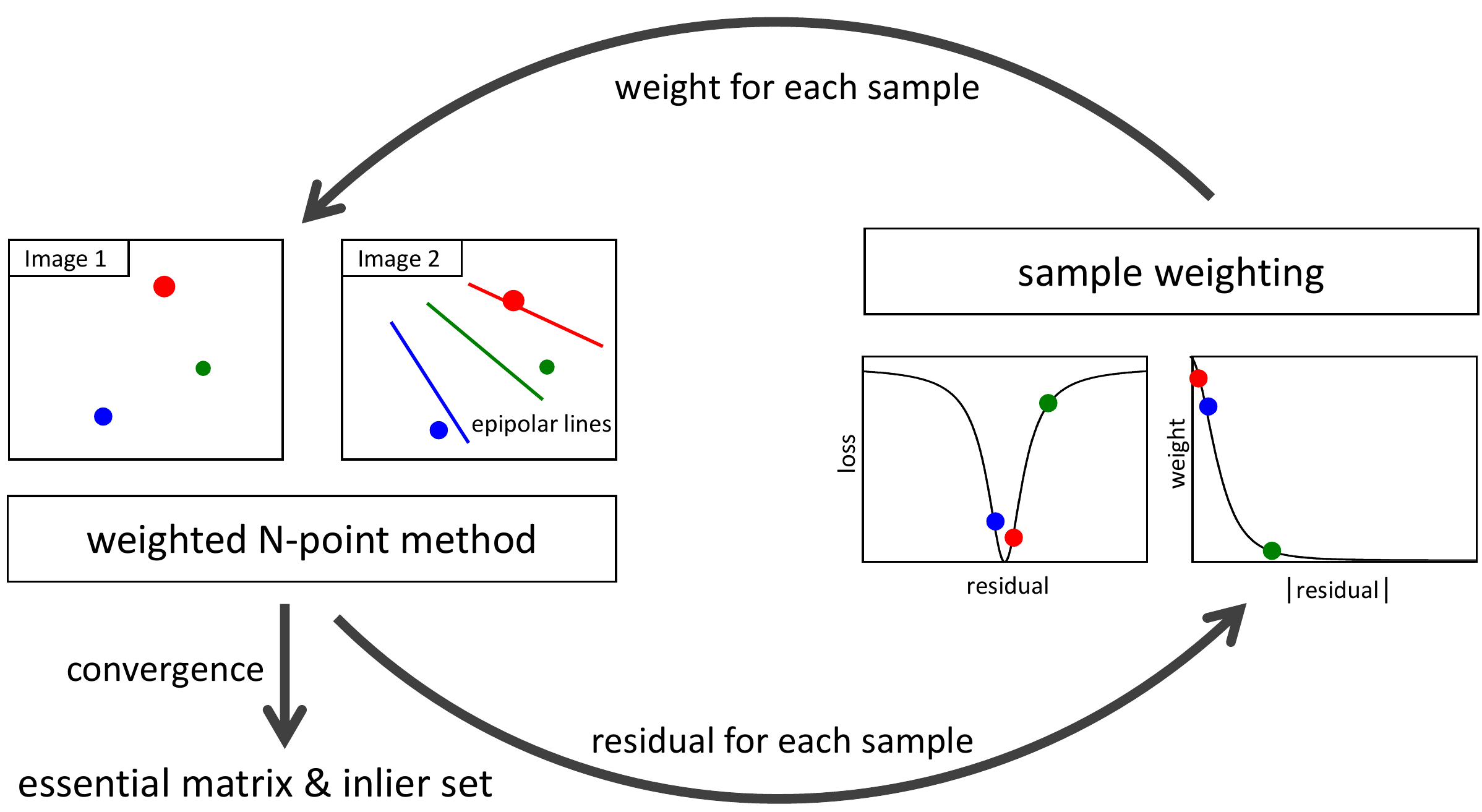}
	}
	\end{center}
	\vspace{-0.15in}
	\caption{Robust estimation of essential matrix by solving a line process of an M-estimator. In left part, the size of a feature point represents its weight.}
	\label{fig:m-estimator}
\end{figure}

It is worth mentioning that the loss function in our robust $N$-point method is not limited to Welsch function. The Black-Rangarajan duality~\cite{black1996unification} provides a theory to build the relation between M-estimators and line processes. Based on this general-purpose framework, integrating other robust loss functions is essentially the same as the Welsch function. Typical robust loss functions include Cauchy (Lorentzian), Charbonnier (pseudo-Huber, $\ell_1$-$\ell_2$), Huber, Geman-McClure, smooth truncated quadratic, truncated quadratic, Tukey's biweight functions, etc, see Fig.~\ref{fig:kernel1}. The Black-Rangarajan duality of these loss functions can be found in~\cite{black1996unification,zach2017iterated}.

\begin{figure}[htbp]
	\begin{center}
	\subfigure[Different loss functions]
	{
		\includegraphics[width=0.475\linewidth]{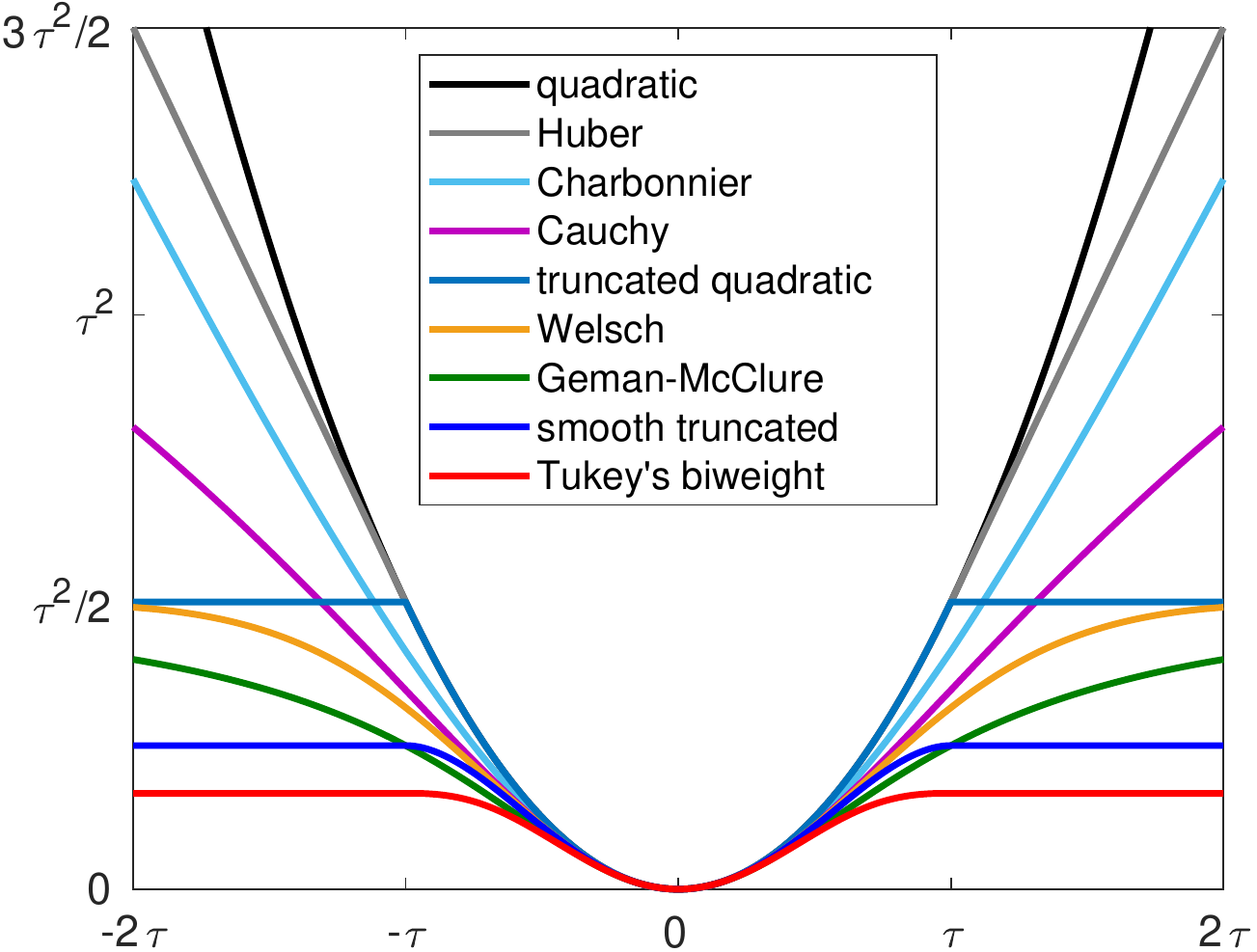}
		\label{fig:kernel1}
	}
	\subfigure[Welsch functions]
	{
		\includegraphics[width=0.45\linewidth]{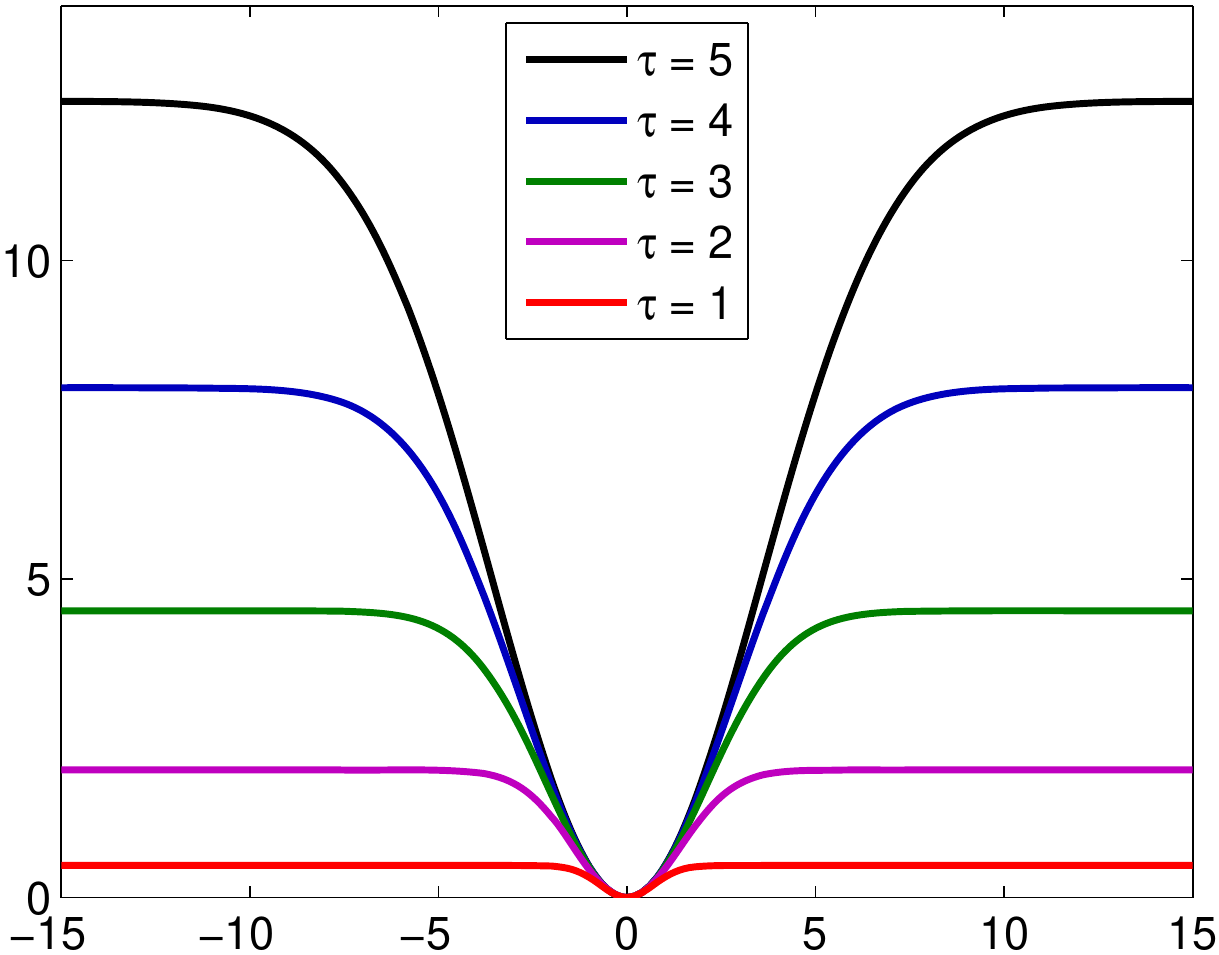}
		\label{fig:kernel2}
	}
	\end{center}
	\vspace{-0.15in}
	\caption{Loss functions. (a) Different loss function. (b) Welsch functions~\cite{black1996unification} with different scale parameters $\tau$.}
\end{figure}

Objective~\eqref{equ:op_new_obj5} is non-convex and its shape is controlled by the parameter $\tau$ of loss function $\rho(\cdot)$. To alleviate the effect of local minima, we employ graduated non-convexity (GNC) which is dating from 1980s~\cite{black1987visual,yang2020graduated}. 
The idea is to start from a problem that is easier to solve, then progressively deformed to the actual objective while tracking the solution along the way.
Large $\tau$ makes the objective function smoother and allows many correspondences to participate in the optimization even when they are not fit well by the essential matrix $\E$, see Fig.~\ref{fig:kernel2}. 
Our method begins with a very large value $\tau$. 
Over the iterations, $\tau$ is automatically decreased, gradually introducing non-convexity into the objective. 
The line process together with the GNC strategy is not guaranteed to obtain the global optimum for the original non-convex problem~\eqref{equ:op_new_obj4}. They can obtain a good solution if not the globally optimal one in most cases when the outlier ratio is below a threshold. 
The outline of the robust $N$-point method is shown in Algorithm~\ref{alg:alg2}.

\begin{algorithm}[htbp]
	\caption{Robust $N$-Point Method} \label{alg:alg2}
	\KwIn{correspondences $\{(\f_i, \f'_i) \}^{N}_{i=1}$, parameter $\tau_{\text{min}}$}
	\KwOut{Essential matrix $\E^\star$, rotation $\R^\star$, translation $\mathbf{t}^\star$, and inlier set $\mathcal{I}$.}
	Initialize $w_i \leftarrow 1, \forall i = 1, \cdots, N$\;
	Initialize $\tau^2 \leftarrow 1\times 10^3$\;
	\Repeat{convergence or $\tau < \tau_{\text{min}}$}{
		Update $\E$ by weighted $N$-point method in Algorithm~\ref{alg:alg1}\;
		Update $w_i$ by Eq.~\eqref{equ:weight}\;
		set $\tau^2 \leftarrow \tau^2/1.3$\;
	}
	Generate inlier set $\mathcal{I} = \{i | w_i > 0.1\}$\;
	Calculate $\E^\star$ by using inlier set $\mathcal{I}$ and unweighted $N$-point method in Algorithm~\ref{alg:alg1}\;
	Decompose $\E^\star$ to obtain $\R^\star$ and $\mathbf{t}^\star$.
\end{algorithm}

\section{Experimental Results}
\label{sec:exp}

{\bf Setting for the $N$-point method:}
We compared the proposed $N$-point method with several state-of-the-art methods on synthetic and real data. 
Specifically, we compared our method with $5$ classical or state-of-the-art methods: 
\begin{itemize}
	\item a general five-point method \texttt{5pt}~\cite{nister2004efficient} for relative pose estimation. 
	
	\item two general methods for fundamental matrix estimation, including seven-point method \texttt{7pt}~\cite{hartley2003multiple} and eight-point method \texttt{8pt}~\cite{hartley1995defence}.
	
	\item eigenvalue-based method proposed by Kneip and Lynen~\cite{kneip2013direct} which is referred to as \texttt{eigen} and a certifiably globally optimal solution by Briales \textit{et al.}~\cite{briales2018certifiably} which is referred to as \texttt{SDP-Briales}. 
\end{itemize}

Among these methods, the implementation of \texttt{SDP-Briales} is provided by the authors. The implementations of other comparison methods are provided by \texttt{OpenGV}~\cite{kneip2014opengv}. 
The method \texttt{eigen} needs initialization, and it is initialized using \texttt{8pt} method. 
Our method and all comparison methods in \texttt{openGV} are implemented in C++. The \texttt{SDP-Briales} method is implemented in Matlab, and the optimization in it relies on an SDP solver with hybrid Matlab/C++ programming.

The \texttt{eigen} method in \texttt{openGV} implementation provides rotation only. 
Once the relative rotation has been obtained, we can calculate the translation $\t$. Recall that $\mathbf{f}_i$ and $\mathbf{f}'_i$ represent bearing vectors of a point correspondence across two images. From Eqs.~\eqref{equ:essential} and~\eqref{equ:epipolar}, the epipolar constraint can be written as 
\begin{align}
\mathbf{f}^\top_i [ \t ]_{\times} \mathbf{R} \mathbf{f}'_i = 0.
\end{align}
Since $\mathbf{R}$ has been calculated, each point correspondence provides a linear constraint on the entries of the translation vector $\t$. Due to the scale-ambiguity, the translation has only two DoFs. After DLT, a normalized version of $\t$ can be recovered by simple linear derivation of the right-hand null-space vector (e.g. via singular value decomposition). Given $N$ ($N
\ge 2$) point correspondences, the least squares fit of $\t$ can be determined by considering the singular vector corresponding to the smallest singular value.

{\bf Setting for robust $N$-point method:}
We compared the proposed robust $N$-point method with \texttt{RANSAC+5pt}, \texttt{RANSAC+7pt} and \texttt{RANSAC+8pt}, which stands for integrating \texttt{5pt}, \texttt{7pt} and \texttt{8pt} into the RANSAC framework~\cite{Fischler81}, respectively.
All comparison methods are provided by \texttt{openGV}~\cite{kneip2014opengv}, and the default parameters are used.
In the experiment on real-world data, we also compare our method with a branch-and-bound method \texttt{BnB-Yang}\footnote{The C++ code is available from \url{http://jlyang.org/}}~\cite{yang2014optimal}. The angular error threshold in this method was set as $0.002$~radians. 

To evaluate the performance of the proposed method, we separately compared the relative rotation and translation accuracy. We follow the criteria defined in~\cite{zheng2013revisiting} for quantitative evaluation. 
Specifically, 
\begin{itemize}
	\item the angle difference for rotations is defined as
	\begin{align}
	\varepsilon_\text{rot} [\text{degree}] = \arccos \left(\frac{\trace(\R_{\text{true}}^\top \R^\star) - 1}{2} \right) \cdot \frac{180}{\pi}, \nonumber
	\end{align}
	\item and the translation direction error is defined as 
	\begin{align}
	\varepsilon_\text{tran} [\text{degree}] = \arccos\left(\frac{\t_{\text{true}}^\top \t^\star} { \|\t_{\text{true}}\| \cdot \|\t^\star\|}\right) \cdot \frac{180}{\pi}. \nonumber
	\end{align}
\end{itemize}
In above criteria, $\R_{\text{true}}$ and $\R^\star$ are the ground truth and estimated rotation, respectively; $\t_{\text{true}}$ and $\t^\star$ are the ground truth and estimated translation, respectively.

\subsection{Efficiency of $N$-Point Method}
The \texttt{SDPA} solver~\cite{yamashita2012latest} was adopted as an SDP solver in our methods. 
All experiments were performed on an Intel Core i7 CPU running at $2.40$ GHz. 
The number of point correspondences is fixed to $100$. 
The SDP optimization takes about $5$~ms. In addition, it takes about $1$~ms for other procedures in our method, including problem construction, optimal essential matrix recovery, and pose decomposition. In summary, the runtime of our method is about $6$~ms. 
 
We compared the proposed method with several state-of-the-art methods~\cite{hartley2009global,chesi2009camera,briales2018certifiably}, which also aim to find the globally optimal relative pose. 
The efficiency comparison is shown in Table~\ref{tab:eff_comp2}. 
It can be seen that our method is $2\sim 3$ orders of magnitude faster than comparison methods. 
The superior efficiency makes our method the first globally optimal method that can be applied to large scale structure-from-motion and realtime SLAM applications.
\setlength{\tabcolsep}{5pt}
\begin{table}[tbp]
	\caption{Efficiency comparison with other globally optimal methods. The last column is the normalized runtime by setting \texttt{ours} as $1$.}
	\label{tab:eff_comp2}
	\vspace{-0.1in}
	\centering
	\begin{tabular}{lccc}
		\toprule
		method  & optimization & runtime & norm. runtime \\
		\midrule
		Hartley \& Kahl~\cite{hartley2009global} & BnB & $>7$ s  & $>$ $1000$ \\
		Chesi~\cite{chesi2009camera} & LMI & $1.15$~s & $190$ \\
    	\texttt{SDP-Briales}~\cite{briales2018certifiably} & SDP & $1$~s & $ 160$ \\
		\texttt{ours} & SDP & $6$~ms & $1$ \\
		\bottomrule
	\end{tabular}
\end{table}
\setlength{\tabcolsep}{1.4pt}

Since both the methods of \texttt{SDP-Briales} and \texttt{ours} take advantage of SDP optimization, further comparison between them is reported in Table~\ref{tab:eff_comp}. It can be seen that our method has a much simpler formulation in terms of numbers of variables and constraints. It is not surprising that our method has significantly better efficiency.

\setlength{\tabcolsep}{5pt}
\begin{table}[htbp]
	\caption{SDP formulations comparison. For domain $\mathcal{S}^n$, the number of variable is $n(n+1)/2$.}
	\label{tab:eff_comp}
	\vspace{-0.1in}
	\centering
	\begin{tabular}{lccc}
		\toprule
		method  & domain & \#variable & \#constraint  \\
		\midrule
		\texttt{SDP-Briales}~\cite{briales2018certifiably} & $\mathcal{S}^{40}$  & $820$ & $536$\\
		\texttt{ours} & $\mathcal{S}^{12}$  & $78$ &  $7$ \\
		\bottomrule
	\end{tabular}
\end{table}
\setlength{\tabcolsep}{1.4pt}

\subsection{Accuracy of $N$-Point Method}
\subsubsection{Synthetic Data}
\label{sec:syn_data1}

To thoroughly evaluate our method's performance, we perform experiments on synthetic scenes in a similar manner to~\cite{kneip2013direct}. 
We generate random scenes by first fixing the position of the first frame to the origin and its orientation to the identity. The translational offset of the second frame is chosen with uniformly distributed random direction and a maximum magnitude of $2$. The orientation of the second frame is generated with random Euler angles bounded to $0.5$~radians in absolute value. This generates random relative poses as they would appear in practical situations. Point correspondences result from uniformly distributed random points around the origin with a distance varying between $4$ and $8$, transforming those points into both frames.
Then a virtual camera is defined with a focal length of $800$~pixels. 
Gaussian noise is added by perturbing each point correspondence. The standard deviation of the Gaussian noise is referred to as the noise level.

First, we test image noise resilience. 
For each image noise level, we randomly generate synthetic scenes and repeat the experiments $1000$ times. The number of correspondences is fixed to $10$, and the step size of the noise level is $0.1$ pixels.
The results for all methods with varying image noise levels are shown in Fig.~\ref{fig:exp_sythetic_noise}.  
It can be seen that our method consistently has smaller rotation error $\varepsilon_{\text{rot}}$ and translation error $\varepsilon_{\text{tran}}$ than other methods. Moreover, our method and \texttt{eigen} significantly outperform \texttt{5pt}, \texttt{7pt} and \texttt{8pt}. This result demonstrates the advantage of non-minimal solvers in terms of accuracy.

\begin{figure}[tbp]
	\begin{center}
	\subfigure[rotation error $\varepsilon_\text{rot}$]
	{
		\includegraphics[width=0.47\linewidth]{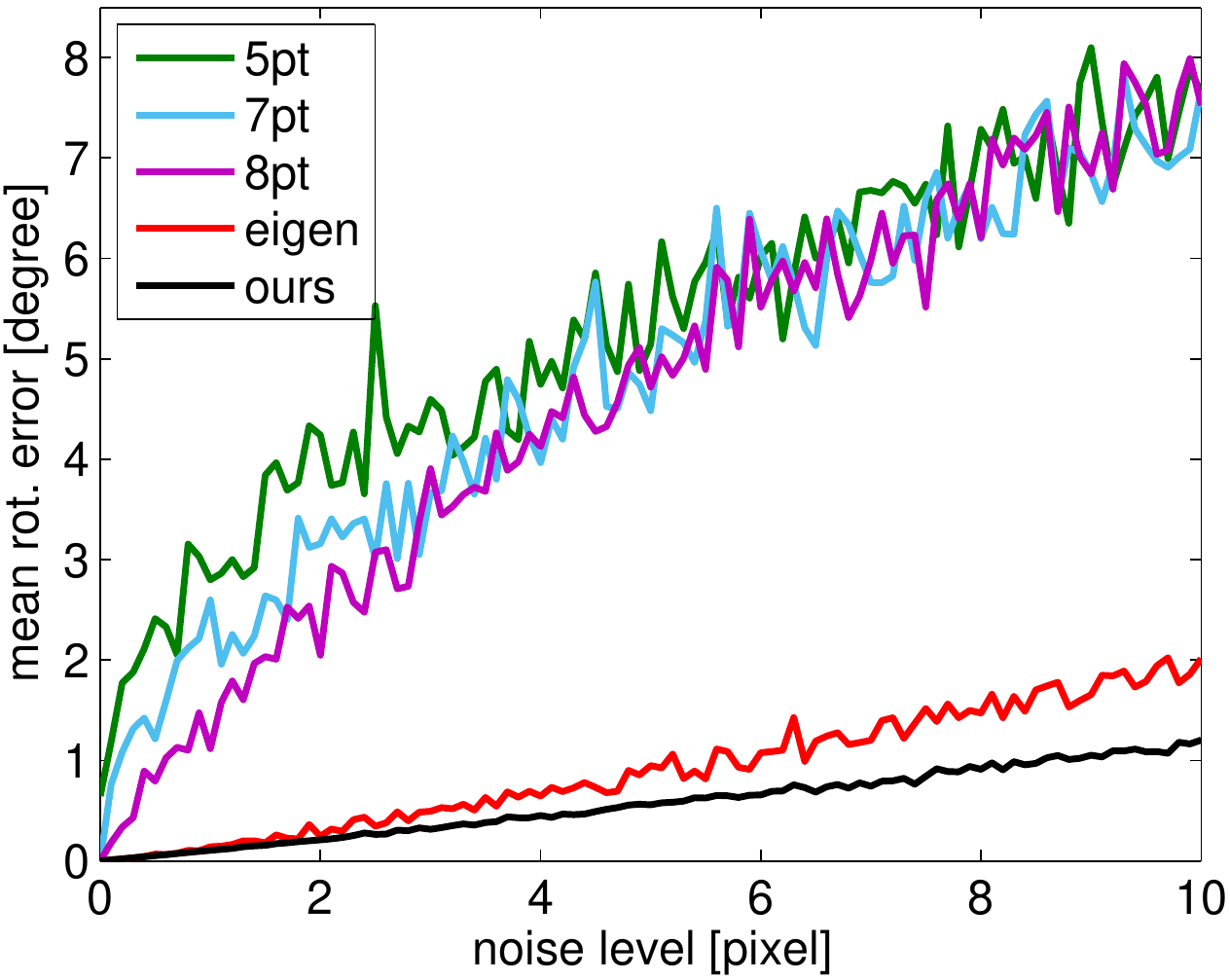}
	}
	\subfigure[translation error $\varepsilon_\text{tran}$]
	{
		\includegraphics[width=0.475\linewidth]{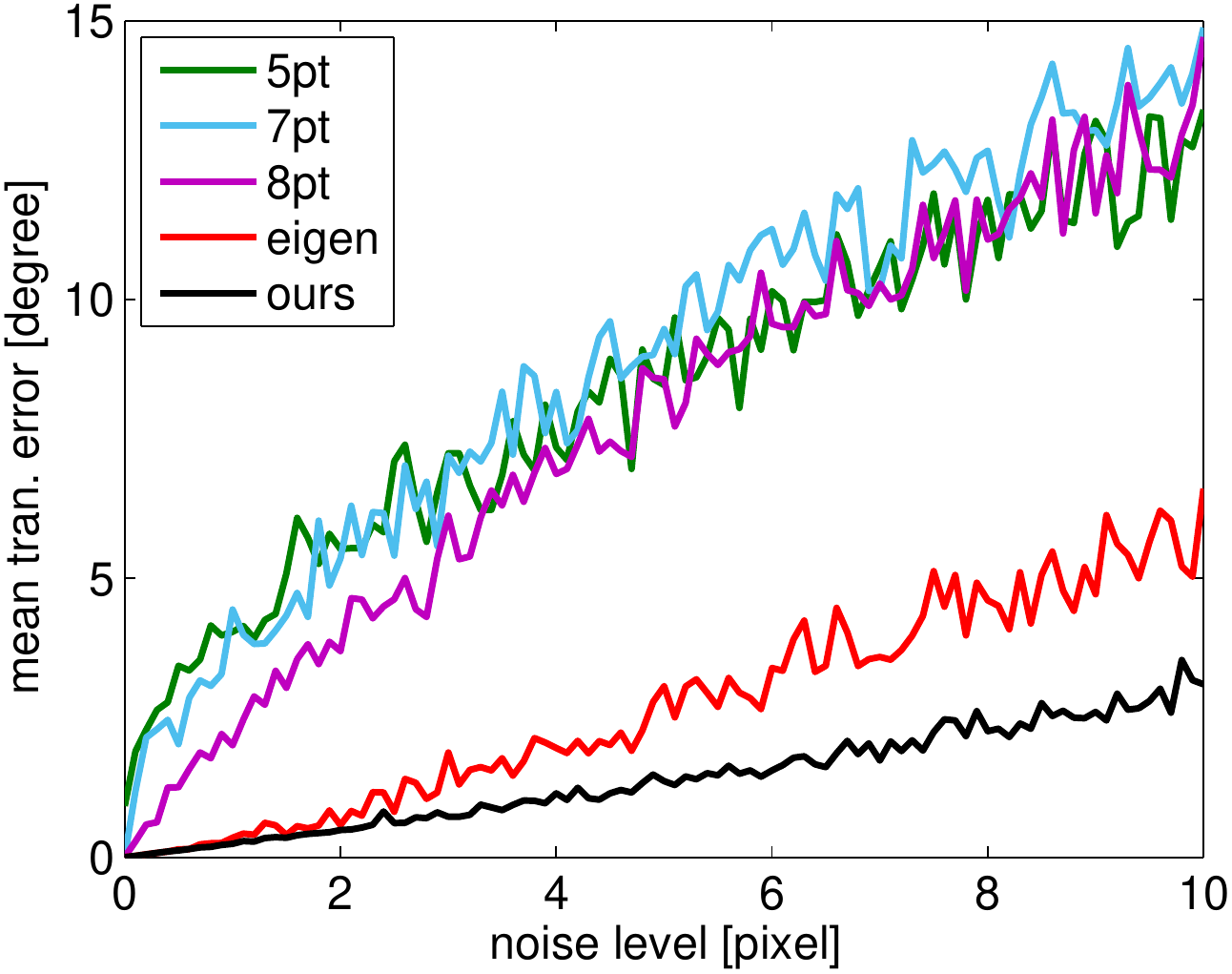}
	}
	\end{center}
	\vspace{-0.15in}
	\caption{Relative pose accuracy with respect to image noise levels. 
	}
	\label{fig:exp_sythetic_noise}
\end{figure}

Second, we set the noise level as $5$ pixels and vary the number $N$ of point correspondences. The step size of the correspondence number is $5$.
The methods \texttt{5pt}, \texttt{7pt}, and \texttt{8pt} can only take a small subset of the point correspondences. 
In contrast, \texttt{eigen} and \texttt{ours} utilize all the point correspondences. 
To make a fair comparison, we randomly sample minimal number of point correspondences for \texttt{5pt}, \texttt{7pt}, and \texttt{8pt}, and repeat $20$ times for each method. Then we find the optimal relative pose among them. Since all point correspondences are inliers,  we cannot use the maximal inlier criterion to find the optimal rotation as that in the traditional RANSAC framework. Instead, we use the algebraic error to find the optimal relative rotation for these methods.

The pose estimation accuracy with respect to the number of point correspondences is shown in Fig.~\ref{fig:exp_sythetic_num}. 
We have the following observations: 
(1) The errors of \texttt{eigen} and \texttt{ours} decrease when increasing the numbers of point correspondences. It further verifies the effectiveness of non-minimal solvers. 
(2)  When $N > 20$, \texttt{eigen} and \texttt{ours} have significantly smaller errors than other methods, and our method has the smallest rotation and translation error among all methods.
(3) The error curve of \texttt{Eigen} has oscillation due to local minima when the noise level is large. In contrast, the error curve of our method is smooth for any noise level.

\begin{figure*}[htbp]
	\begin{center}
		\subfigure[$\varepsilon_\text{rot}$, noise level: $2.5$ pix]
		{
			\includegraphics[height=0.18\linewidth]{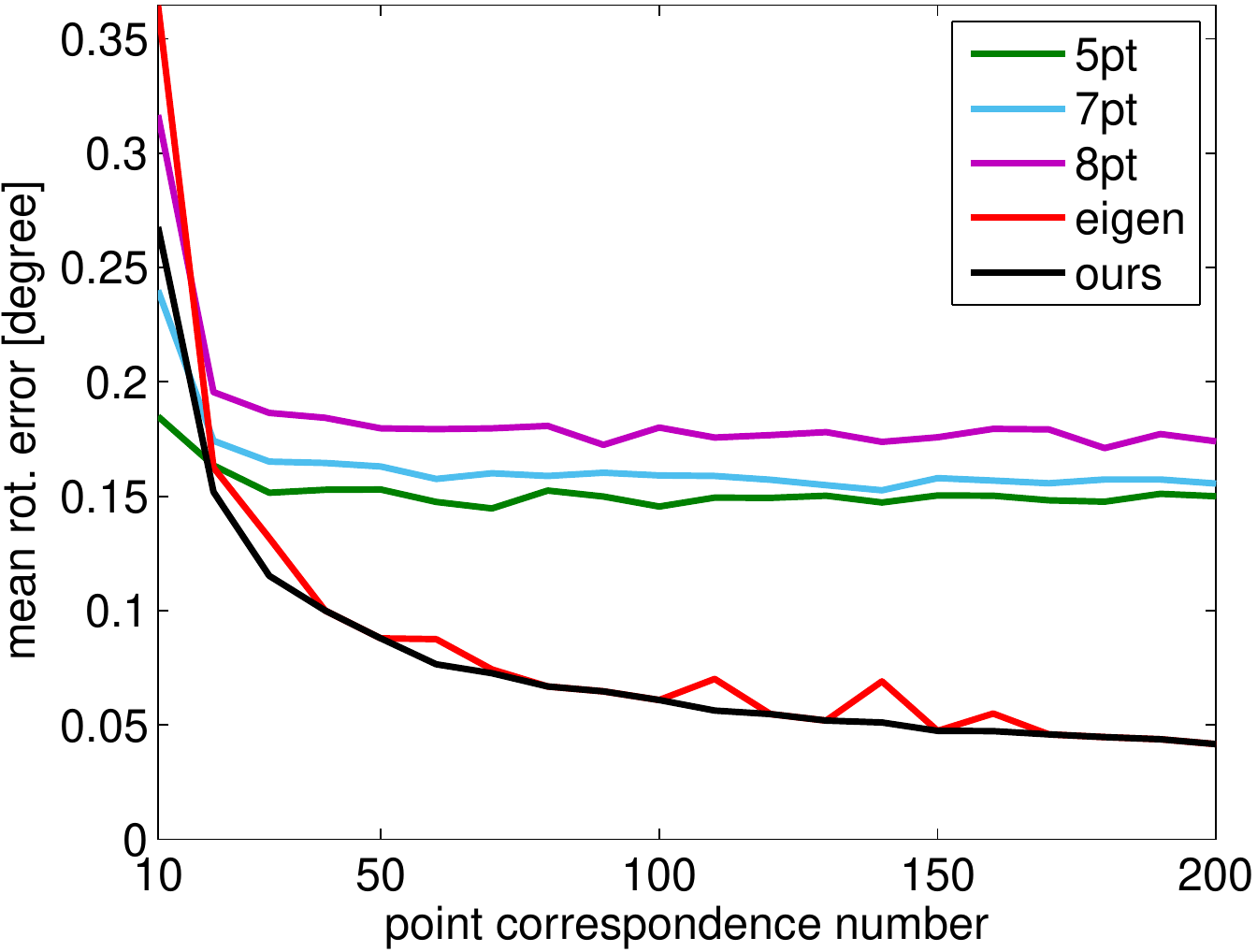}
		}
		\subfigure[$\varepsilon_\text{rot}$, noise level: $5$ pix]
		{
			\includegraphics[height=0.18\linewidth]{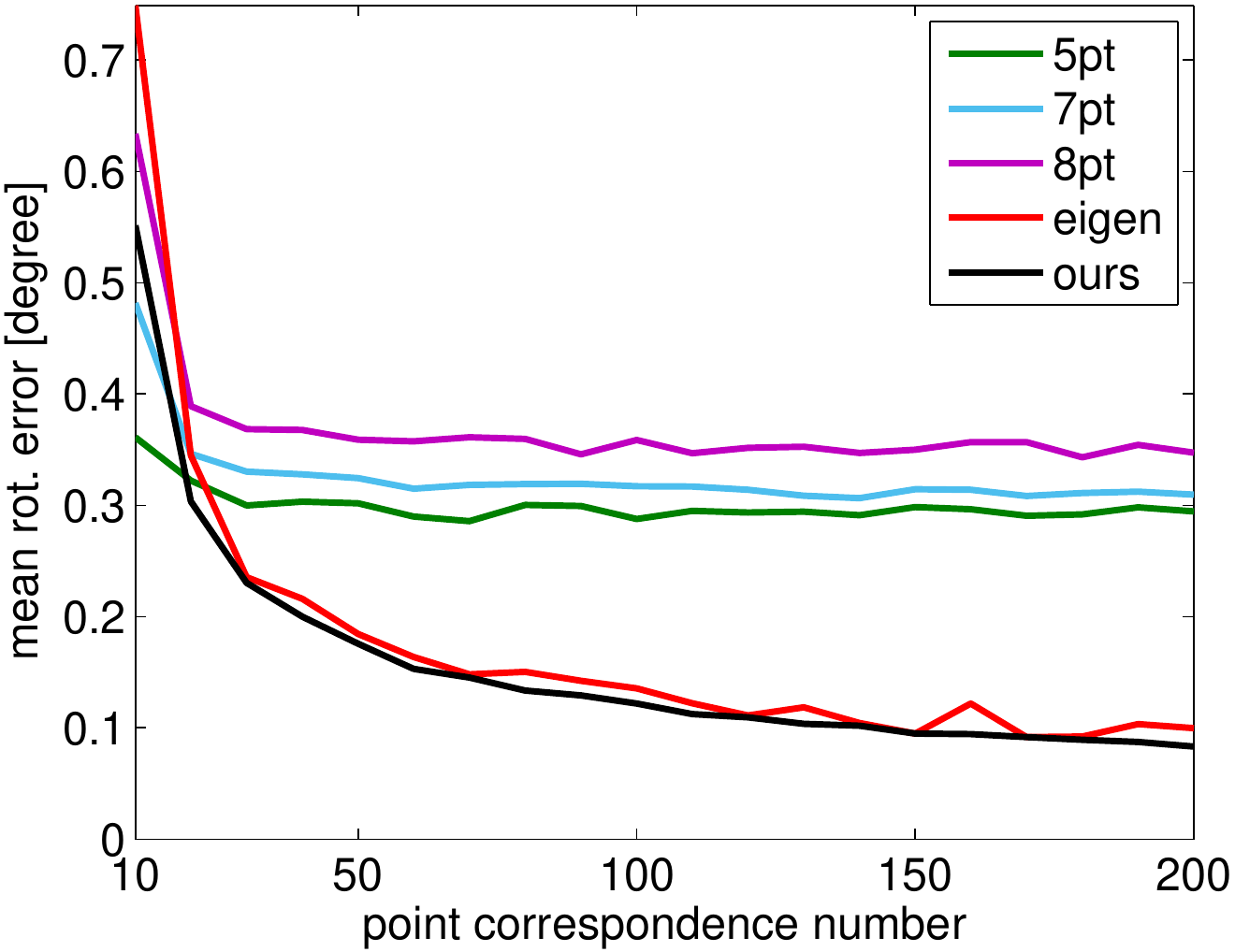}
		}
		\subfigure[$\varepsilon_\text{rot}$, noise level: $10$ pix]
		{
			\includegraphics[height=0.18\linewidth]{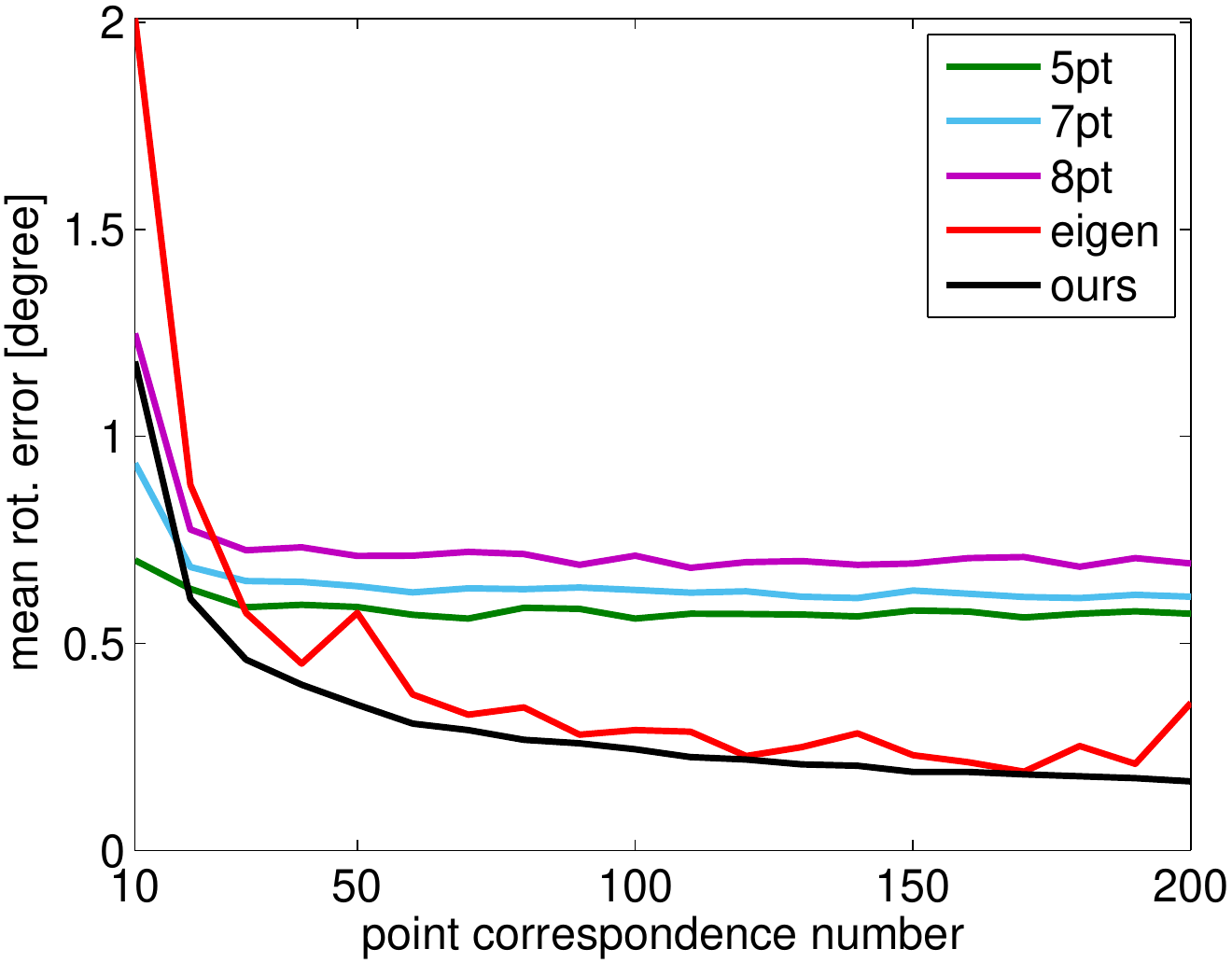}
		}
		\subfigure[$\varepsilon_\text{rot}$, noise level: $20$ pix]
		{
			\includegraphics[height=0.18\linewidth]{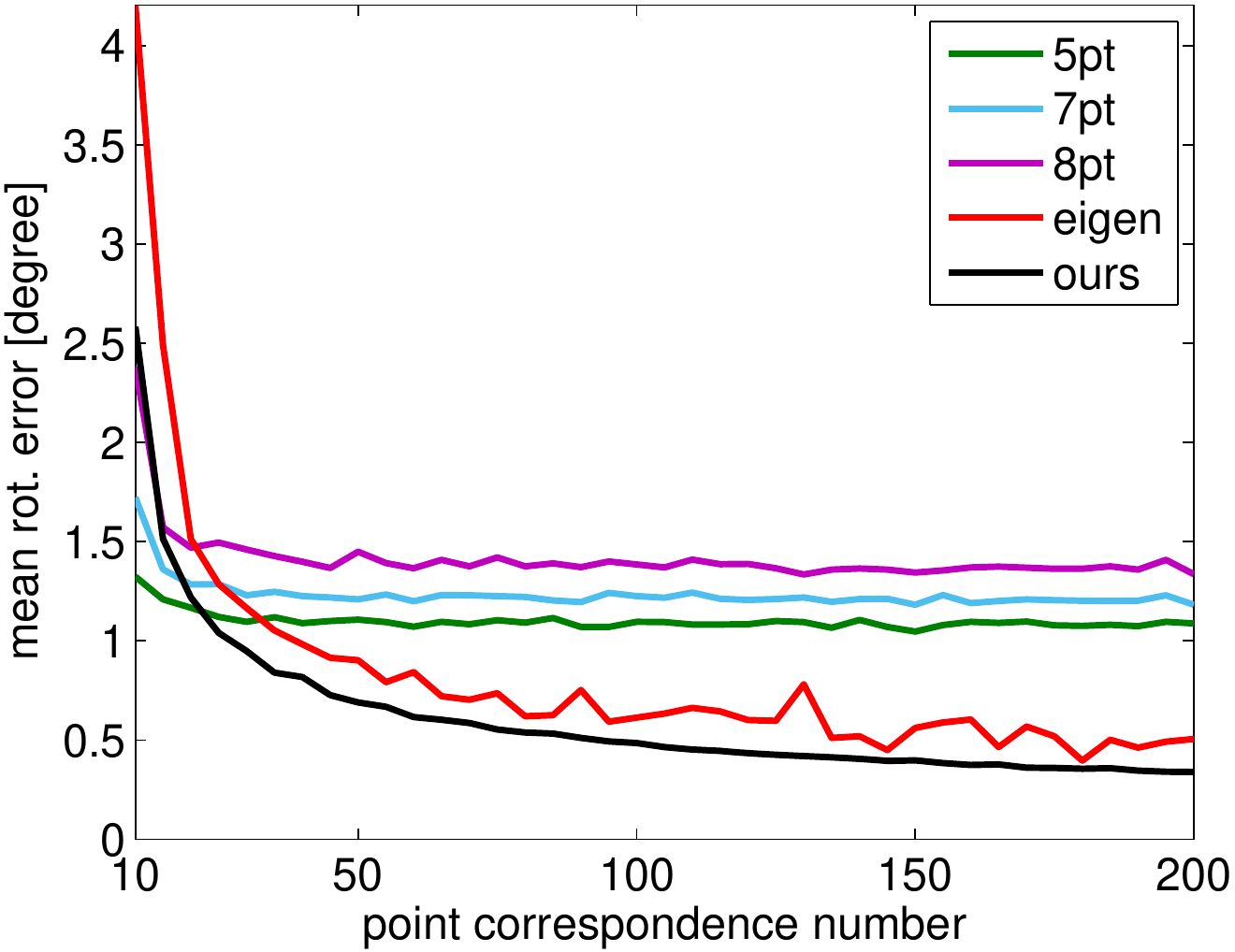}
		}
		\subfigure[$\varepsilon_\text{tran}$, noise level: $2.5$ pix]
		{
			\includegraphics[height=0.18\linewidth]{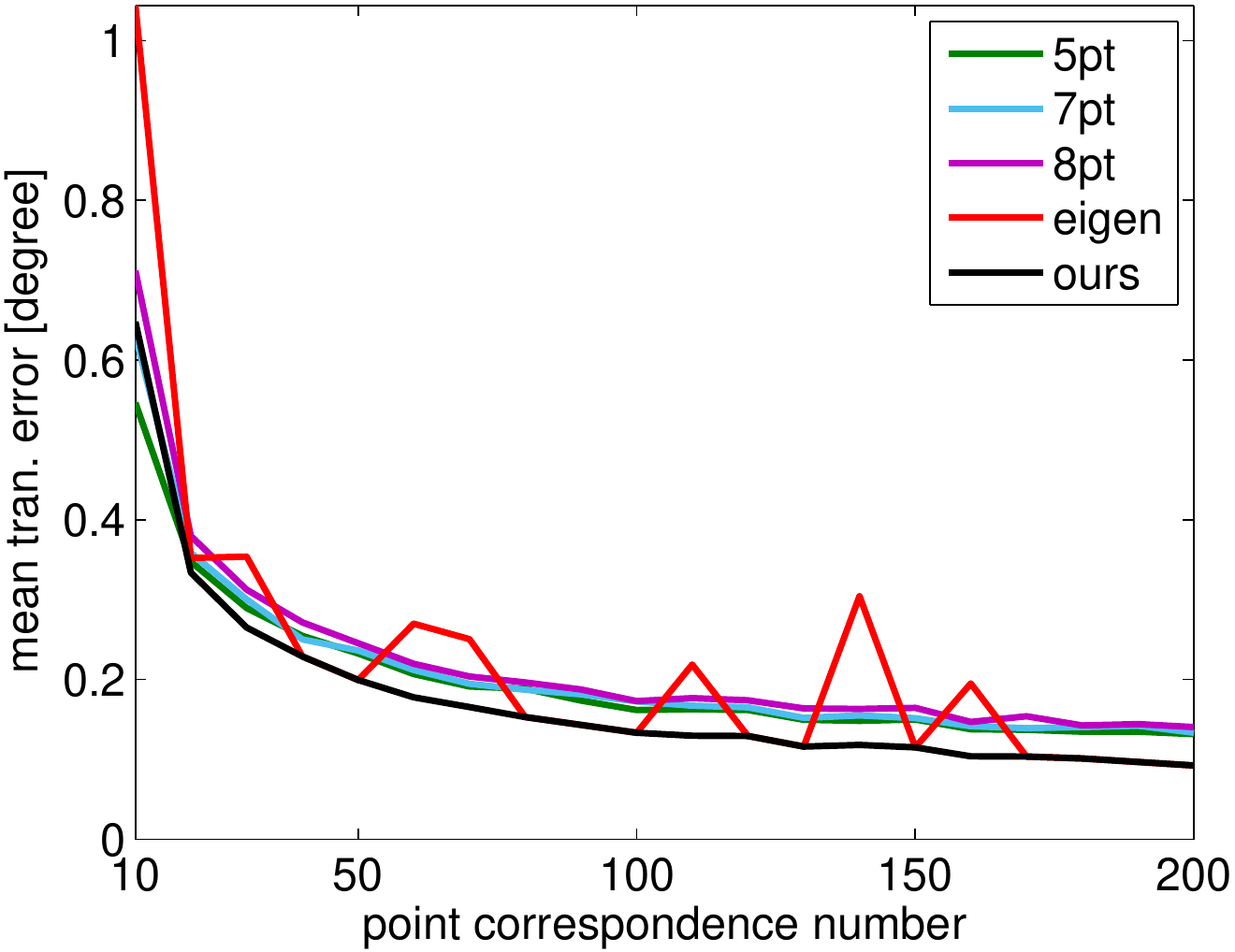}
		}
		\subfigure[$\varepsilon_\text{tran}$, noise level: $5$ pix]
		{
			\includegraphics[height=0.18\linewidth]{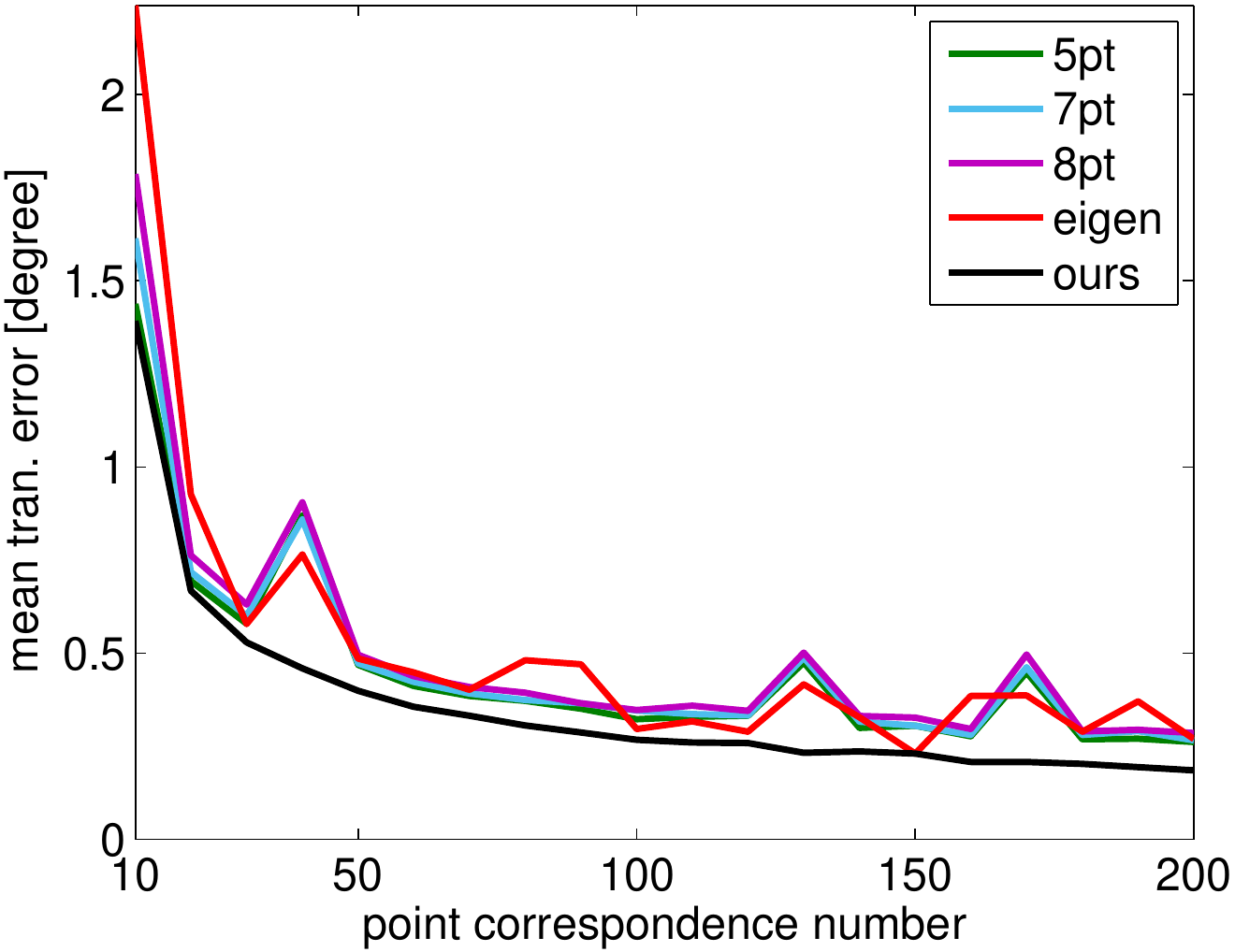}
		}
		\subfigure[$\varepsilon_\text{tran}$, noise level: $10$ pix]
		{
			\includegraphics[height=0.18\linewidth]{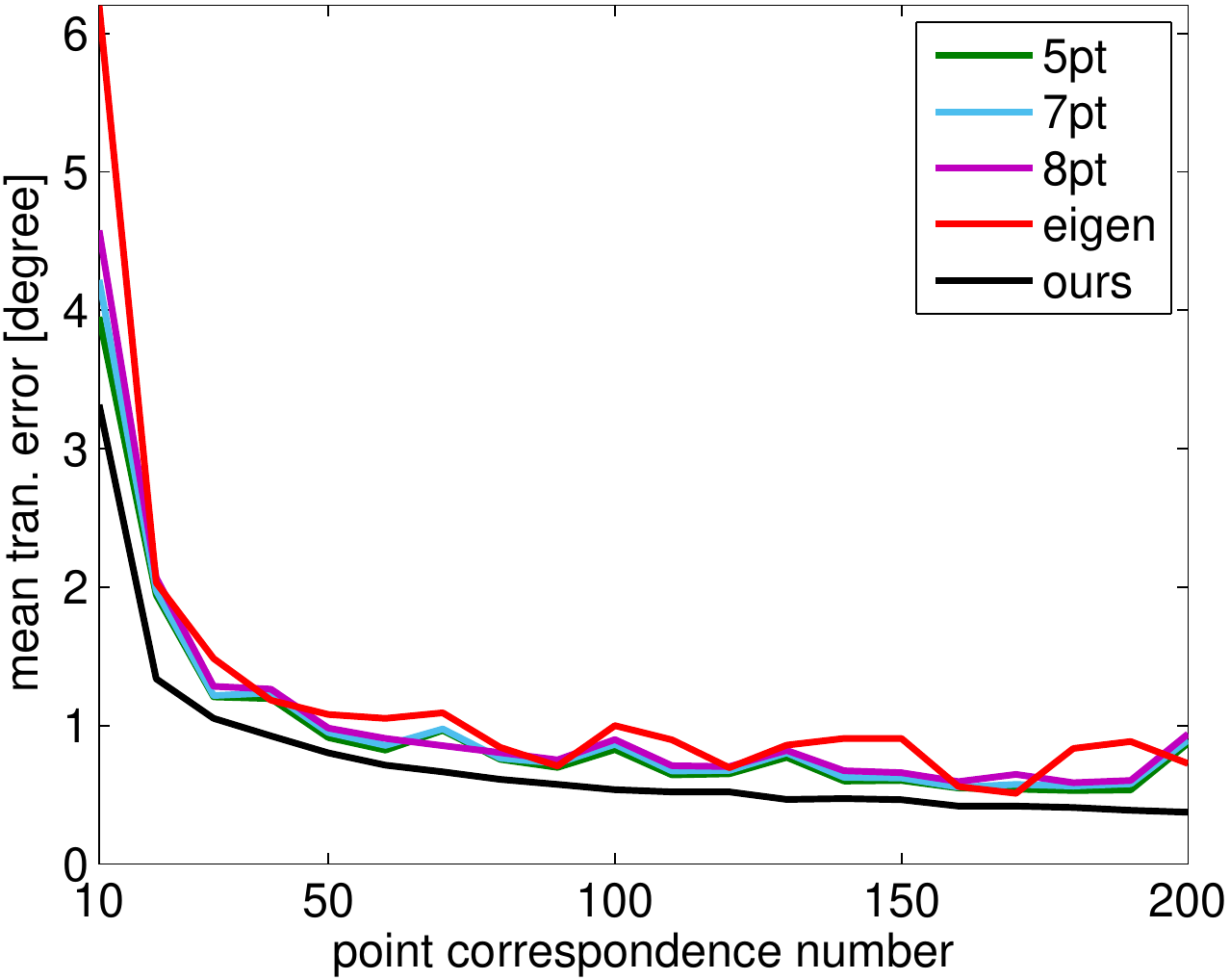}
		}
		\subfigure[$\varepsilon_\text{tran}$, noise level: $20$ pix]
		{
			\includegraphics[height=0.18\linewidth]{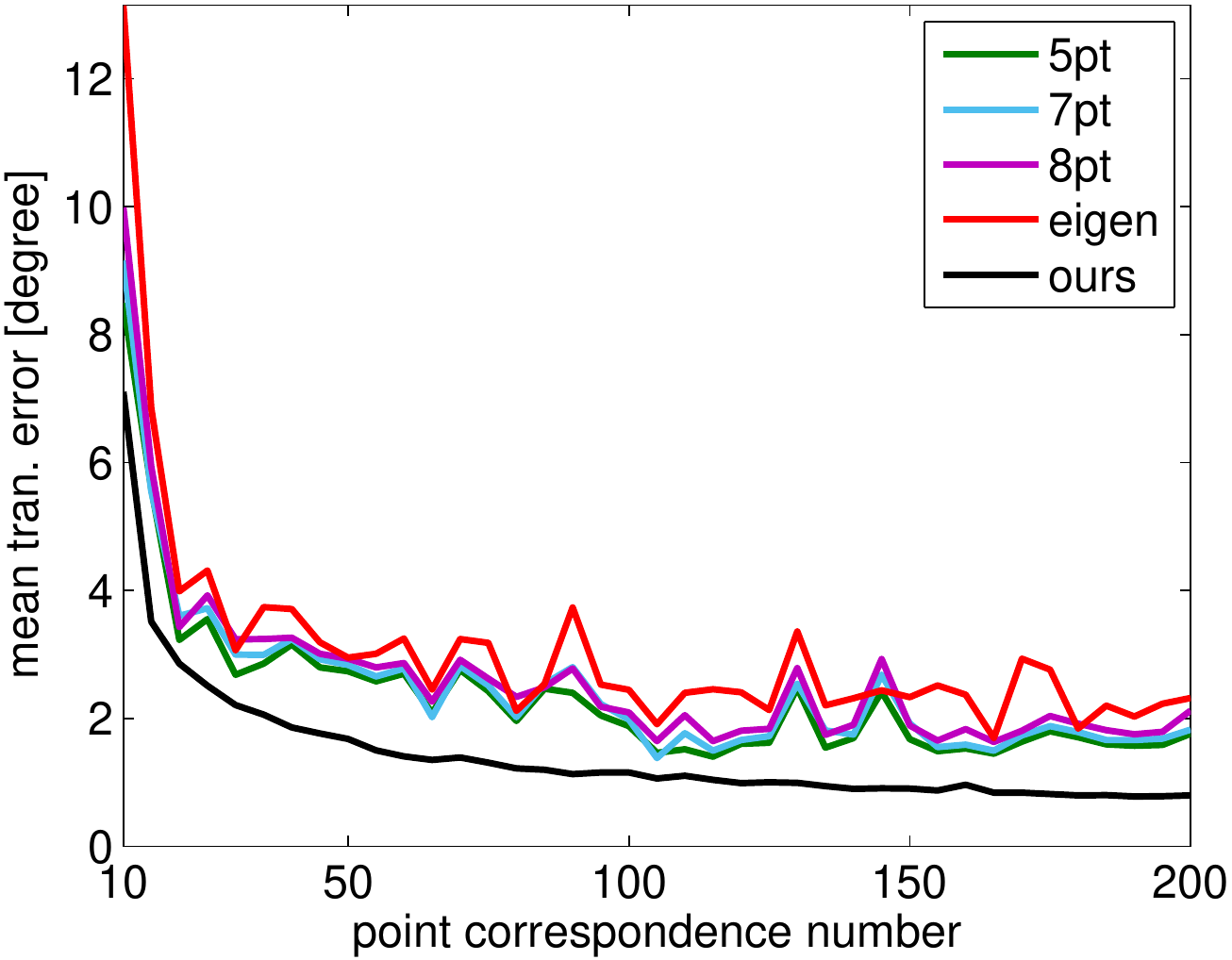}
		}
	\end{center}
	\vspace{-0.15in}
	\caption{Relative pose accuracy with respect to number of point correspondences. 
	}
	\label{fig:exp_sythetic_num}
\end{figure*}

\subsubsection{Real-World Data}
\label{sec:real_data1}

\begin{figure}[htbp]
	\begin{center}
		{
			\includegraphics[width=0.47\linewidth]{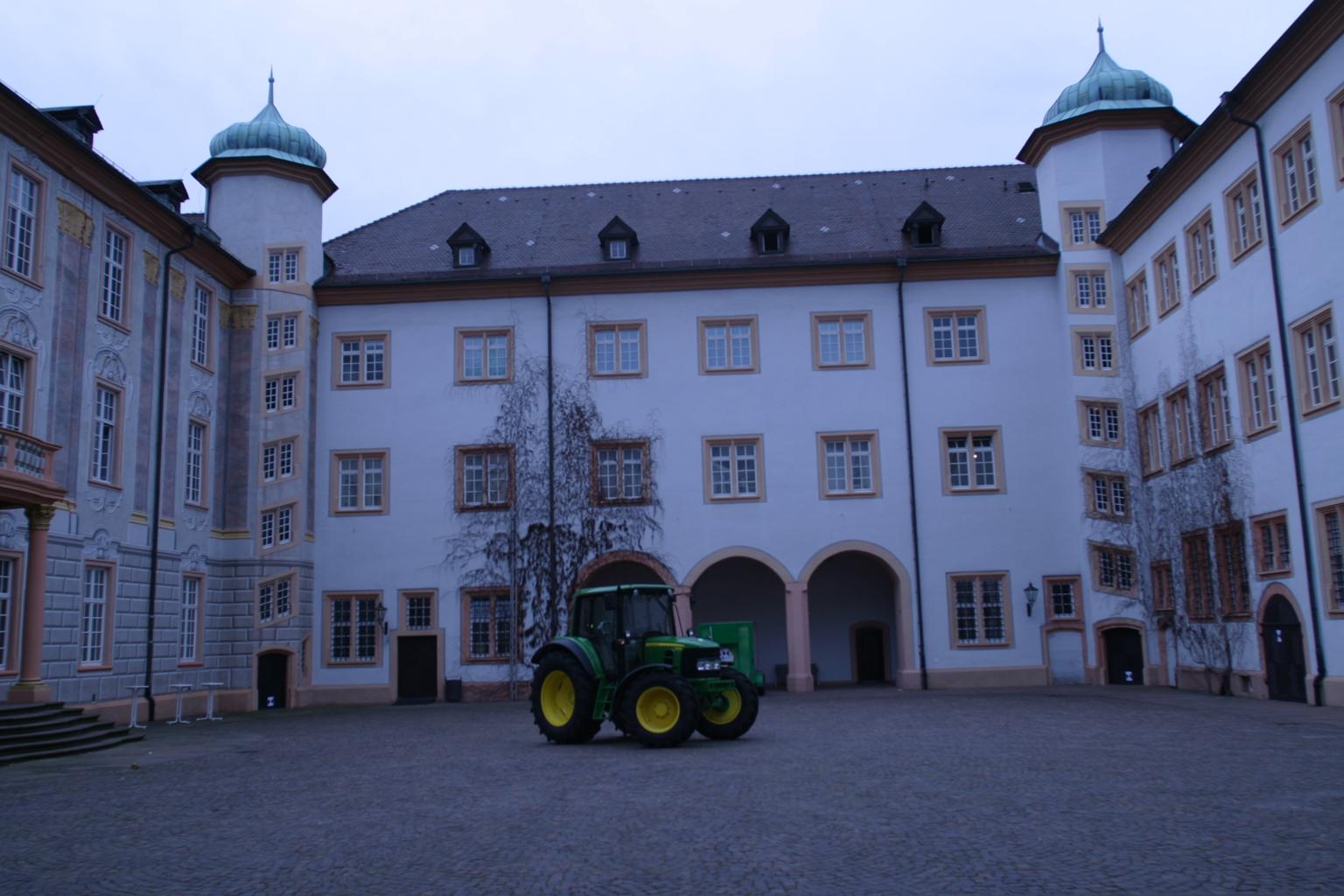}
		}
		{
			\includegraphics[width=0.47\linewidth]{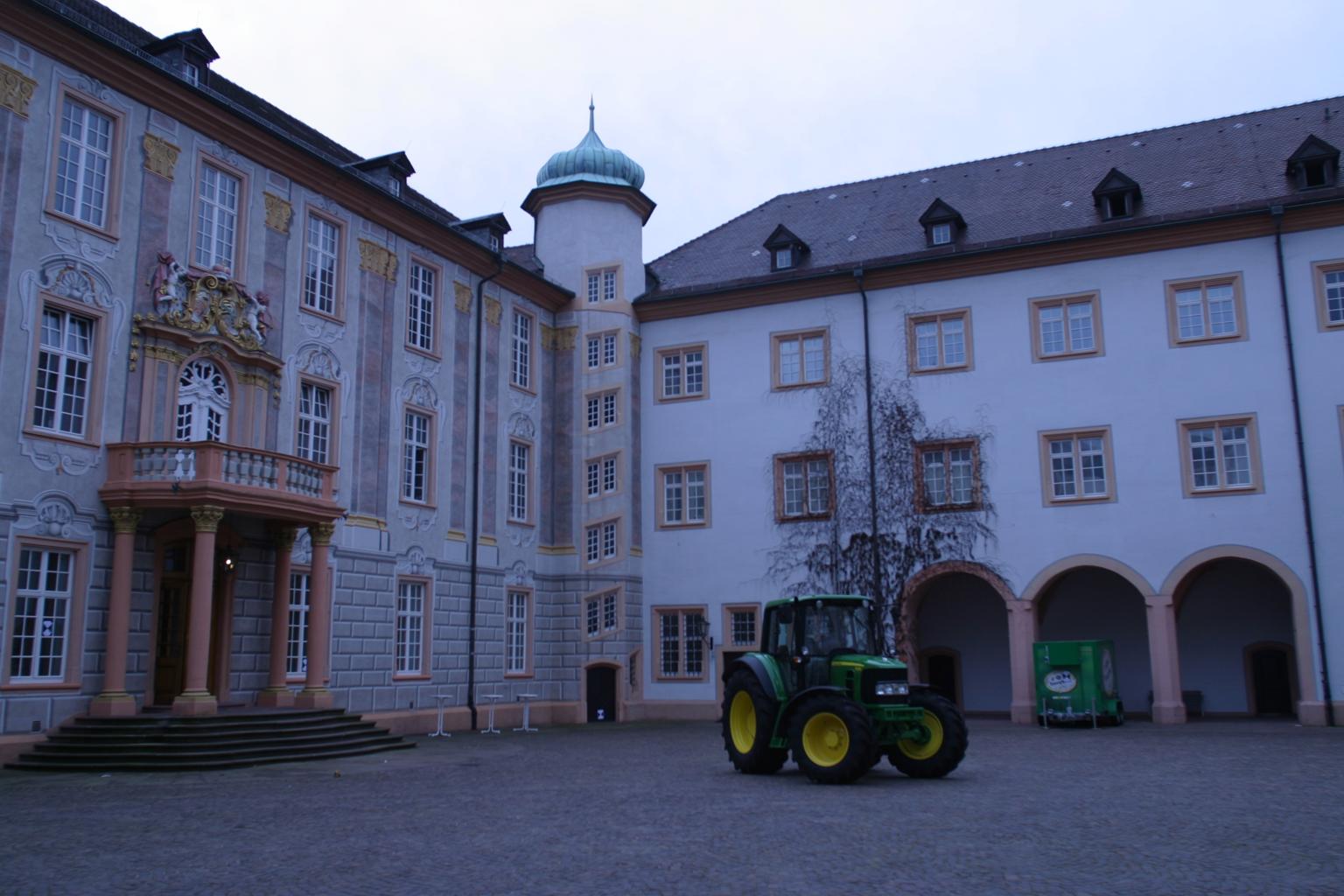}
		}
	\end{center}
	\vspace{-0.15in}
	\caption{A sample image pair of \texttt{EPFL Castle-P19} dataset.}
	\label{fig:epfl}
\end{figure}

We further provide an experiment on real-world images from the \texttt{EPFL Castle-P19} dataset~\cite{strecha2008on}. 
It contains $19$ images in this dataset. We generate $18$ wide-baseline image pairs by grouping adjacent images. 
For each image pair, putative point correspondences are determined by SIFT feature~\cite{lowe2004distinctive}. Then we use RANSAC with iteration number $2000$ and Sampson distance threshold $1.0\times 10^{-3}$ for outlier removal. 
Since the iteration number is sufficiently large and the distance threshold is small, the preserved correspondences can be treated as inliers.

Given correct point correspondences, we compare the relative pose accuracy of different methods. For \texttt{5pt}, \texttt{7pt} and \texttt{8pt} methods, they only use a small portion of the correspondences. For a fair comparison, we repeat them $10$ times using randomly sampled subsets. The rotation and translation errors of different methods are shown in Fig.~\ref{fig:exp_epfl}. 
It can be seen that the mean and median error of our $N$-point method is significantly smaller than those errors produced by the comparison methods.
Specifically, our method achieves a median rotation error of $0.15^\circ$ and a median translation error of $0.56^\circ$. In contrast, \texttt{5pt}, \texttt{7pt} and \texttt{8pt} achieve a median rotation error of $0.33^\circ$, $0.99^\circ$ and $1.02^\circ$, respectively; and they achieve a median translation error of $1.07^\circ$, $3.07^\circ$ and $3.01^\circ$, respectively. 
In this experiment, \texttt{eigen} and our method have a negligible difference. It means that when the noise level is small, local optimization methods might work as well as global optimization methods.

\begin{figure}[tbp]
	\begin{center}
	\subfigure[rotation error $\varepsilon_\text{rot}$]
	{
		\includegraphics[width=0.46\linewidth]{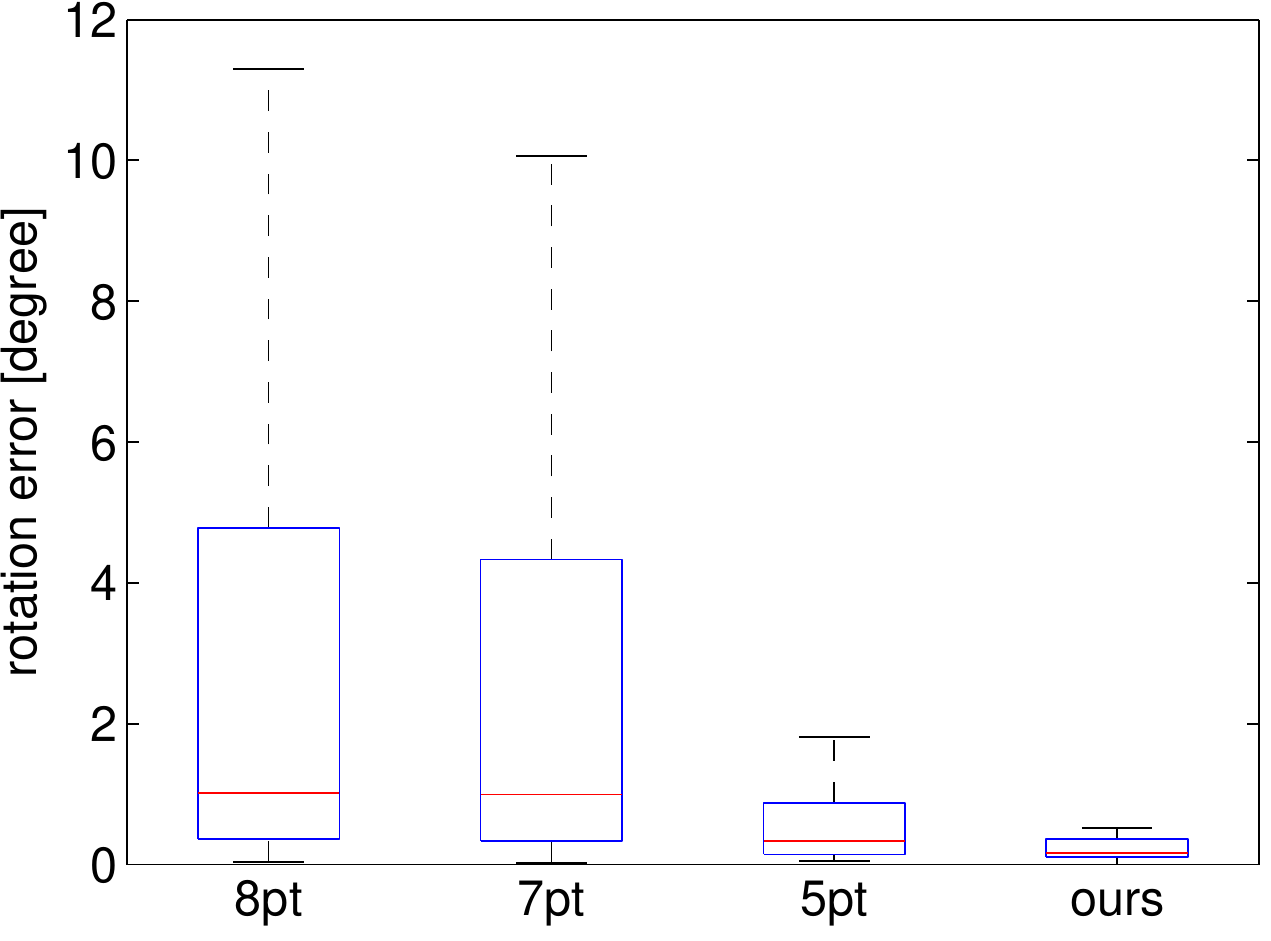}
	}
	\subfigure[translation error $\varepsilon_\text{tran}$]
	{
		\includegraphics[width=0.47\linewidth]{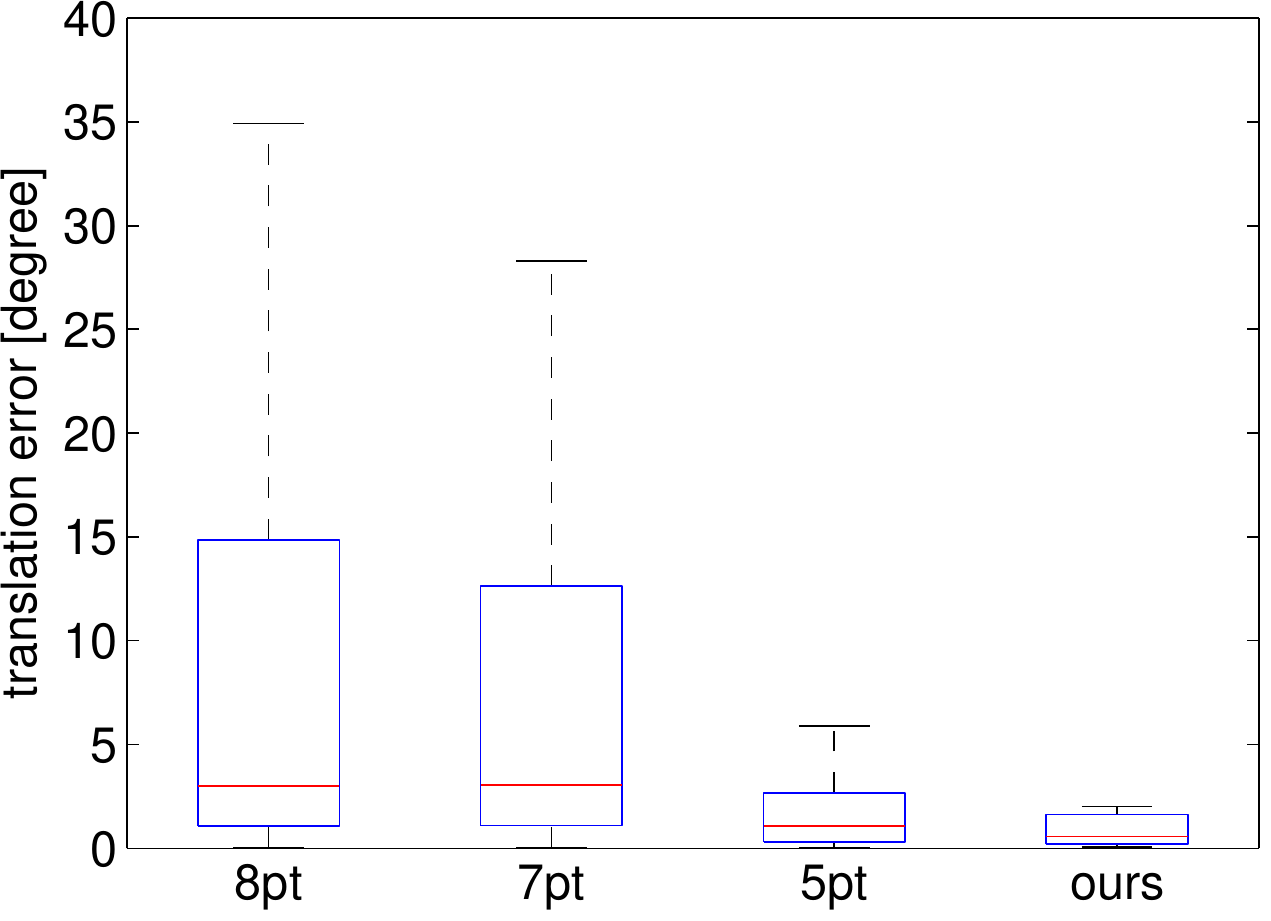}
	}
	\end{center}
	\vspace{-0.15in}
	\caption{Relative pose accuracy of the \texttt{EPFL Castle-P19} dataset. 
	}
	\label{fig:exp_epfl}
\end{figure}

\subsubsection{Performance for Pure Rotation}

\begin{figure}[tbp]
	\begin{center}
	\subfigure[rotation error $\varepsilon_\text{rot}$]
	{
		\includegraphics[width=0.47\linewidth]{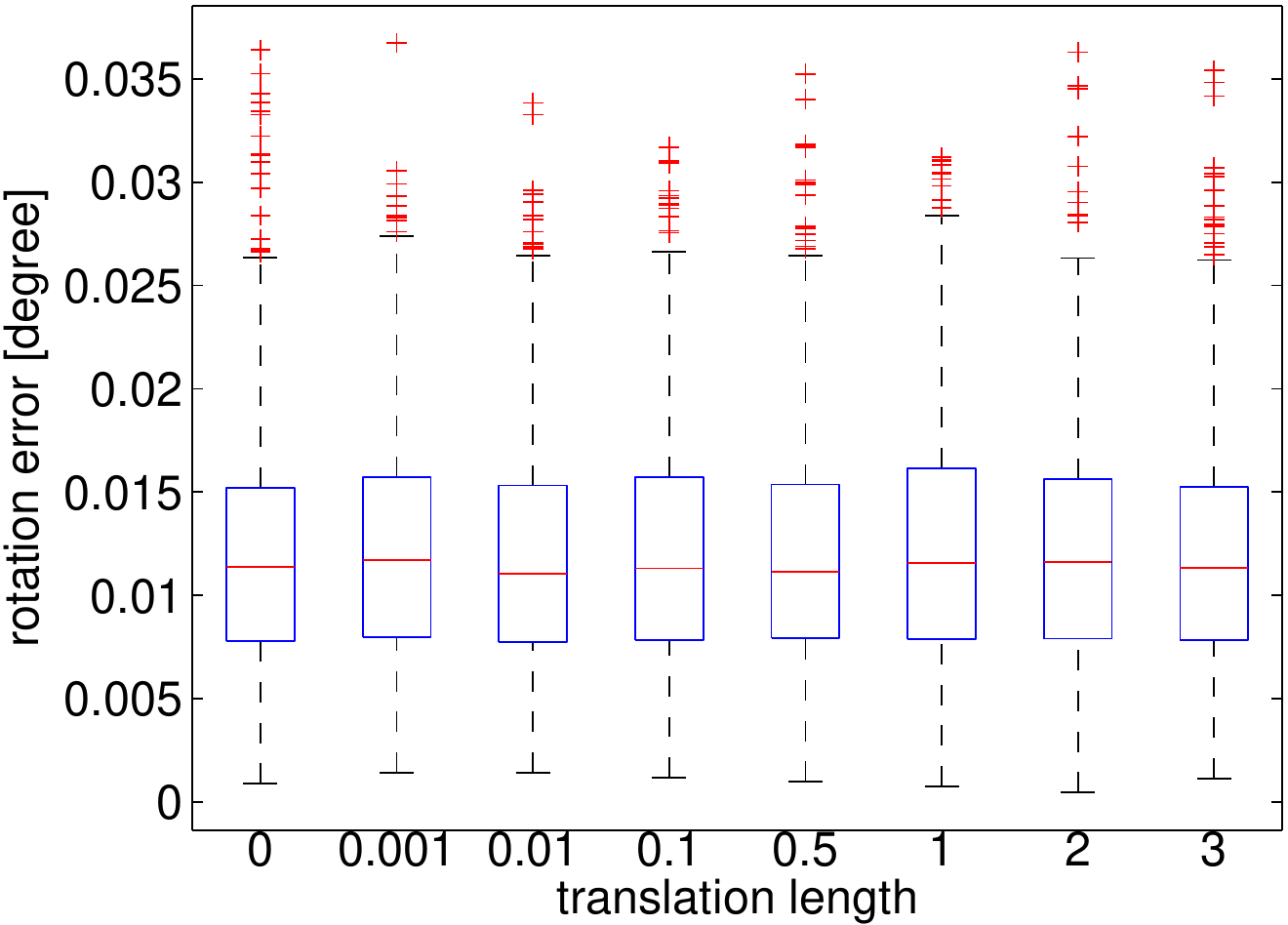}
	}
	\subfigure[translation error $\varepsilon_\text{tran}$]
	{
		\includegraphics[width=0.47\linewidth]{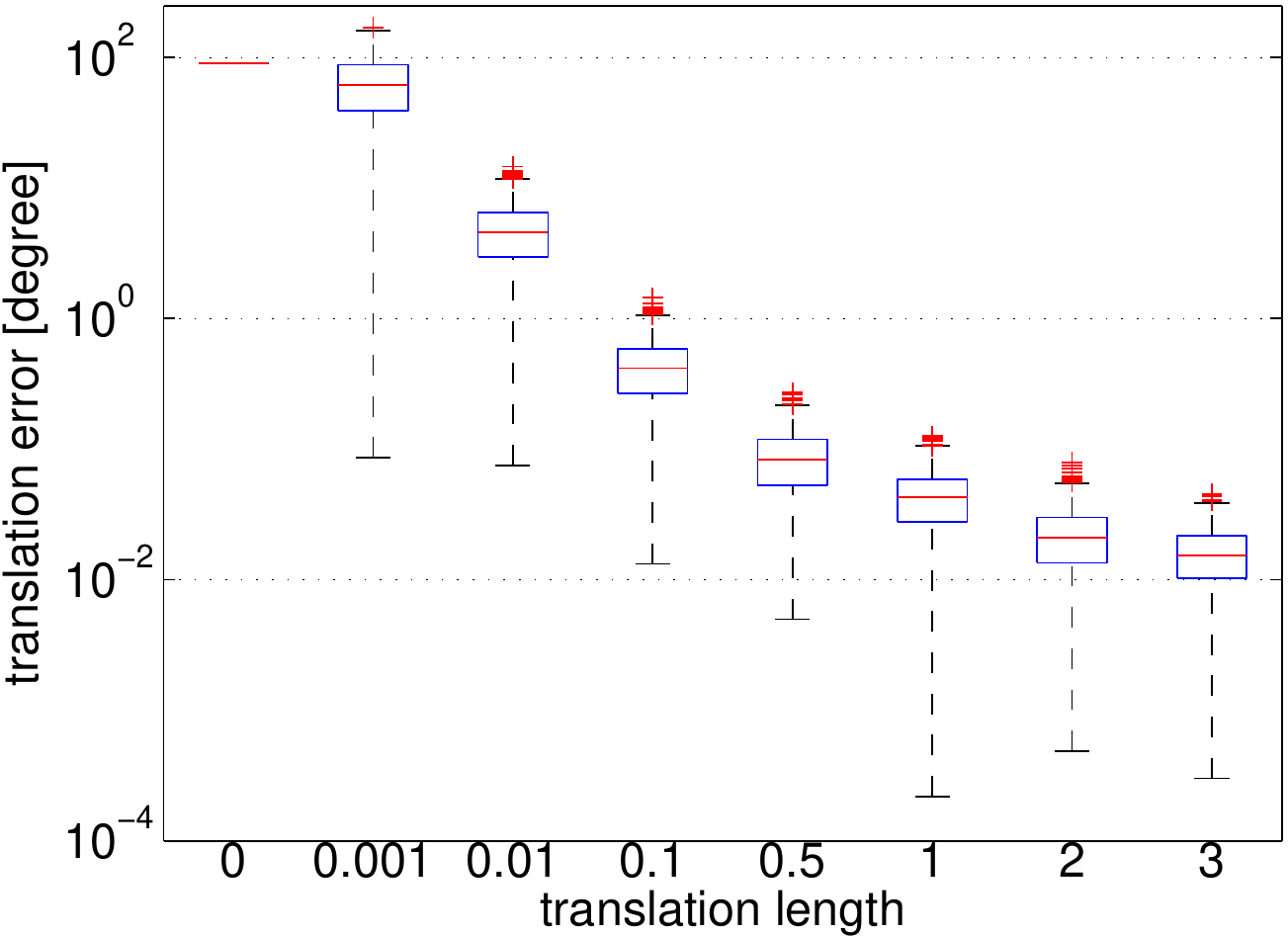}
	}
	\subfigure[pure rotation statistic]
	{
		\includegraphics[width=0.47\linewidth]{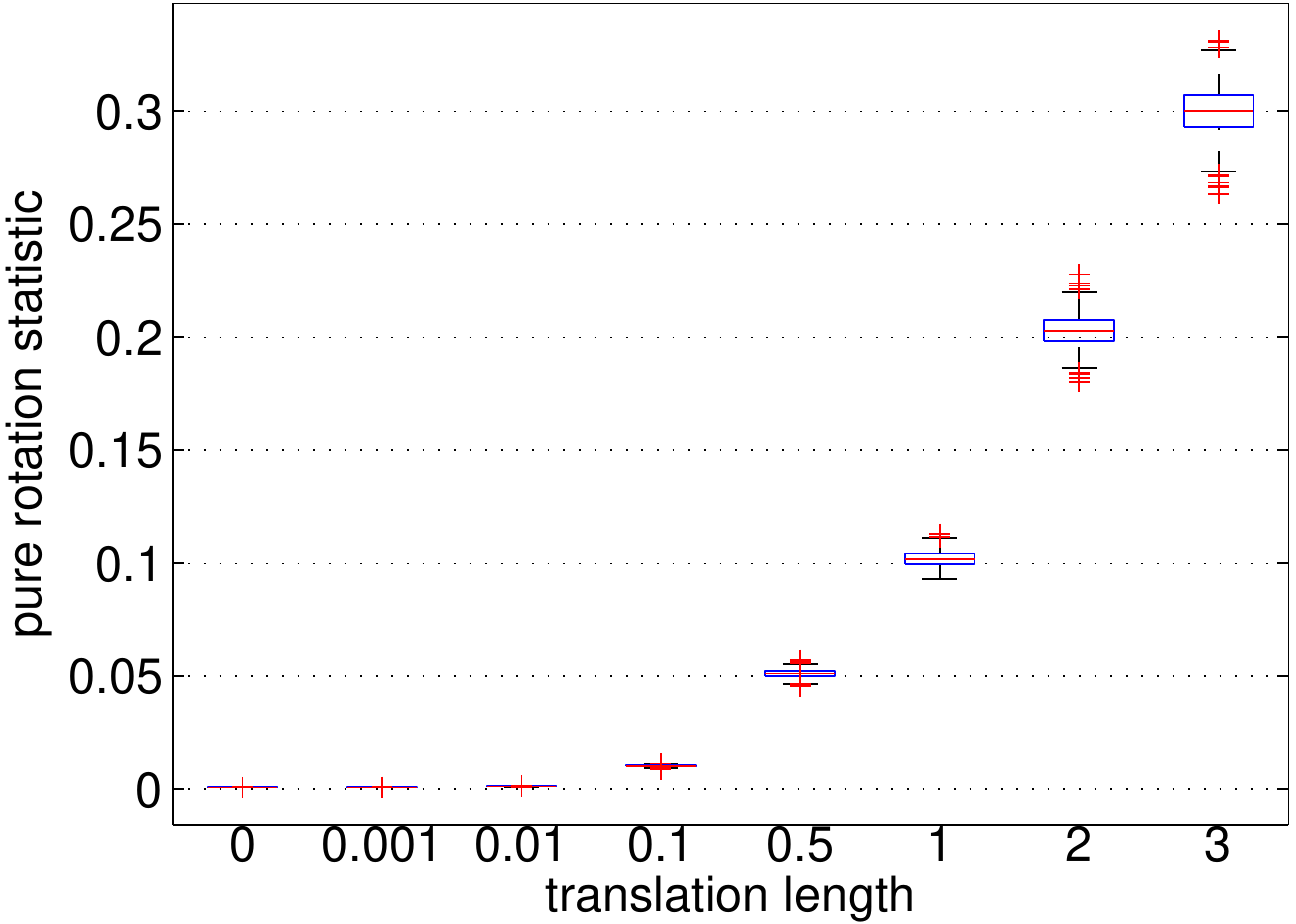}
	}
	\end{center}
	\vspace{-0.15in}
	\caption{Relative pose accuracy with respect to translation length. 
	}
	\label{fig:exp_pure_rotation}
\end{figure}

We use synthetic data to validate the performance of our method for pure rotation. The synthetic scenes are generated as that in Section~\ref{sec:syn_data1}. 
The field-of-view of the camera is $180^\circ$. 
We set the noise level as $0.5$~pixels and the number of point correspondences as $100$. The translation length is varied. 
For each translation length, we repeat the experiments $1000$ times using randomly generated scenes. 
The pose estimation errors with respect to translation length is shown in Fig.~\ref{fig:exp_pure_rotation}. 
It can be seen that the rotation error is not affected by translation length. In contrast, the translation error is greatly affected by translation length\footnote{When the ground truth of translation is zero, the translation error $\varepsilon_{\text{rot}}$ is ill-defined. We define $\varepsilon_{\text{rot}}$ as $90^\circ$, because the expectation of $\varepsilon_{\text{rot}}$ is $90^\circ$ when the estimated translation is uniformly random. }. The larger translation length tends to result in more accurate translation estimation. 

In Figure~\ref{fig:exp_pure_rotation}(c), we plot a pure rotation statistic~\cite{cai2018equivalent} with respect to translation length. The statistic is defined as the mean of $\frac{\x_i \times \R^\star \x'_i}{\|\x_i\| \|\x'_i\|}$, $i = 1, \cdots, N$. 
It can be seen that this statistic is discriminative to identify pure rotation scenarios. 
When (near) pure rotation occurs, this statistic is (near) zero. 
From the above experiments, it certifies that the proposed $N$-point method can be applied to pure rotation cases.

\subsection{Global Optimality of $N$-Point Method}

Recall that $\rank(\X_e^\star)$ is used to verify the optimality of the proposed $N$-point method. 
Thus the second largest singular value is the key to ensure rank-$1$ condition and global optimality.
We use the synthetic scenes defined in Section~\ref{sec:syn_data1} to demonstrate the second largest singular value of $\rank(\X_e^\star)$. 
Figure~\ref{fig:exp_optimality} shows the cumulative distribution functions (CDF) of the second largest singular values under different settings.
Given a threshold to determine the rank of $\rank(\X_e^\star)$, the proportion of instances with global optimality can be obtained from corresponding CDF. 
The success rate of the global optimality depends on the threshold of singular value and the accuracy of the SDP solver.

In Figure~\ref{fig:exp_optimality}(a)(b), the number of point correspondences is fixed to $100$, and the noise level is varied for forward and sideways motion modes.
In Figure~\ref{fig:exp_optimality}(c)(c), the noise level is fixed to $0.5$ pixels, and the number of point correspondences is varied for forward and sideways motion modes. 
We have the following observations.  
(1) When the threshold of singular value is $1.0 \times 10^{-6}$, the global optimality can be obtained in most cases. (2) Smaller noise levels or more point correspondences will result in higher success rates of global optimality. (3) For usual cases whose noise level is below $1$ pixel and the number of point correspondences is above $20$, the global optimality can be obtained in most cases.

\begin{figure}[tbp]
	\begin{center}
	\subfigure[forward \& varied noise level]
	{
		\includegraphics[width=0.47\linewidth]{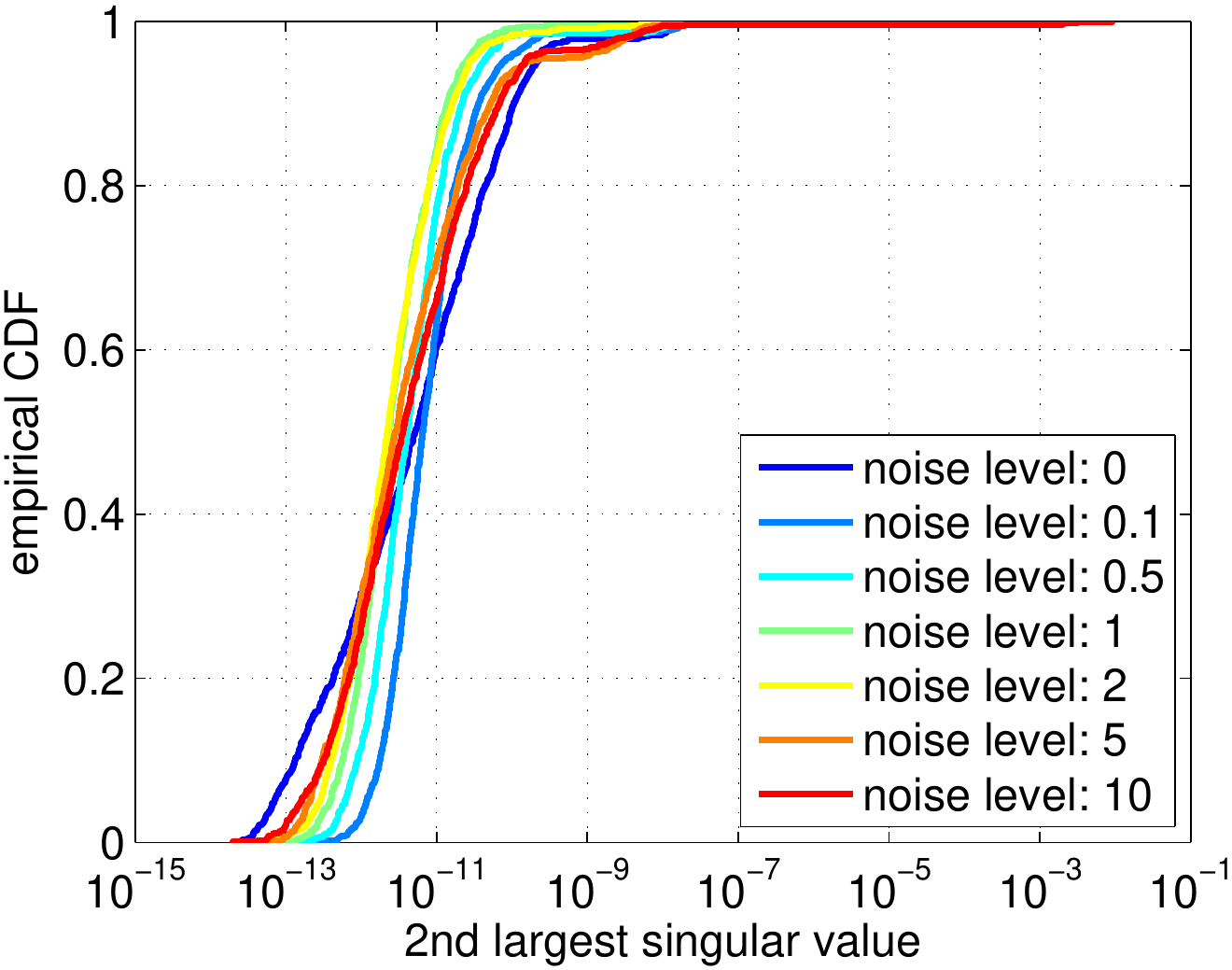}
	}
	\subfigure[sideways \& varied noise level]
	{
		\includegraphics[width=0.475\linewidth]{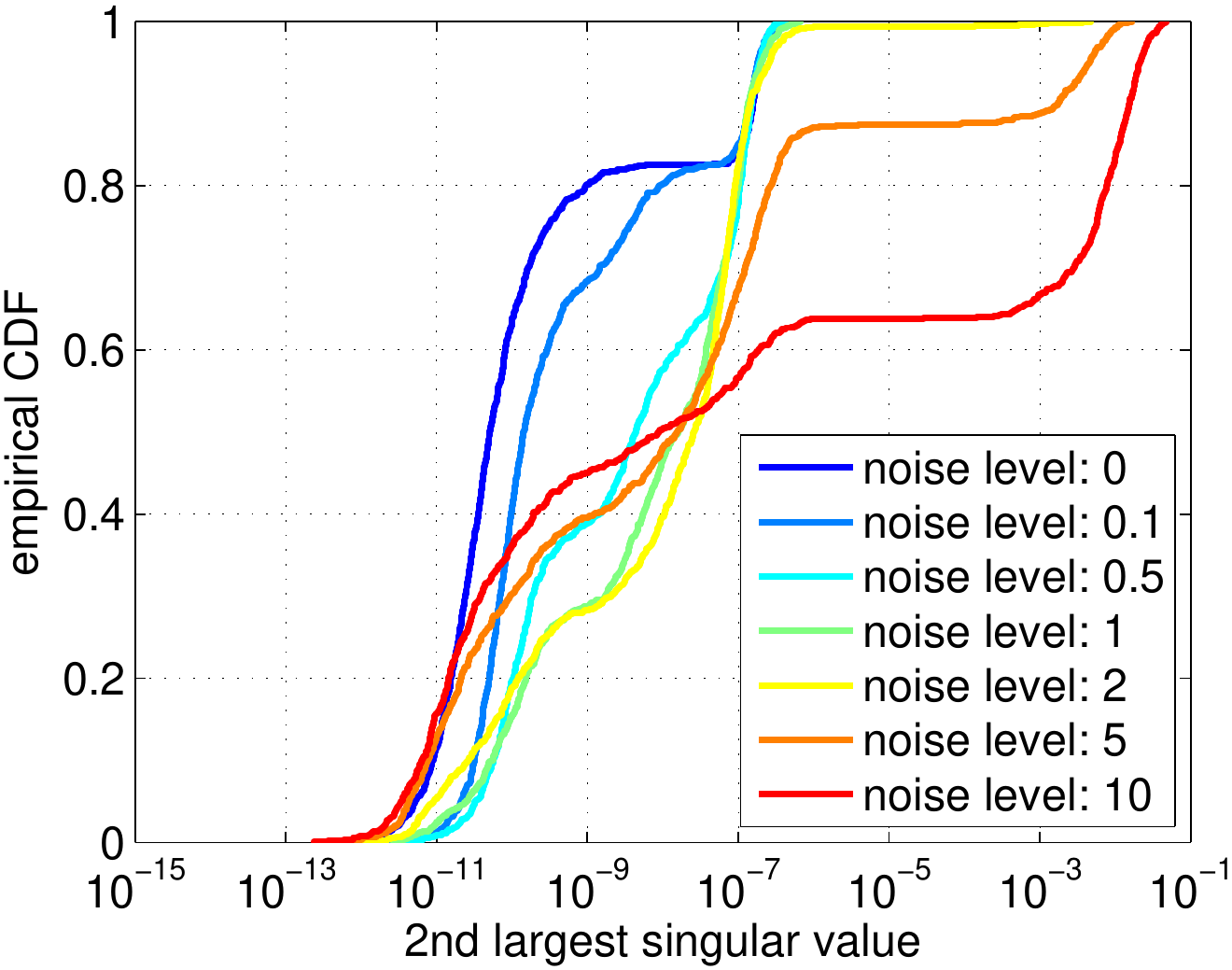}
	}
	\subfigure[forward \& varied point num.]
	{
		\includegraphics[width=0.47\linewidth]{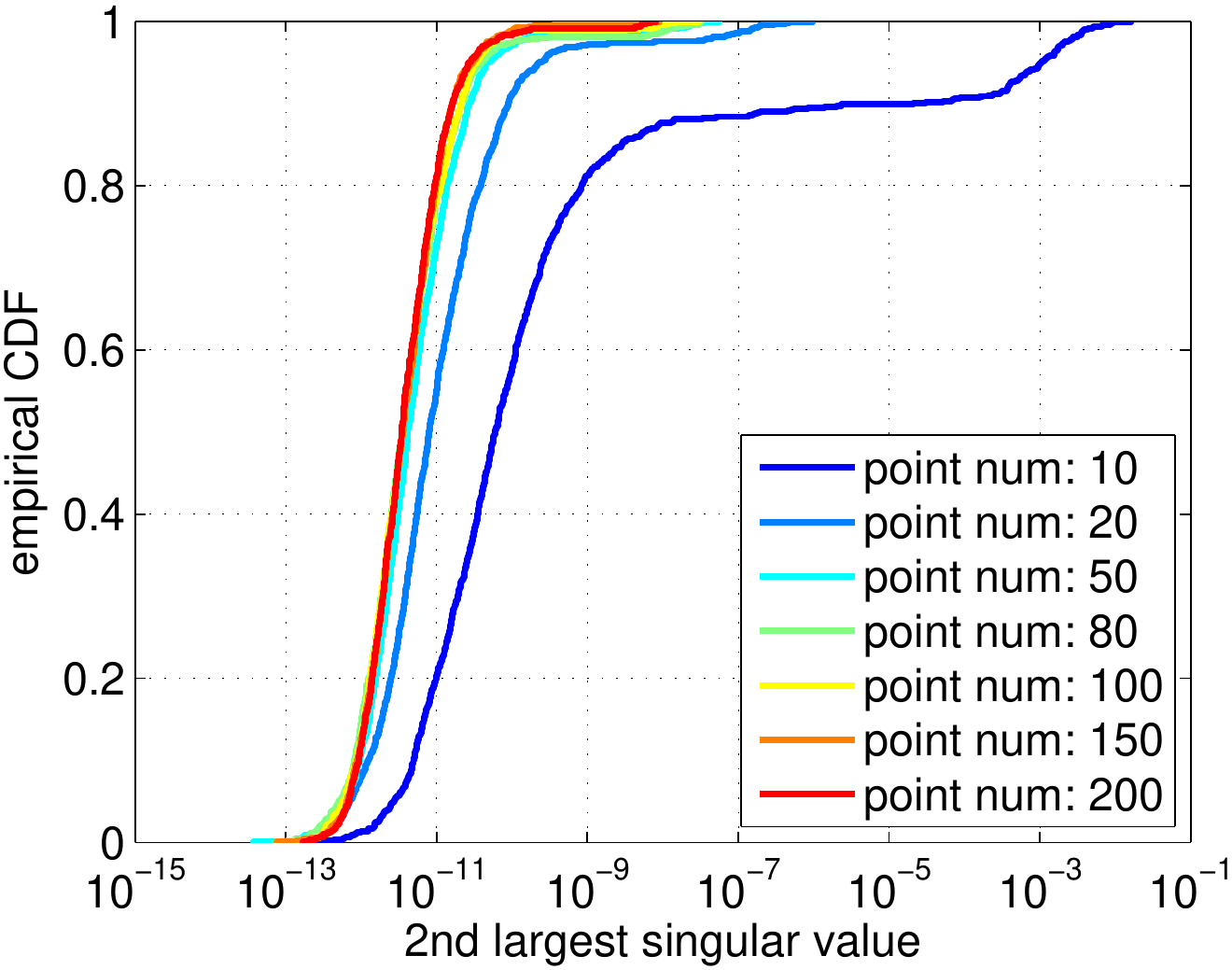}
	}
	\subfigure[sideways \& varied point num.]
	{
		\includegraphics[width=0.475\linewidth]{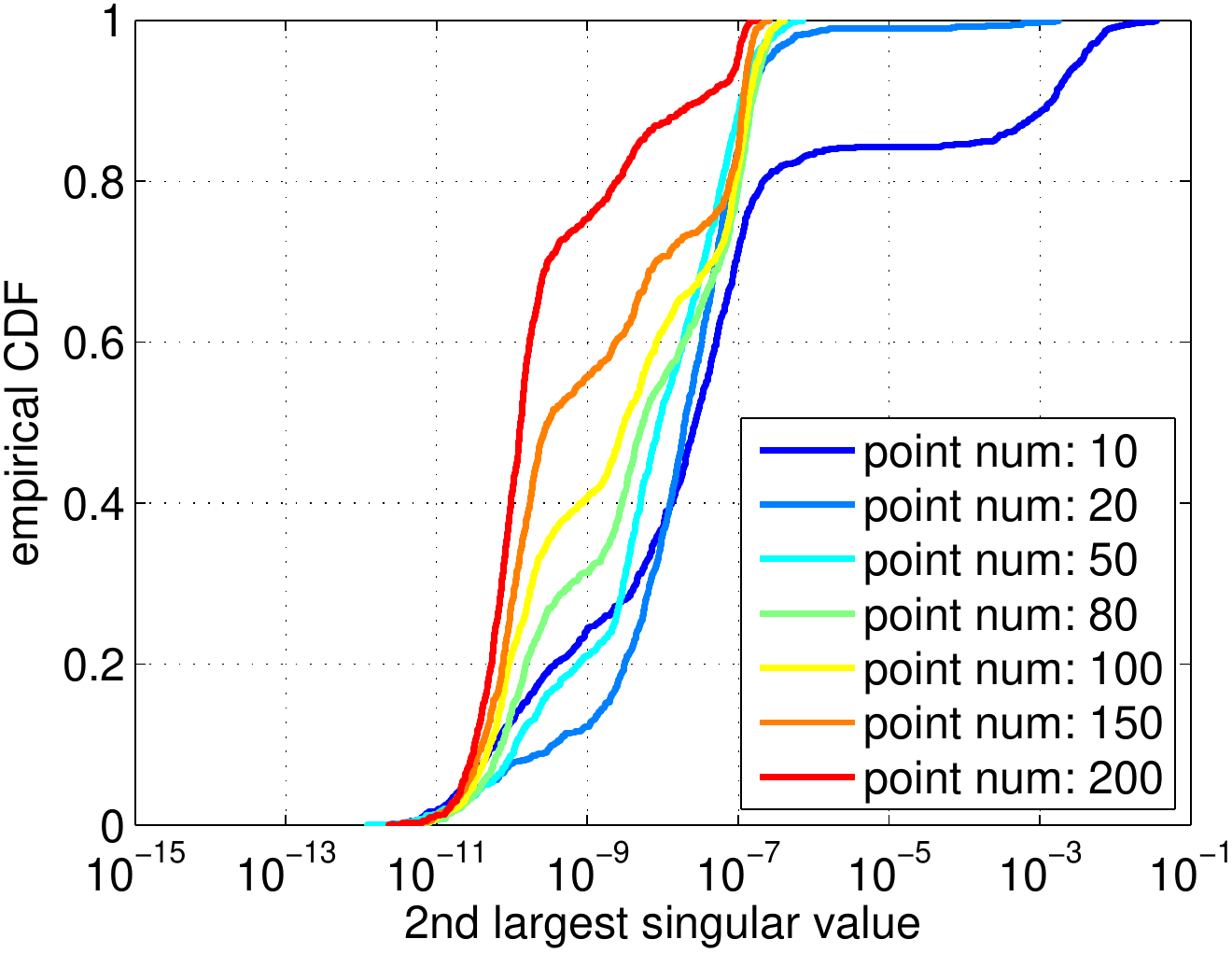}
	}
	\end{center}
	\vspace{-0.15in}
	\caption{The second largest singular value with respect to image noise levels and numbers of point correspondences. 
	}
	\label{fig:exp_optimality}
\end{figure}

\subsection{Performance of Robust $N$-Point Method}

We test the robust $N$-point method on both synthetic data and real-world data.
The parameter $\tau^2_{\text{min}}$ is set as $6.0\times 10^{-7}$. 

\subsubsection{Synthetic Data}
\label{sec:synthetic_data}

The synthetic scene is generated as that in Section~\ref{sec:syn_data1}. The noise level is fixed to $0.5$~pixels, and the correspondence number is fixed to $100$. The outlier ratio is varied from $0\%$ to $100\%$ with a step size of $5\%$. 
For each outlier ratio, we repeat the experiments $100$ times using randomly generated data. 

\begin{figure}[tbp]
	\begin{center}
	\subfigure[Robustness comparison]
	{
		\includegraphics[width=0.47\linewidth]{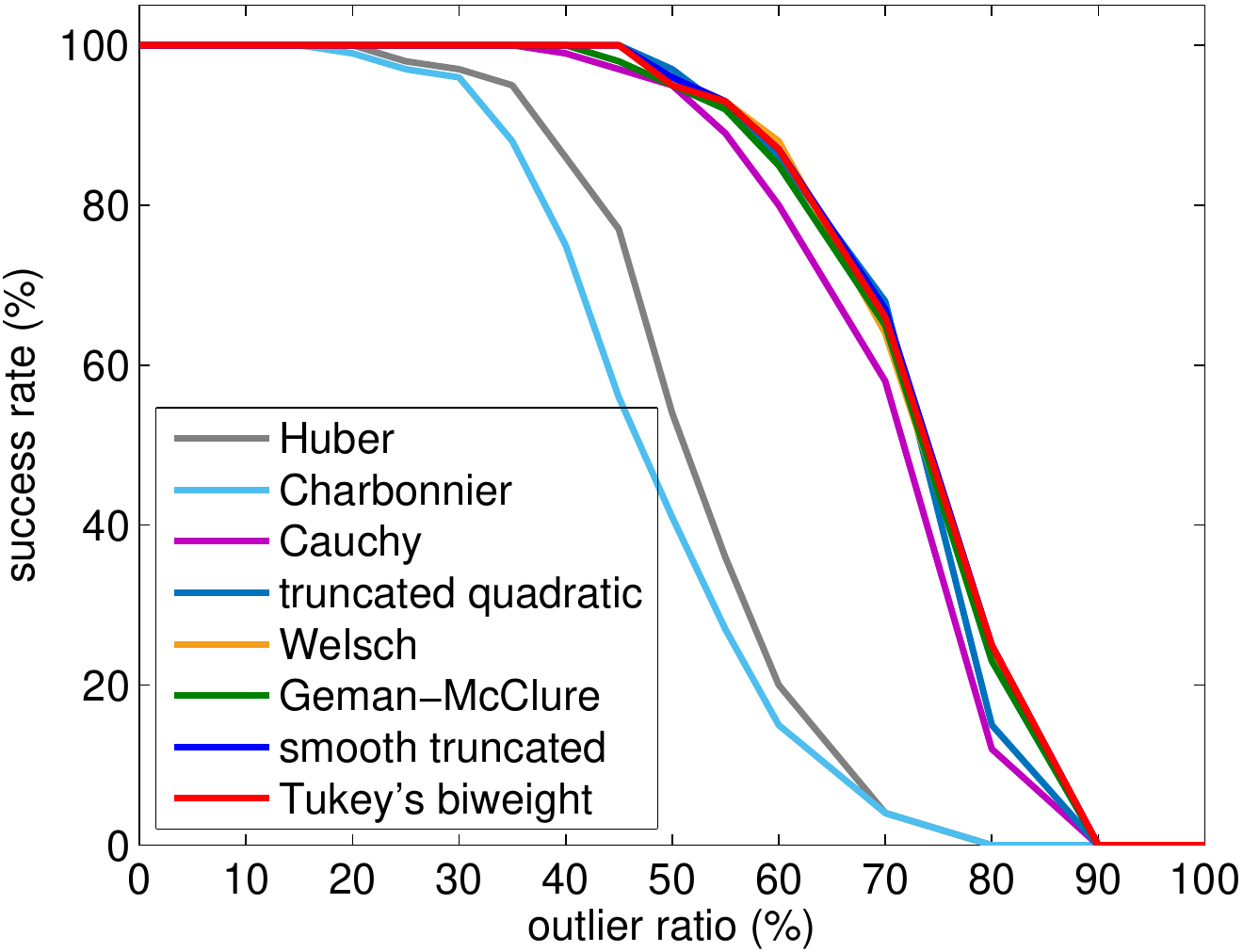}
	}
	\subfigure[Robustness of Welsch function]
	{
		\includegraphics[width=0.47\linewidth]{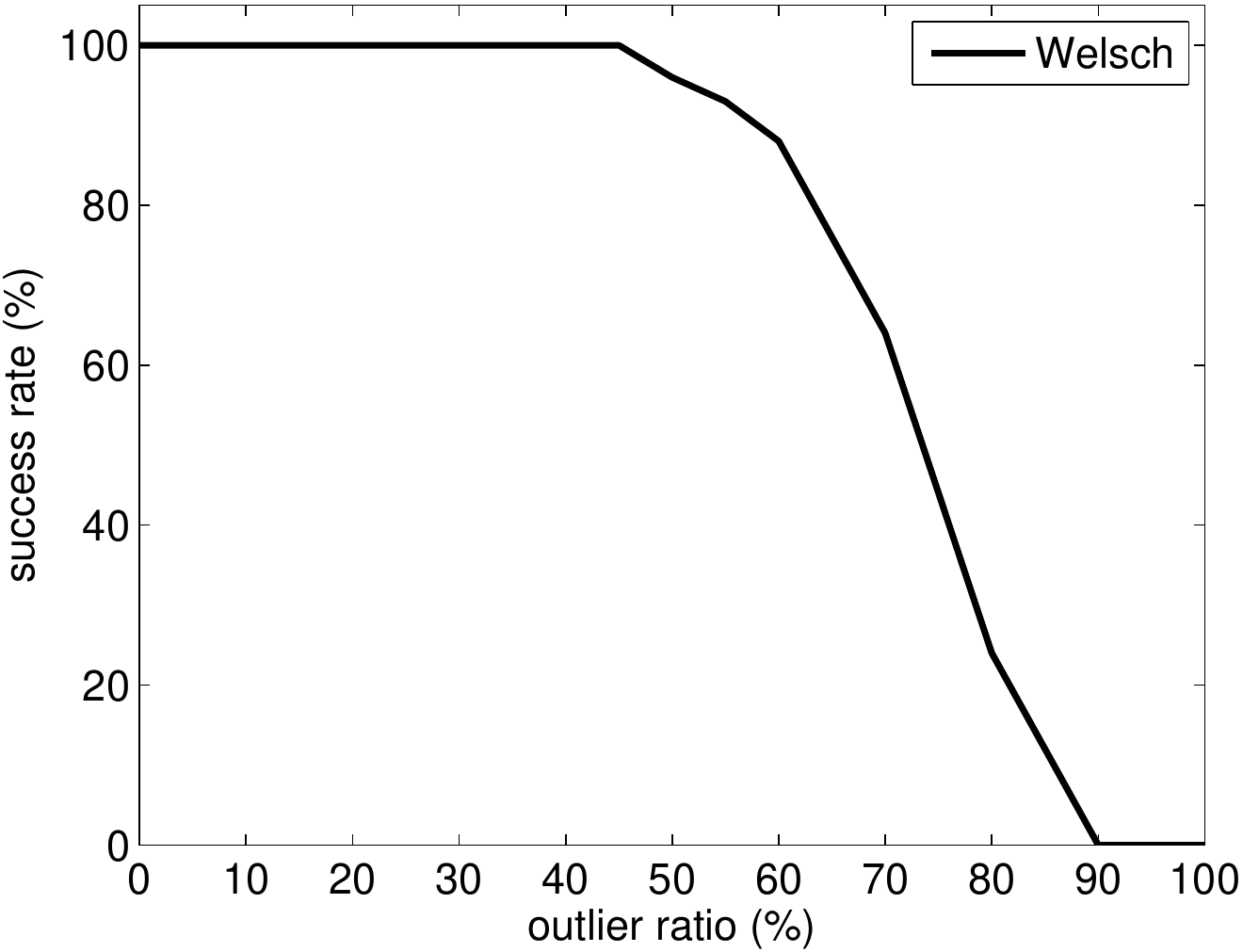}
	}
	\end{center}
	\vspace{-0.15in}
	\caption{Success rate of the robust $N$-point method. (a) Robustness comparison with different loss functions. (b) Robustness of Welsch function. (It is plotted separately for better visualization)}
	\label{fig:exp_outlier_loss_comp}
\end{figure}

We evaluate $8$ robust loss functions in Fig.~\ref{fig:kernel1}, and report their success rates. 
The success rate is the ratio of successful cases to the overall trials. 
A trial is treated as a success if $\varepsilon_{\text{rot}} \le 0.15^\circ$ and $\varepsilon_{\text{tran}} \le 0.5^\circ$. 
The results are shown in Fig.~\ref{fig:exp_outlier_loss_comp}.
It can be seen that truncated quadratic, Welsch, smooth truncated quadratic, and Turkey's biweight functions can tolerate up to $45\%$ outliers. 
Cauchy and Geman-McClure functions can tolerate up to $40\%$ outliers; Charbonnier and Huber functions can tolerate up to $15\%$ outliers.
In the following experiments, we will use Welsch function in the robust $N$-point method since it has superior performance.

\begin{figure*}[htbp]
	\begin{center}
	\subfigure[mean $\varepsilon_\text{rot}$]
	{
		\includegraphics[width=0.23\linewidth]{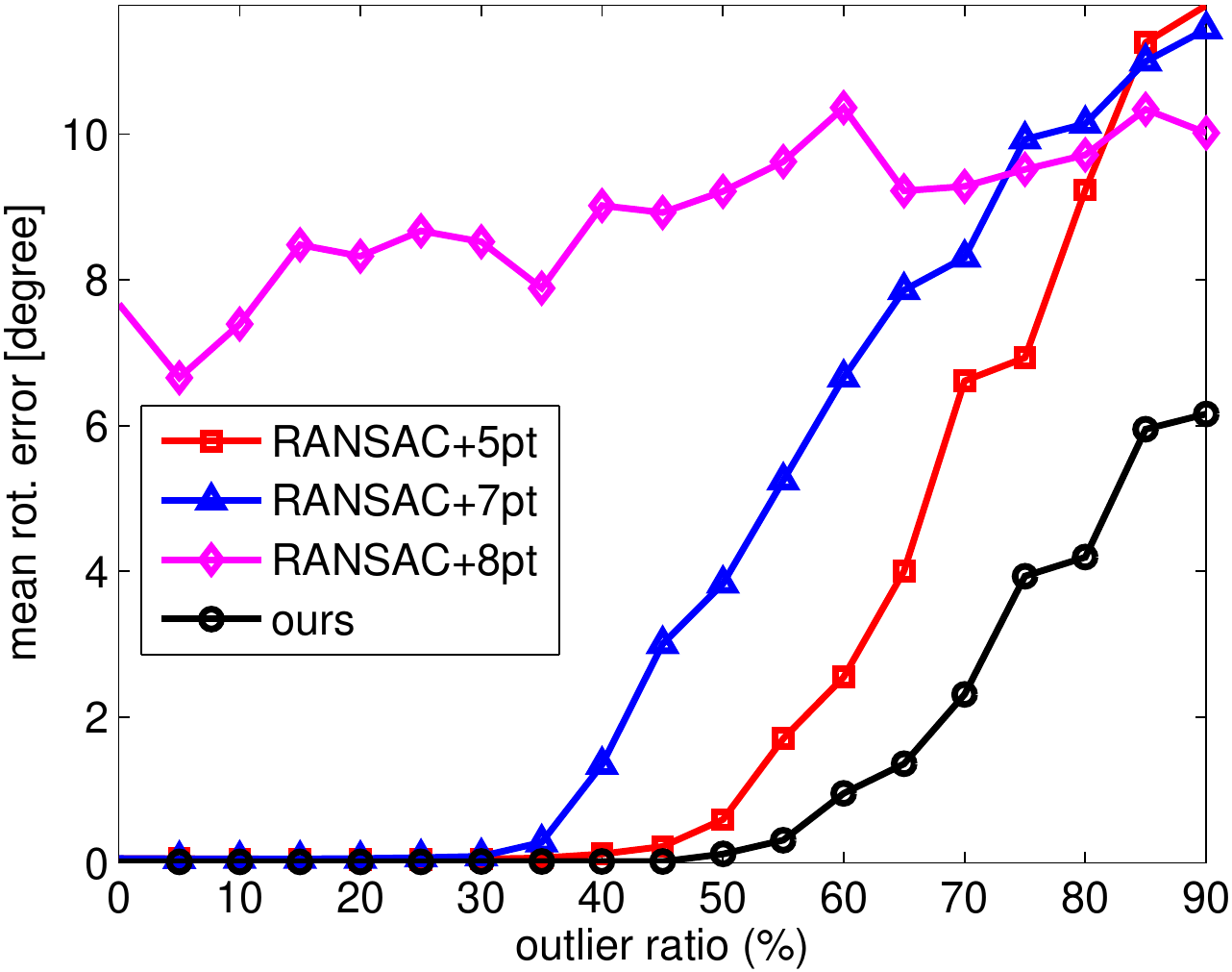}
	}
	\subfigure[mean $\varepsilon_\text{tran}$]
	{
		\includegraphics[width=0.23\linewidth]{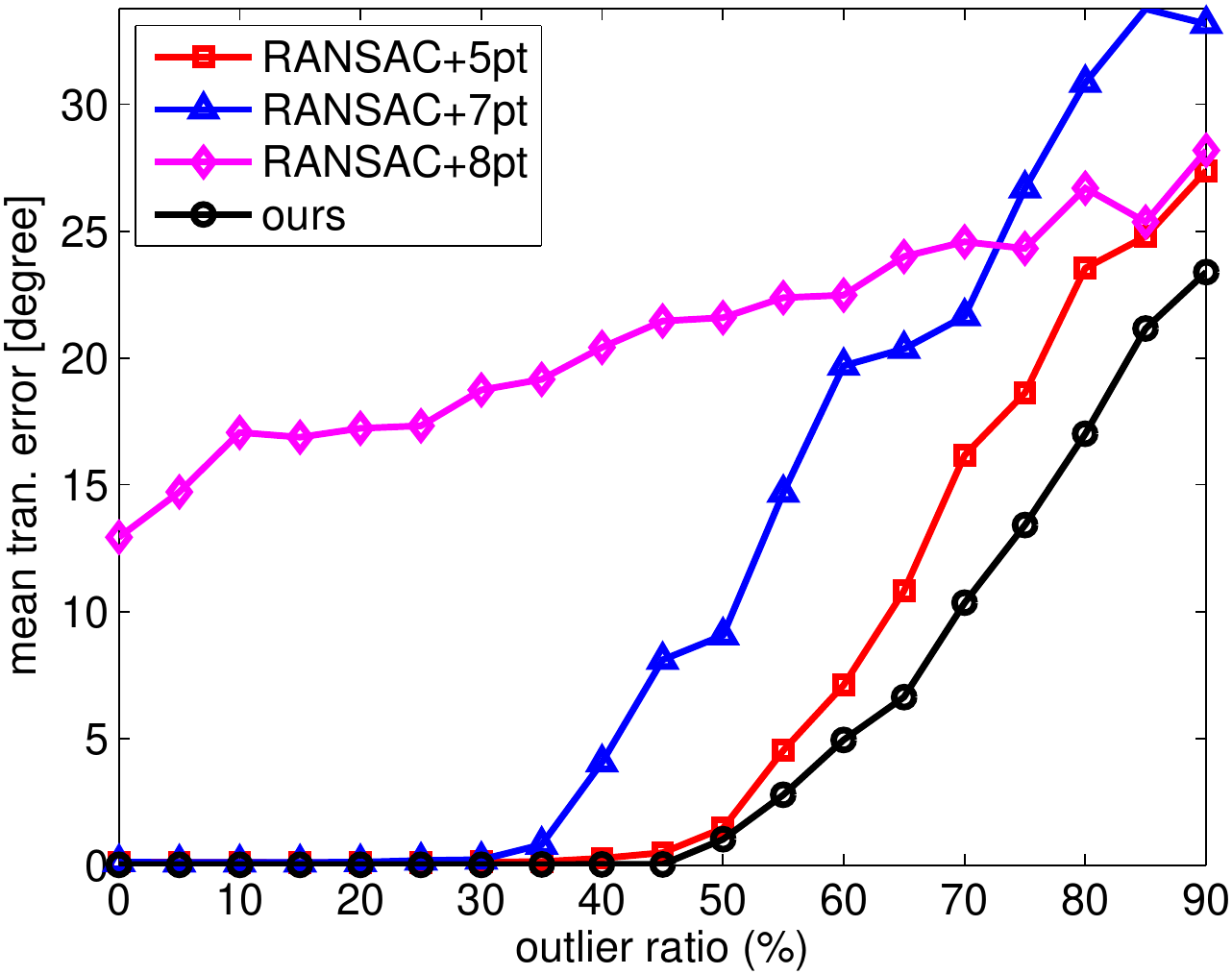}
	}
	\subfigure[median $\varepsilon_\text{rot}$]
	{
		\includegraphics[width=0.23\linewidth]{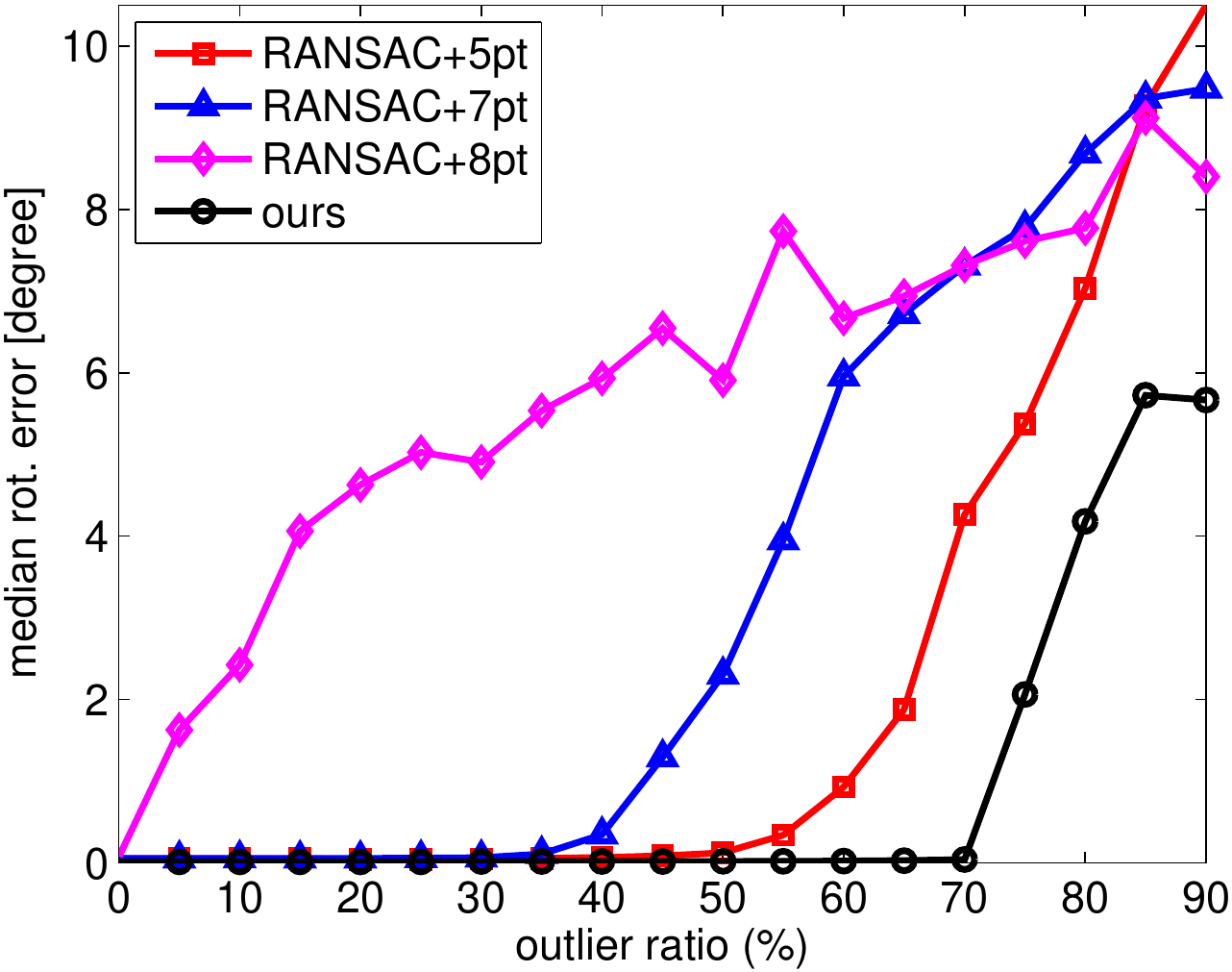}
	}
	\subfigure[median $\varepsilon_\text{tran}$]
	{
		\includegraphics[width=0.23\linewidth]{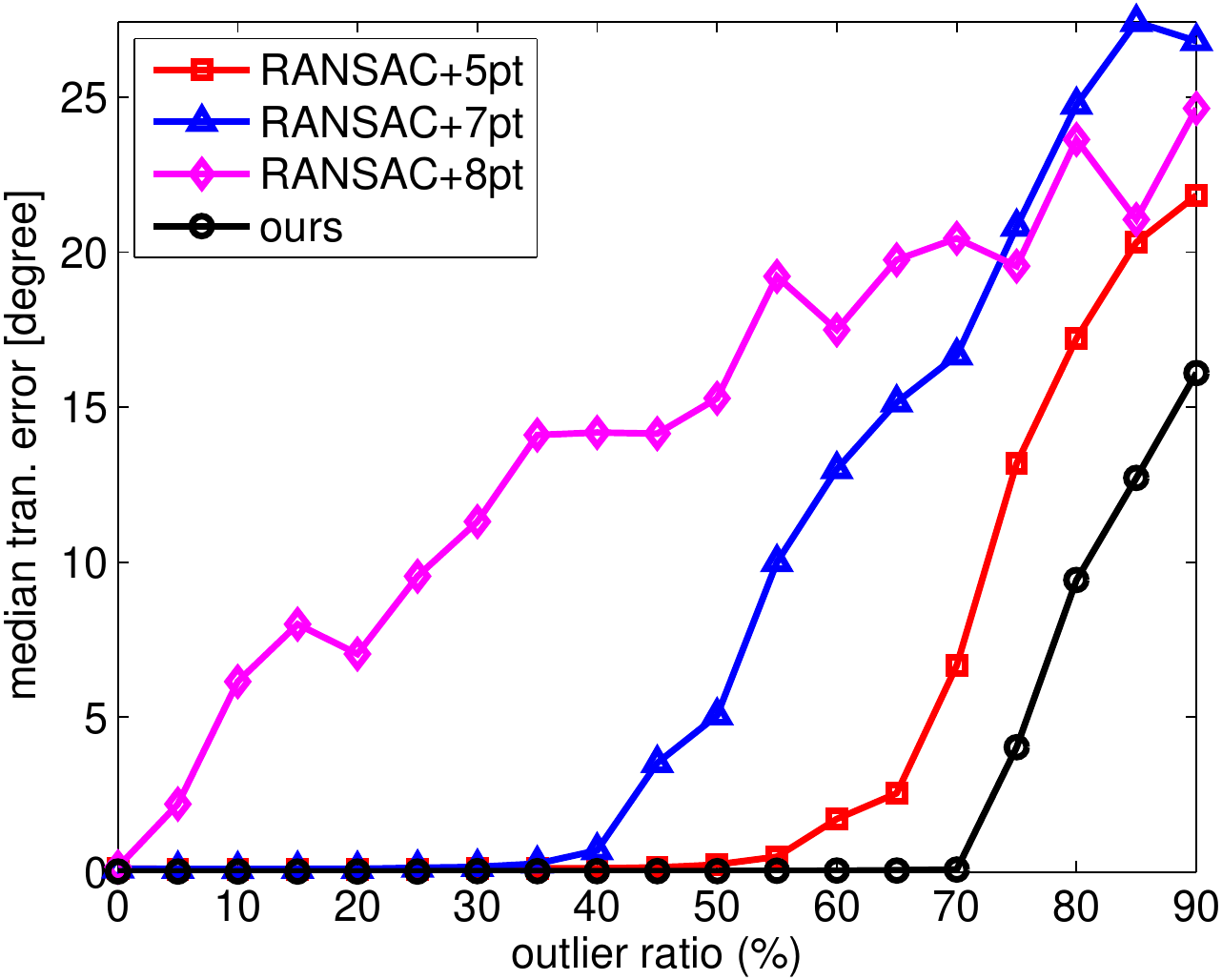}
	} \\
	\subfigure[mean $\varepsilon_\text{rot}$]
	{
		\includegraphics[width=0.23\linewidth]{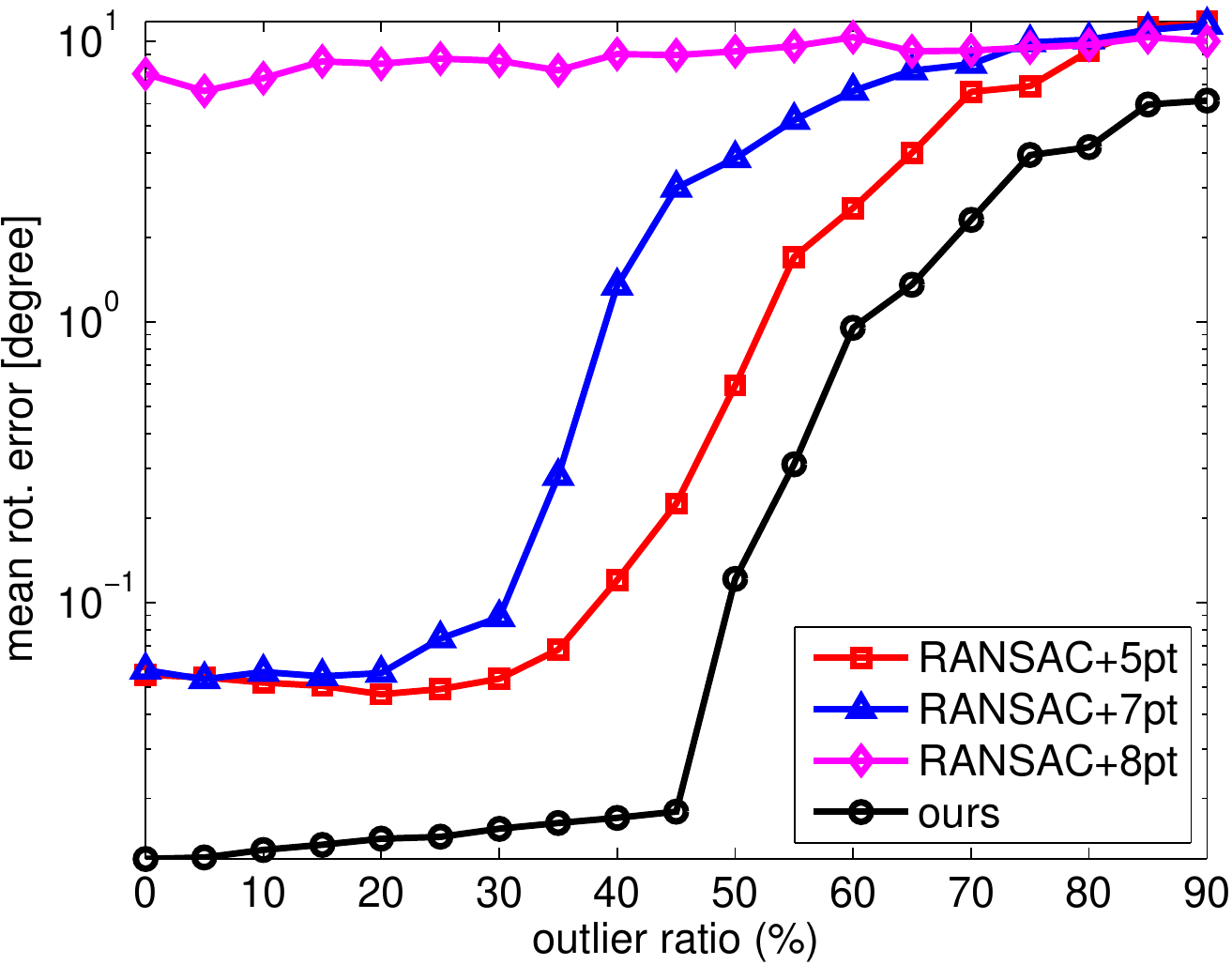}
	}
	\subfigure[mean $\varepsilon_\text{tran}$]
	{
		\includegraphics[width=0.23\linewidth]{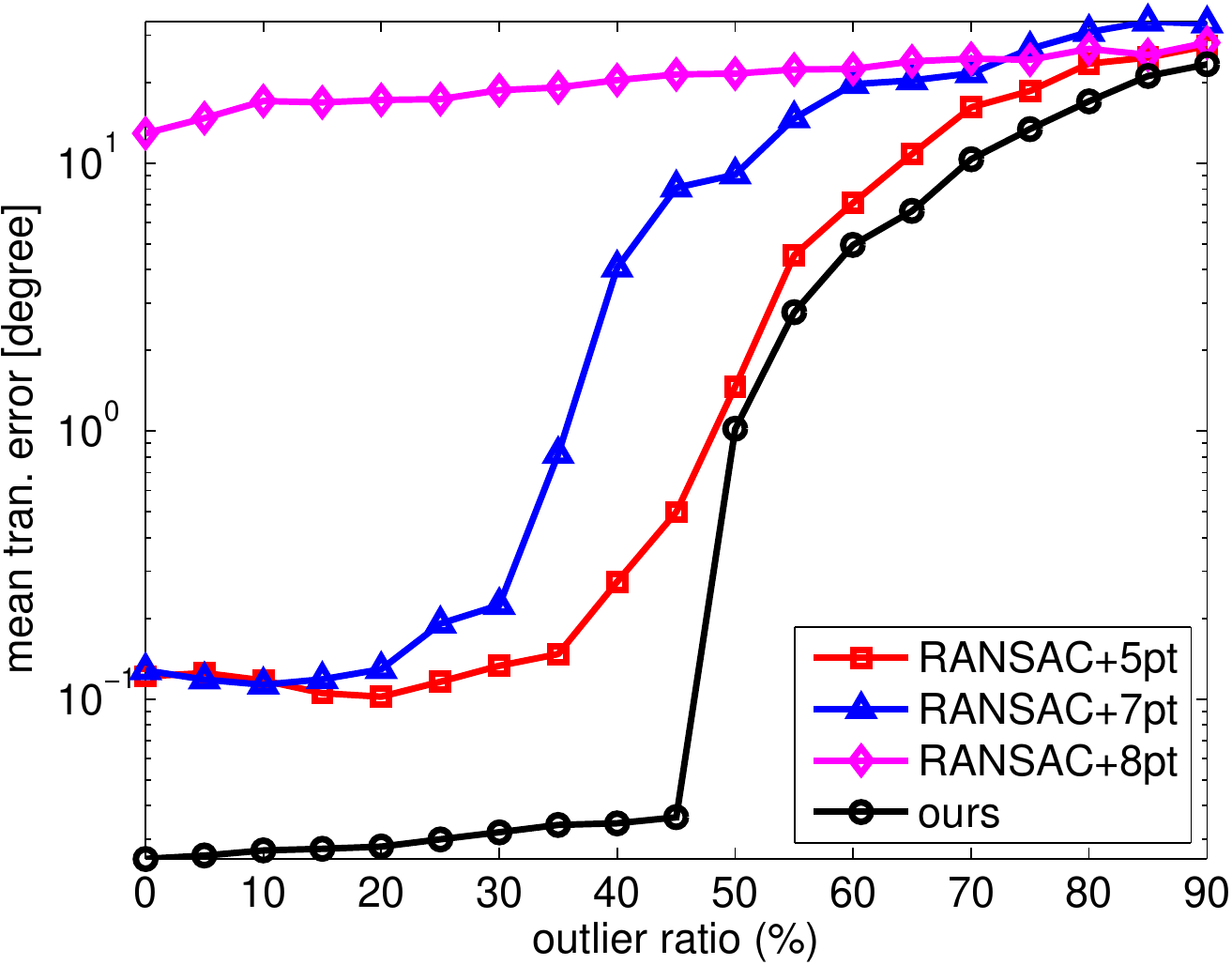}
	}
	\subfigure[median $\varepsilon_\text{rot}$]
	{
		\includegraphics[width=0.23\linewidth]{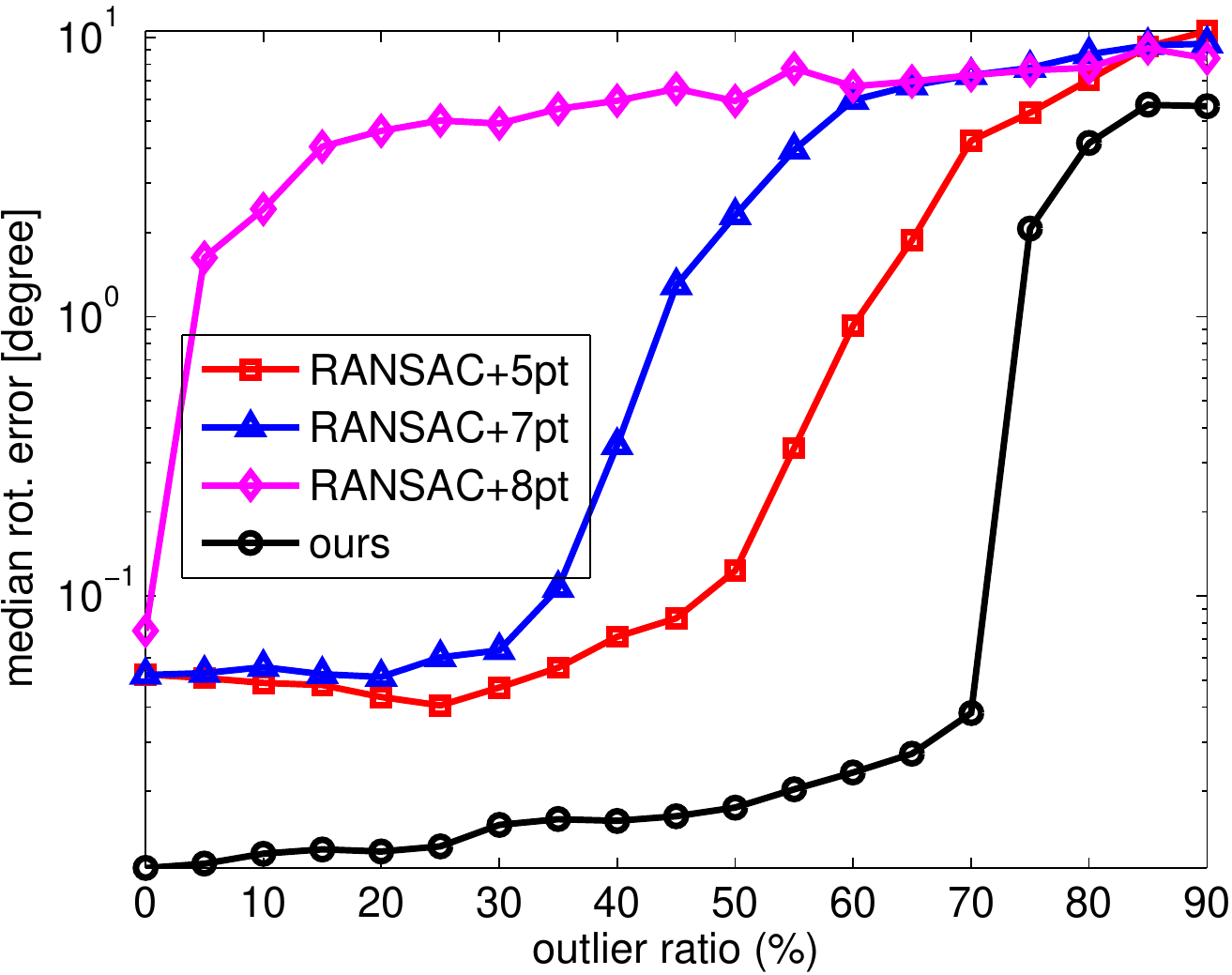}
	}
	\subfigure[median $\varepsilon_\text{tran}$]
	{
		\includegraphics[width=0.23\linewidth]{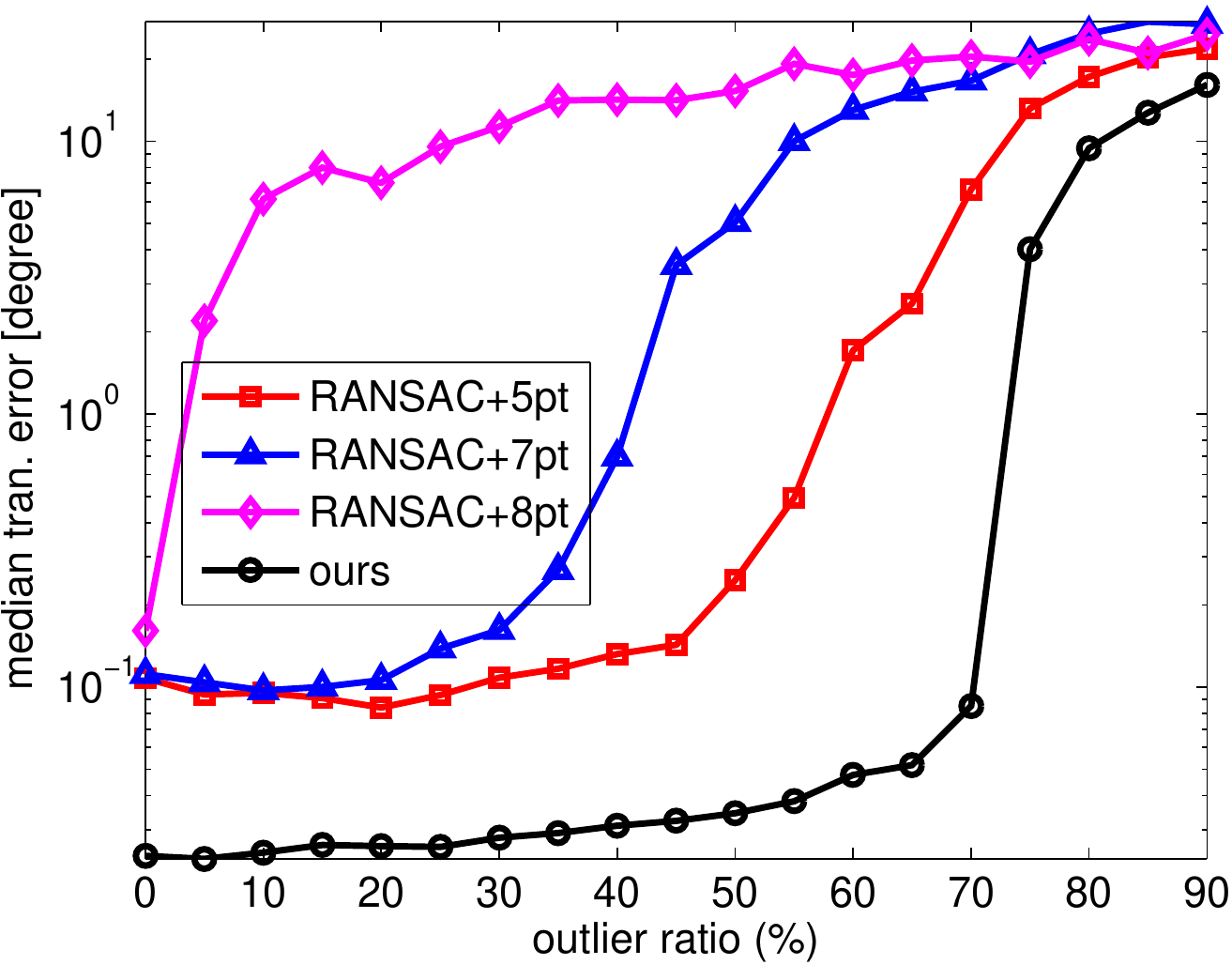}
	}
	\end{center}
	\vspace{-0.15in}
	\caption{Relative pose accuracy with respect to outlier ratios. Figures in the first row and the second rows use linear scale and logarithmic scale for vertical axis, respectively. 
	}
	\label{fig:exp_sythetic_outlier}
\end{figure*}

The performance of the robust $N$-point method and RANSAC-based methods is shown in Fig.~\ref{fig:exp_sythetic_outlier}. 
Our method can consistently obtain smaller errors than RANSAC-based methods in terms of both mean and median errors. 
From logarithmic scale plots in Fig.~\ref{fig:exp_sythetic_outlier}(e)$\sim$(h), our method has significantly smaller rotation and translation error than other methods when the outlier ratio is below a threshold.

\subsubsection{Real-World Data}

\begin{figure}[tbp]
	\begin{center}
	\subfigure[rotation error $\varepsilon_\text{rot}$]
	{
		\includegraphics[width=0.47\linewidth]{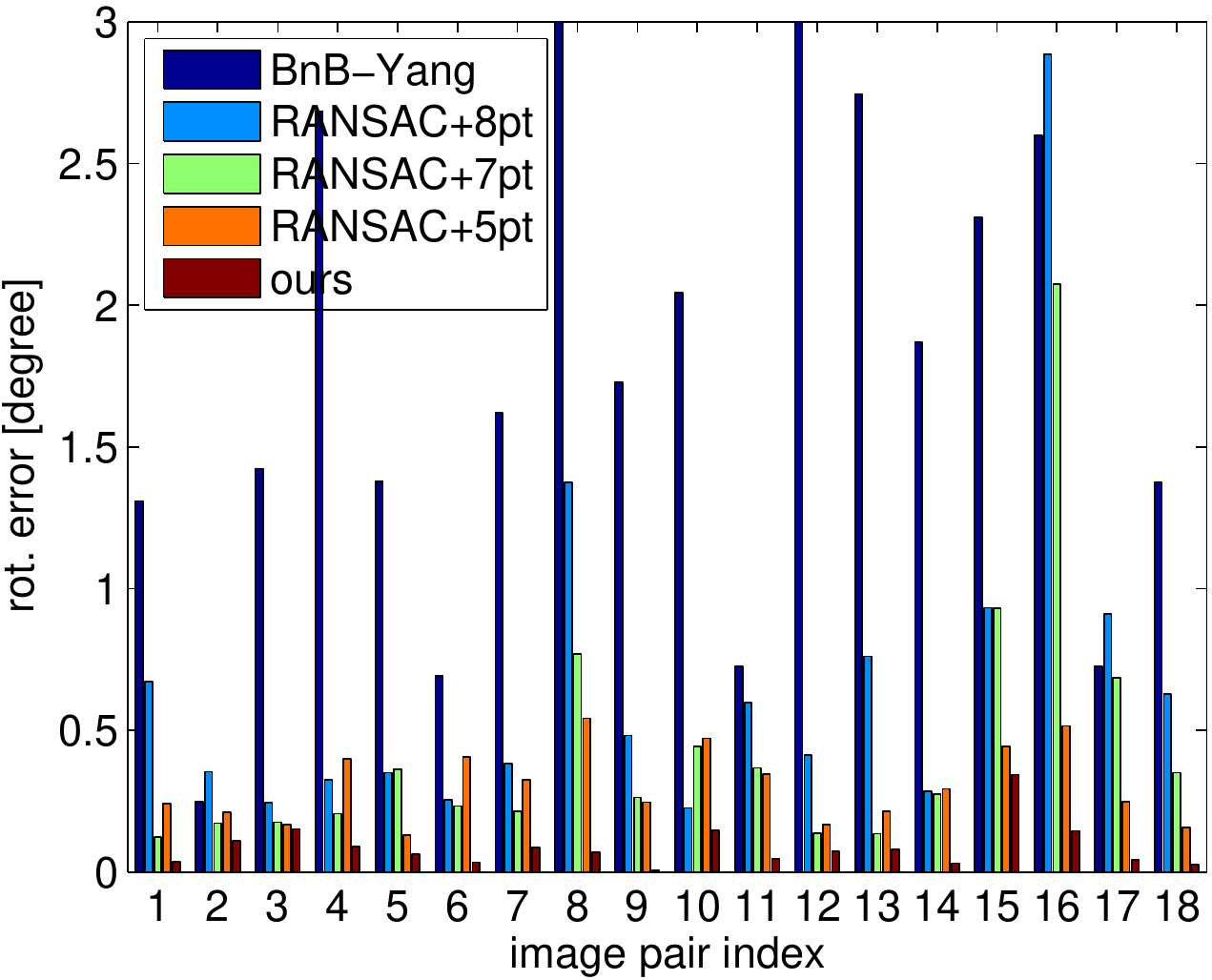}
	}
	\subfigure[translation error $\varepsilon_\text{tran}$]
	{
		\includegraphics[width=0.475\linewidth]{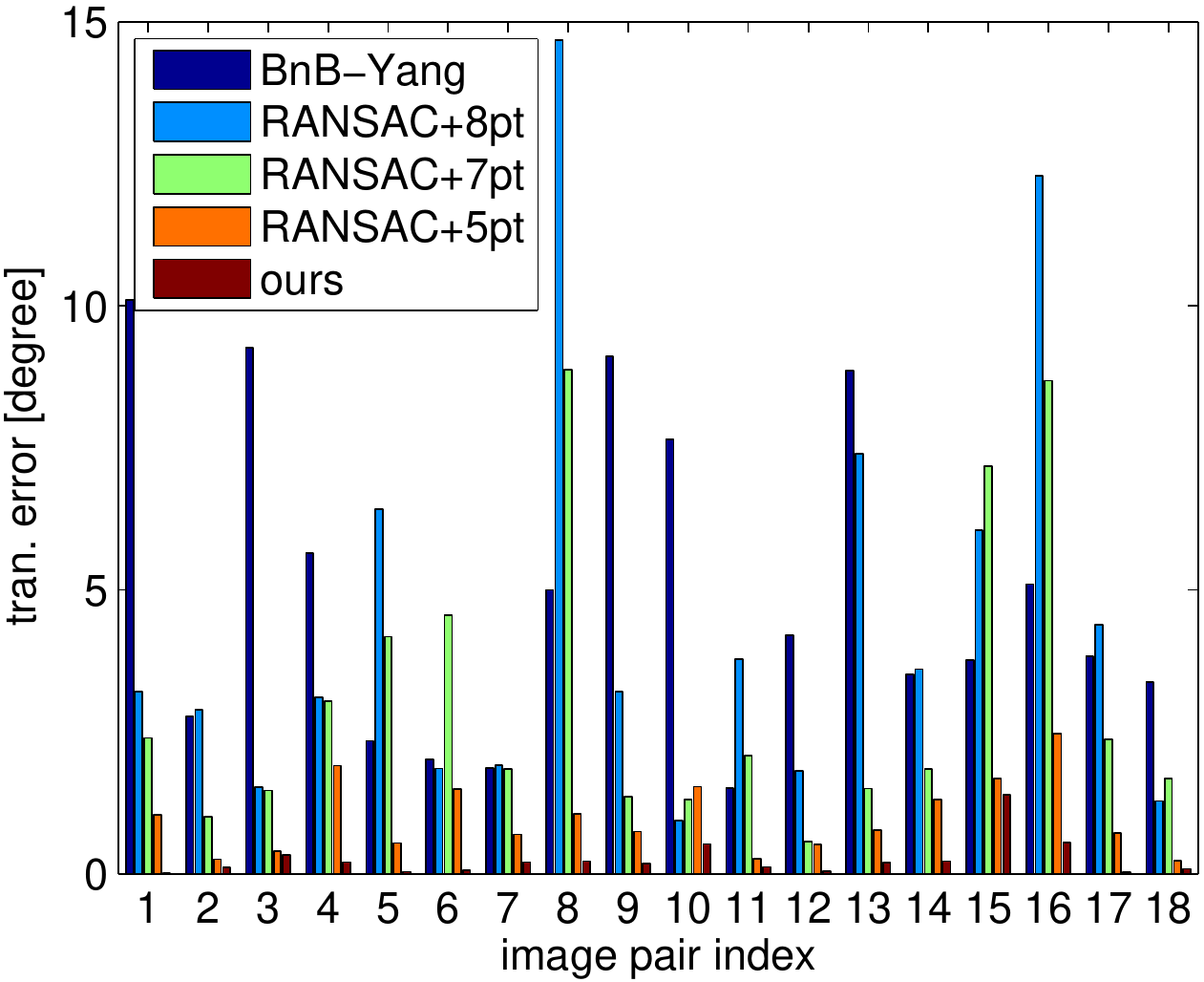}
	}
	\end{center}
	\vspace{-0.15in}
	\caption{Relative pose accuracy of the \texttt{EPFL Castle-19} dataset. 
	}
	\label{fig:exp_epfl_outlier}
\end{figure}

We use the \texttt{EPFL Castle-19} dataset~\cite{strecha2008on} for real-world data. The image pairs and their putative correspondences are generated in the same way as that in Section~\ref{sec:real_data1}.
Each image pair contains $4408$ putative correspondences on average. 
Though it is theoretically sound, the \texttt{BnB-Yang} method~\cite{yang2014optimal} cannot handle a large number of point correspondences. It fails to return a solution for an image pair in a day,  therefore we randomly select $100$ putative correspondences as input for it.

The relative pose accuracy of different methods is shown in Fig.~\ref{fig:exp_epfl_outlier}.
It can be seen that the proposed robust $N$-point method has higher overall accuracy than RANSAC-based methods and \texttt{BnB-Yang} method. 
The efficiency comparison between our method and RANSAC-based methods is summarized as below: 
(1)~The robust $N$-point method takes a roughly constant number of iterations. According to the GNC strategy and its parameter setting, its iteration number is at most $79$. In contrast, the iteration in RANSAC-based methods increases drastically with a high outlier ratio or a high confidence level. 
For example, given $45\%$ outliers and $99\%$ confidence level, \texttt{5pt}, \texttt{7pt} and \texttt{8pt} methods need to iterate at least $90$, $301$ and $548$ times, respectively. 
In practice, the RANSAC-based methods need more iterations than the least required to achieve high accuracy.
(2)~The non-minimal solver in the robust $N$-point method takes much more time than minimal solvers in RANSAC-based methods ($6$~milliseconds vs. $30\sim 180$~microseconds). As a result, our method takes about $480$~ms, and RANSAC-based methods take $52$~ms on average. 
(3)~The \texttt{BnB-Yang} method takes $118\sim 9007$~seconds for randomly selected $100$ point correspondences.

\section{Conclusions}
\label{sec:conclusion}

This paper introduces a novel non-minimal solver for $N$-point problem in essential matrix estimation. 
First, we reformulate  $N$-point problem as a simple QCQP by proposing an equivalent form of the essential manifold. 
Second, semidefinite relaxation is exploited to convert this problem to an SDP problem, and pose recovery from an optimal solution of SDP is proposed. 
Finally, a theoretical analysis of tightness and local stability is provided. 
Our method is stable, globally optimal, and relatively easy to implement. 
In addition, we propose a robust $N$-point method by integrating the non-minimal solver into M-estimators.
Extensive experiments demonstrate that the proposed $N$-point method can find and certify the global optimum of the optimization problem, and it is  $2\sim 3$ orders of magnitude faster than state-of-the-art non-minimal solvers. Moreover, the robust $N$-point method outperforms state-of-the-art methods in terms of robustness and accuracy.


%


\appendices
\section{Another Formulation of $N$-Point Problem}
\label{sec:another_form}
The proposition below provides another equivalent form to define essential manifold $\mathcal{M}_{\E}$, which will derivate another simple optimization problem.

\begin{lemma}
	For an essential matrix $\E$ which can be decomposed by $\E= [\mathbf{t}]_{\times} \R$, it satisfies that $\emph{trace}(\E\E^\top) = \sum_{i=1}^3 \sum_{j=1}^3 \E_{ij}^2 = 2 \|\mathbf{t}\|^2$.
	\label{lemma:e_trace}
\end{lemma}
\begin{proof}
	Note that the norm of each row of $\R$ is $1$, and the rows of $\R$ are orthogonal to each other. Taking advantage of $\E= [\mathbf{t}]_{\times} \R$, it can be verified the $\trace(\E\E^\top) = \sum_{i=1}^3 \sum_{j=1}^3 \E_{ij}^2 = 2 (t_1^2 + t_2^2 + t_3^2) = 2\|\mathbf{t}\|^2$.
\end{proof}

\begin{proposition}[Proposition~7.3 in~\cite{faugeras1993three}]
	\label{theorem:nister}
	A real $3\times 3$ matrix, $\E$, is an essential matrix if and only if it satisfies the equation: 
	\begin{align}
	\E\E^\top\E - \frac{1}{2} \trace(\E\E^\top)\E = 0.
	\label{equ:nister}
	\end{align}
\end{proposition}

\begin{proposition}
	\label{theorem:equivalent}
	A real $3\times 3$ matrix, $\E$, is an essential matrix in $\mathcal{M}_{\E}$ if and only if it satisfies the following two conditions: 
	\begin{align}
	\emph{\text{(i)}} \ \ \trace(\E\E^\top) = 2 \quad \emph{\text{and}} \quad \emph{\text{(ii)}} \ \ \E\E^\top\E = \E.
	\end{align}
\end{proposition}
\begin{proof}
	For the \emph{if} direction, by combining conditions (i) and (ii) we obtain Eq.~\eqref{equ:nister}. 
	According to Theorem~\ref{theorem:nister}, $\E$ is a valid essential matrix. So there exist (at least) a pair of $\mathbf{t} \in \mathbb{R}^3$ and $\R \in \text{SO}(3)$ such that $\E= [\mathbf{t}]_{\times} \R$. According to condition (i) and Lemma~\ref{lemma:e_trace}, we have $\trace(\E\E^\top) = 2\|\mathbf{t}\|^2 = 2$, which means $\|\mathbf{t}\| = 1$. Thus we prove $\E \in \mathcal{M}_{\E}$.
	
	For \emph{only if} direction, since $\E \in \mathcal{M}_{\E}$, it is straightforward that condition (i) is satisfied according to Lemma~\ref{lemma:e_trace}. In addition, Eq.~\eqref{equ:nister} is satisfied since $\E$ is an essential matrix. By substituting condition (i) in Eq.~\eqref{equ:nister}, we obtain condition (ii).
\end{proof}

According to Proposition~\ref{theorem:equivalent}, an equivalent form of minimizing the algebraic error is
\begin{align}
\label{equ:op_new_obj6}
\min_{\E} & \ \ \e^\top \C \e \\
\text{s.t.} & \ \  \trace(\E\E^\top) = 2, \quad \E\E^\top\E -\E = \mathbf{0}. \nonumber
\end{align}
By introducing an auxiliary matrix $\G$, this problem can be reformulated as a QCQP
\begin{align}
\label{equ:op_new_obj7}
\min_{\E, \G} & \ \ \e^\top \C \e \\
\text{s.t.} &  \ \ \G = \E\E^\top, \ \ \trace(\G) = 2, \ \ \G\E -\E = \mathbf{0}. \nonumber
\end{align}
Note that $\G$ is a symmetric matrix which introduces $6$ variables. Thus there are $15$ variables and $16$ constraints in this QCQP. 

\section{Strong Duality Between Primal SDP and Its Duality}
\label{sec:another_lemma}

\begin{lemma}
\label{lemma:dual}
For QCQP~\eqref{equ:op_new_obj3}, there is no duality gap between the primal SDP problem~\eqref{equ:op_obj6} and its dual problem~\eqref{equ:op_dual}.
\end{lemma}
\begin{proof}
Denote the optimal value for problem~\eqref{equ:op_obj6} and its dual problem~\eqref{equ:op_dual} as $f_{\text{primal}}$ and $f_{\text{dual}}$. The inequality $f_{\text{primal}} \ge f_{\text{dual}}$ follows from weak duality. Equality, and the existence of $\X^\star$ and $\boldsymbol{\lambda}^\star$ which attain the optimal values follow if we can show that the feasible regions of both the primal and dual problems have nonempty interiors, see~\cite[Theorem~3.1]{vandenberghe1996semidefinite} (also known as Slater's constraint qualification~\cite{boyd2004convex}.)
		
For the primal problem, let $\E_0$ be an arbitrary point on the essential manifold $\mathcal{M}_\E$: $\E_0 = [\mathbf{t}_0]_{\times} \R_0$, where $\|\mathbf{t}_0\| = 1$. Denote $\x_0 = [\text{vec}(\E_0); \mathbf{t_0}]$. It can be verified that $\X_0 = \x_0 \x_0^\top$ is an interior in the feasible domain of the primal problem. 
For the dual problem, let $\boldsymbol{\lambda}_0 = [-1, -1, -1, 0, 0, 0, -3]^\top$. Recall that $
\C \succeq 0$ and it can be verified that $\Q(\boldsymbol{\lambda}_0) = \begin{bmatrix}
\C & \mathbf{0} \\
\mathbf{0} & \mathbf{0}
\end{bmatrix} + \mathbf{I} \succ 0$. That means $\boldsymbol{\lambda}_0$ is an interior in the feasible domain of the dual problem. 
\end{proof}

\ifCLASSOPTIONcompsoc
  \section*{Acknowledgments}
\else
\fi

The author would like to thank Prof. Laurent Kneip at ShanghaiTech, Prof. Qian Zhao at XJTU, and Haoang Li at CUHK for fruitful discussions. The author also thanks Dr. Jesus Briales at the University of Malaga for providing the code of~\cite{briales2018certifiably} and Dr. Yijia He at CASIA for his help in the experiments.

\ifCLASSOPTIONcaptionsoff
  \newpage
\fi



\bibliographystyle{IEEEtran}
%
{
	\bibliography{egbib}
}

\end{document}